%% file: main_full.tex
\documentclass[letterpaper]{article} 
\usepackage{aaai2026}  
\usepackage{times}  
\usepackage{helvet}  
\usepackage{courier}  
\usepackage[hyphens]{url}  
\usepackage{graphicx} 
\usepackage{natbib}  
\usepackage{caption} 
\usepackage{algorithm}
\usepackage{algorithmic}

\usepackage{newfloat}
\usepackage{listings}
\DeclareCaptionStyle{ruled}{labelfont=normalfont,labelsep=colon,strut=off} 
\floatstyle{ruled}
\newfloat{listing}{tb}{lst}{}
\floatname{listing}{Listing}
\usepackage[utf8]{inputenc} 
\usepackage[T1]{fontenc}    
\usepackage{url}            
\usepackage{booktabs}       
\usepackage{amsfonts}       
\usepackage{nicefrac}       
\usepackage{microtype}      
\usepackage{xcolor}         

\usepackage{graphicx}
\usepackage{subcaption}

\usepackage{amsmath}
\usepackage{amssymb}
\usepackage{mathtools}
\usepackage{amsthm}
\usepackage{cleveref}

\usepackage{tabularx}
\usepackage{soul}
\usepackage{algorithm}
\usepackage{algorithmic}

\usepackage[table]{xcolor}
\usepackage{array}
\usepackage{booktabs}

\theoremstyle{plain}
\newtheorem{theorem}{Theorem}[section]

\newtheorem{lemma}[theorem]{Lemma}
\newtheorem{corollary}[theorem]{Corollary}
\theoremstyle{definition}
\newtheorem{definition}[theorem]{Definition}
\newtheorem{assumption}[theorem]{Assumption}

\crefname{assumption}{Assumption}{Assumptions}
\Crefname{assumption}{Assumption}{Assumptions}

\DeclareMathOperator*{\argmin}{arg\,min}

\newcommand{\alg}{$\mathsf{WC}$-$\mathsf{MHGD}$~}
\newcommand{\algns}{$\mathsf{WC}$-$\mathsf{MHGD}$}

\title{Multi-Objective Bilevel Learning}
\author{
    Zhiyao Zhang\textsuperscript{\rm 1}, Zhuqing Liu\textsuperscript{\rm 2}, Xin Zhang\textsuperscript{\rm 3}, Wen-Yen Chen\textsuperscript{\rm 3}, Jiyan Yang\textsuperscript{\rm 3}, Jia Liu\textsuperscript{\rm 1} \\
}
\affiliations{
    \textsuperscript{\rm 1}Dept. of Electrical and Computer Engineering, The Ohio State University\\
    \textsuperscript{\rm 2}Dept. of Computer Science, University of North Texas\\
    \textsuperscript{\rm 3}Meta Platforms, Inc.\\

}
\usepackage{bibentry}

\begin{document}

\maketitle

\begin{abstract}
    As machine learning (ML) applications grow increasingly complex in recent years, modern ML frameworks often need to address multiple potentially conflicting objectives with coupled decision variables across different layers.
    This creates a compelling need for multi-objective bilevel learning (MOBL).
    So far, however, the field of MOBL remains in its infancy and many important problems remain under-explored.
    This motivates us to fill this gap and systematically investigate the theoretical and algorithmic foundation of MOBL.
    Specifically, we consider MOBL problems with multiple conflicting objectives guided by preferences at the upper-level subproblem, where part of the inputs depend on the optimal solution of the lower-level subproblem. 
    Our goal is to develop efficient MOBL optimization algorithms to (1) identify a preference-guided Pareto-stationary solution with low oracle complexity; and (2) enable systematic Pareto front exploration. 
    To this end, we propose a unifying algorithmic framework called \ul{w}eighted-\ul{C}hebyshev \ul{m}ulti-\ul{h}yper-\ul{g}radient-\ul{d}escent (\algns) for both deterministic and stochastic settings with finite-time Pareto-stationarity convergence rate guarantees, which not only implies low oracle complexity but also induces systematic Pareto front exploration. 
    We further conduct extensive experiments to confirm our theoretical results.
\end{abstract}

\section{Introduction} \label{sec:intro}

{\bf 1) Background and motivation:}
As machine learning (ML) applications grow increasingly complex, modern ML training frameworks often need to address {\em multiple objectives} that are potentially conflicting.
As an example, in the design of online shopping recommender systems, the objective of a training loss function could aim to prioritize lower prices, more popular brand names, or faster delivery speeds, all of which may be in conflict with each other. 
Moreover, such objective functions could further depend on an optimal solution from a separate and coupled optimization task, thus making the overall training task {\em bilevel} in nature.
For instance, the multi-objective trainings in the aforementioned online shopping recommender system could all rely on a joint pre-training of a set of shared model parameters that define all objectives.
Similar multi-objective bilevel problem structures also arise from multi-agent reinforcement learning with an actor-critic framework~\cite{ji2021bilevel,yang2021provably,zhang2024introduction,qiu2023diamond} and the pretraining-finetuning pipeline of multi-task adaptation of a foundation model~\cite{chakraborty2024parl,shen2024principled}.
These complex problem structures create a compelling need for multi-objective bilevel learning (MOBL).

Mathematically, an MOBL problem can be formulated in the following general form:
\begin{equation}\label{eq:bilevel-def-intro}
    \begin{aligned}
        & \min_{x\in\mathbb{R}^p}\Phi(x) := [\phi_1(x),\dotsc, \phi_S(x)] \\
        & \text{s.t. } y^*(x) \in \argmin_{y\in\mathbb{R}^q}g(x,y),
    \end{aligned}
\end{equation}
where $\phi_s(x) := f^{(s)}(x, y^*(x))$, $s\in[S]$, represents an upper-level (UL) objective function with an implicit UL decision variable $y^*(x)$ being an optimizer of a lower-level (LL) subproblem. 
Clearly, MOBL combines the problem structures of multi-objective optimization (MOO) and bilevel optimization (BLO), both of which are challenging ML optimization paradigms in their own right and have attracted a significant amount of attention in the research community in recent years.
Although numerous algorithms have been proposed for solving MOO and BLO problems (e.g., \cite{konak2006multi,sener2018multi,yang2024federated} for MOO and \cite{ji2021bilevel,yang2021provably,zhang2024introduction,qiu2023diamond} for BLO, respectively; see Section~\ref{sec:rw} for in-depth discussions), the field of MOBL remains in its infancy, with many important open problems yet to be addressed.
This gap between current MOBL research state and practical needs of MOBL arises primarily from the following technical challenges.

\smallskip
{\bf 2) Technical challenges:}
Solving MOBL problems is, surprisingly, {\em significantly more challenging} than a straightforward combination of applying existing MOO and BLO solution techniques, which is due to the following two key factors:
\textbf{First}, being both multi-objective and bilevel, an MOBL problem not only inherits all the challenges of MOO and BLO, but also sees many unique and new complications unseen in either MOO or BLO.
Specifically, due to the finite-time constraint in solving the LL problem under the bilevel structure of MOBL problems, one typically has an inexact $y_k$ at each iteration instead of the true solution $y^*(x_k)$. 
This inaccuracy in $y_k$ subsequently results in a {\em biased estimate} $\hat{\nabla}\phi_s(x_k)$, which could {\em deviate} from a common descent direction computed according to $\nabla\phi_s(x_k)$ for all objectives. 
Beyond the cascaded error in $\nabla\phi_s(x_k)$, the inaccuracy can accumulate over iterations, causing the solution trajectory to {\em drift significantly} from the desired target.
\textbf{Second}, preference-guided Pareto front (i.e., the set of all Pareto-optimal solutions) exploration in MOBL is even more challenging (see solution philosophies in \Cref{sec:statement} for more details), since the Pareto stationarity (necessary condition of Pareto-optimality) and preference alignment should be balanced in the algorithmic design, which is highly susceptible to the accumulative cascaded errors inherent in the bilevel structure. 
Moreover, although existing works have incorporated scalarization methods to handle preference alignment, conventional rescaling strategies no longer apply due to the bilevel nature of MOBL problems, thereby necessitating novel technical designs.

\smallskip
{\bf 3) Main contributions:}
In this paper, we propose a series of new MOBL algorithms to overcome the aforementioned challenges, all of which enjoy a fast finite-time convergence rate guarantees and hence low oracle complexities.
Collectively, our results contribute toward establishing a theoretical foundation for MOBL.
Our main contributions and key results are summarized as follows:

\begin{list}{\labelitemi}{\leftmargin=1.5em \itemindent=-0.0em \itemsep=-.0em}
\item We propose a unifying MOBL algorithmic framework called \ul{w}eighted-\ul{C}hebyshev \ul{m}ulti-\ul{h}yper-\ul{g}radient-\ul{d}escent (\algns) for both deterministic and stochastic settings. 
Our \alg framework offers provable fast finite-time convergence rates and low oracle complexities of $\mathcal{O}(\epsilon^{-1})$ and $\mathcal{O}(\epsilon^{-2})$ for achieving an $\epsilon$-Pareto stationary solution in deterministic and stochastic settings, respectively.
To our knowledge, our finite-time convergence and oracle complexity results are the first of their kind in the MOBL literature involving preference.

\item It is worth pointing out that the algorithmic design and theoretical convergence analysis of our proposed \alg framework are highly non-trivial and far from straightforward extensions of traditional MOO and BLO techniques.
Notably, we rigorously show that our {\em new} dynamic weighting optimization subproblem (cf. Eq.~\eqref{eq:WC}) guarantees Pareto stationarity in MOBL problems.
Moreover, we establish a general result that the preference-guided dynamic weighted hypergradient direction in \alg satisfies a key minimum-norm condition (cf.~Lemma~\ref{lemma:PGMNL}), which serves as a foundation for our theoretical analysis and could be of independent interest for other MOBL research problems.

\item In addition to \alg algorithmic design and theoretical convergence analysis, we also offer the rationale behind our key algorithmic approaches and the geometric interpretation of our \alg approach in this paper, which advances and deepens our fundamental understandings of MOBL optimization theory and algorithms.
Moreover, in \Cref{sec:experiments}, we conduct extensive numerical experiments to verify our \alg algorithms. 
All of our empirical results confirm the efficacy and effectiveness of our \alg algorithms.
\end{list}

\section{Related Work}\label{sec:rw}

In this section, we provide an overview on recent works in MOBL to put our work in comparative perspectives. 
Due to space limitation, we relegate the overview of two closely related research areas, namely multi-objective optimization and bilevel optimization to Appendix B.

As mentioned in Section~\ref{sec:intro}, the field of MOBL remains in a nascent stage and results in this area are rather limited.
To our knowledge, existing works on MOBL include \cite{wang2024pareto,li2024bi,yang2024gradient,ye2024first,ye2021multi,gu2023min}.
Among these, the MOBL algorithms proposed in \cite{wang2024pareto,li2024bi} only provided empirical results without theoretical convergence and oracle complexity analysis. 
In contrast, it was shown that algorithms in \cite{yang2024gradient,ye2021multi} converge to a Pareto stationary point.
However, no finite-time convergence rate result was provided.

The most related works to our paper are \cite{ye2024first,fernando2022mitigating}, which have been shown to offer a finite-time convergence rate for the deterministic and stochastic setting, respectively.
However, our work differs from \cite{fernando2022mitigating,ye2024first} in the following key aspects:
{\bf (1)} Due to the use of less efficient value function algorithmic approach, the convergence rates in \cite{fernando2022mitigating,ye2024first} are $\mathcal{O}(SK^{-\frac{1}{2}})$ and $\mathcal{O}(1/K^{\frac{1}{4}})$, respectively.
In contrast, with a judicious algorithmic design, our \alg approach achieves a significantly faster $\mathcal{O}(1/K)$ finite-time convergence rate, and covers both the deterministic setting and the stochastic setting simultaneously;
and {\bf (2)} Our \alg approach enables a systematic {\em Pareto stationarity front exploration} guided by preferences (see solution philosophies in \Cref{sec:statement} for more details). Due to space limitation, we provide a detailed comparison in Appendix B.

\section{Multi-Objective Bilevel Learning: A Primer}\label{sec:statement}

To facilitate our MOBL deterministic and stochastic algorithm designs and theoretical convergence analysis, we first provide a primer on MOBL in this section.

\smallskip
{\bf 1) Deterministic and Stochastic Versions of MOBL:}
Note that the deterministic version of the MOBL problem is stated Eq.~\eqref{eq:bilevel-def-intro} in Section~\ref{sec:intro}. 
For the stochastic version of the MOBL problem, we assume the UL and LL objective functions can be expressed in the following forms, respectively:
\begin{align*}
& \phi_s(x) = f^{(s)}(x, y^*(x)) := \mathbb{E}_{\xi\sim\pi_{s}}[F^{(s)}(x, y^*(x);\xi)],\\
& g(x,y) := \mathbb{E}_{\zeta\sim\pi_{g}}[G(x,y;\zeta)],
\end{align*}
where $\pi_{s}$ and $\pi_g$ represent the data distributions underlying the $s$-th UL and the LL objective functions, respectively.
As will be shown later, the deterministic and stochastic versions of MOBL will have key differences in their algorithmic designs and theoretical convergence analysis.

\smallskip
{\bf 2) Optimality Criterion of MOBL:}
As mentioned in Section~\ref{sec:intro}, it is in general impossible to find a common $x$-solution to minimize all UL functions simultaneously due to the presence of potential conflicts between them. 
As a result, it is more appropriate to adopt the following notion of Pareto optimality as the performance metric for MOBL:

\begin{definition}[Pareto Optimality]\label{def:PO}
    A solution $x_1$ is said to dominate another solution $x_2$ if and only if 1) $\phi_s(x_1) \le \phi_s(x_2)$, $\forall s\in [S]$ and 2) there exists at least one $s\in [S]$ such that the inequality holds strictly. 
    A solution $x$ is Pareto optimal if no other $x'$ dominates $x$.
\end{definition}

Similar to solving most non-convex MOO or BLO problems, finding a Pareto-optimal
solution in non-convex MOBL is NP-Hard in general. 
As a result, it is often more practical to search for a Pareto-stationary solution defined as follows (a necessary condition for Pareto optimality):

\begin{definition}[Pareto Stationarity] \label{def:PS}
A solution $x$ is Pareto stationary if there does not exist a common descent direction $d \in \mathbb{R}^p$ such that $\nabla \phi_s (x)^{\top} d <0$, $\forall s \in [S]$.
\end{definition}

Although Definition~\ref{def:PS} rigorously defines the notion of Pareto stationary solution, it is inconvenient to work with in MOBL algorithm design.
Toward this end, we will adopt the notion of $\epsilon$-Pareto stationary solution defined as follows:

\begin{definition}[$\epsilon$-Pareto Stationarity]  \label{def:PS1}
    Let $\nabla\Phi(x) := [\nabla\phi_1(x), \dotsc, \nabla\phi_S(x)]^{\top}$.
    A solution $x$ is $\epsilon$-Pareto stationary if, for $\epsilon>0$, there exists a vector $\lambda \in \mathbb{R}^S$ in the $S$-simplex (i.e., $\lambda$ is nonnegative and $\textbf{1}^\top \lambda=1$) such that the $\lambda$-weighted gradient direction $d := \nabla\Phi(x)\lambda \in \mathbb{R}^p$ satisfies $\|d\|^2 \le \epsilon$.
\end{definition}

The above $\epsilon$-Pareto stationarity is derived from solving the Lagrangian dual problem of minimizing the worst descent among all UL objectives, which is logically equivalent to Definition~\ref{def:PO} (cf. e.g., \cite{sener2018multi, yang2024federated, zhou2024finite}).
We note that Definition~\ref{def:PS1} not only provides a useful metric in convergence analysis, but also motivates our \alg algorithm design later.

\smallskip
{\bf 3) MOBL Solution Philosophies:}
Due to the flexible multi-objective nature of MOBL problems, there are four basic solution philosophies commonly seen in practice, depending on whether or not a preference weight vector $r\in \Delta_S^+$ ($S$-simplex) associated with the UL objectives is used: 

\begin{list}{\labelitemi}{\leftmargin=1.8em \itemindent=-0.7em \itemsep=-.2em}

\item[\textbf{P1)}] Given preference $r$, identify a Pareto stationary point that aligns most closely with the preference; 

\item[\textbf{P2)}] Explore (partially) the Pareto front first, and then allow the decision-maker to select one from the exploration results based on some given preference $r$; 

\item[\textbf{P3)}] Allow the decision-maker to select iteratively by combining 1) and 2) with changing preferences; 

\item[\textbf{P4)}] Identify a solution based on certain global criterion rather than relying on  preference information. 
\end{list}

In this paper, we focus on Philosophy \textbf{P2}, since (i) the preference weight $r$ could be difficult to pre-determine in many MOBL applications; and (ii) the preference vector $r$ may not intersect with the Pareto front, rendering the P1 problem ill-defined.
Also, we will show in \Cref{sec:analysis} that the P4 can be treated as a {\em special case} of our proposed solution. 

Under Philosophy P2, our goal in solving a non-convex MOBL problem is two-fold: (1) achieving a $\epsilon$-Pareto stationary solution; and (2) systematically exploring the Pareto-stationarity front (i.e., the collection of all Pareto-stationary solutions) guided by preference information.

\section{The Proposed \alg Algorithms}\label{sec:alg}

In this section, we will first provide necessary preliminaries of our algorithmic designs in Section~\ref{sub:prelim}.
Then, we will present our \alg algorithms for deterministic and stochastic settings in Sections~\ref{subsec:alg-det} and \ref{subsec:alg-stoc}, respectively.

\subsection{Preliminaries} \label{sub:prelim}

\textbf{1) The basic idea of our \alg algorithms:} 
To achieve our first goal in finding a Pareto stationary solution in MOBL, our idea is to generalize the multi-gradient descent algorithm (MGDA) \cite{desideri2012multiple} for solving MOO problems \cite{coello2007evolutionary,elsken2018efficient,xu2024psmgd,sener2018multi} by replacing gradient information in MGDA with the {\em hypergradient} used in BLO approaches, hence the name multi-hypergradient-descent (MHGD).
Specifically, MHGD can be viewed as an extension of the (stochastic) gradient descent method to the multi-objective bilevel setting by first identifying a {\em common UL descent} direction that reduces the values of all $S$ UL objectives, and then move along this direction with an appropriately chosen step-size. 
If no common descent direction exists, then the current solution is a Pareto stationary point of the UL subproblem.

Next, to achieve the second goal in systematic Pareto-stationarity front exploration, we propose to integrate the weighted Chebyshev (WC) scalarization technique with our MHGD design.
The WC-scalarization transforms a vector-valued objective function into a scalar-valued objective function, which is more amenable for algorithm design.
Specifically, let $\Delta_S^+$ represent the $S$-dimensional simplex.
In an MOBL problem, the WC-scalarization with a preference vector $r \in \Delta_S^+$ is defined in the following {\bf min-max} form: $\mathsf{WC}_{r}(\Phi(x)) \!\!:=\! \min_{x} \max_i \{ r_i \phi_i(x)\}_{i=1}^{S} \!\!=\! \min_{x} \| r \odot \Phi(x) \|_{\infty}$,
where $\odot$ denotes the Hadamard product. 
The use of WC-scalarization in our \alg algorithmic design is motivated by the following key fact in MOO~\citep{Golovin2020RandomHS,qiu2024traversingparetooptimalpolicies}:
\begin{lemma}[Pareto Optimality Equivalence] \label{lem_WC}
A solution $x^*$ is weakly Pareto-optimal to an MOO problem $\min_{x} \Phi(x)$ if and only if $x^* \in \arg\min_{x} \mathsf{WC}_{r}(\Phi(x))$ for some $r \in \Delta_S^+$.    
\end{lemma}

Lemma~\ref{lem_WC} suggests that, by appropriately integrating WC-scalarization in MOBL algorithm design, we can systematically obtain all weakly Pareto-optimal solutions (i.e., exploring the weak Pareto front) by enumerating the WC-scalarization preference weight vector $r$ if the WC-scalarization problem can be solved optimally.
Indeed, combining the WC scalarization and MHGD yields our \alg framework.

\smallskip
{\bf 2) The hypergradient of the MOBL problem:}
In our \alg algorithms, we need to evaluate the set of all hypergradients denoted by the matrix $\nabla\Phi(x)$, or equivalently, to find the hypergradient $\nabla\phi_s(x)$ of each UL objective $s\in[S]$.
Following the implicit function theorem \cite{zhang2024introduction}, it can be shown that the each UL hypergradient can be computed as follows \cite{ghadimi2018approximation,ji2021bilevel}:
\begin{equation}\label{eq:true-gradient}
    \nabla\phi_s(x) = \nabla_x f^{(s)}(x, y^*(x)) - \nabla_{xy}^2g(x, y^*(x))v^*,
\end{equation}
where $v^*$ is the solution of the following linear equation system $\nabla_y^2g(x,y^*(x))v = \nabla_y f^{(s)}(x, y^*(x))$.
It is clear from Eq.~\eqref{eq:true-gradient} and the definition of $v^*$ that the Hessian $\nabla_y^2g(x,y^*(x))$ being invertible is a {\em necessary} condition for the hypergradient $\phi_s(x)$ to be well-defined.
Similar to most existing works on bilevel optimization in the ML literature, in this paper, we assume that the LL function $g(x,y)$ is strongly convex in $y$ (cf. \Cref{ass:convex}), so that $\nabla_y^2g(x,y^*(x))$ is positive definite and hence invertible. 
Thus, the LL solution $y^*(x)$ can be uniquely determined.

In practice, however, since the LL subproblem is typically solved numerically, one usually does {\em not} have the exact $y^*(x)$-solution and can only use some estimation of $y^*(x)$ in Eq.~\eqref{eq:true-gradient} to compute the hypergradient.
Consequently, the inexact $y^*(x)$-information results in systematic bias in hypergradient information.
Indeed, most of our subsequent MOBL algorithm designs in deterministic and stochastic settings are centered around addressing this key challenge.

Moreover, note that the hypergradient evaluation in Eq.~\eqref{eq:true-gradient} involves numerous first-order and second-order oracle evaluations (i.e., the gradient information, the Jacobian information, and the Hessian information), all of which are computationally expensive.
Therefore, to evaluate the efficiency of our proposed algorithms, we adopt the {\em oracle complexity} metric~\cite{zhang2024introduction}.
Specifically, we let $\mathsf{Gc}(f^{(s)},\epsilon)$ denote the number of the first-order oracle evaluations (i.e., partial gradients $\nabla_xf^{(s)}$ and $\nabla_yf^{(s)}$) for $f^{(s)}(x,y)$, $s\in[S]$. 
Also, for a vector $v$ and a positive $\epsilon$, we let $\mathsf{JV}(g,\epsilon)$ and $\mathsf{HV}(g,\epsilon)$ represent the numbers of second-order oracle evaluations in Jacobian-vector product $\nabla_{xy}^2g(x,y)v$ and  Hessian-vector product $\nabla_{y}^2g(x,y)v$, respectively.
Then, the oracle complexity is defined as follows:

\begin{definition}[Oracle Complexity of MOBL Algorithms] \label{def:SC}
The oracle complexity of an MOBL algorithm is defined as the total required numbers of first-order and second-order oracle evaluations for  the algorithm to converge to an $\epsilon$-Pareto-stationary solution.
\end{definition}

With all these preliminaries, we are now in a position to present our \alg algorithms.

\subsection{The Deterministic \alg Algorithm}\label{subsec:alg-det}

As illustrated in \Cref{alg:deterministic}, our deterministic \alg algorithm adopts a ``double-loop'' structure. 
In the inner-loop, we perform gradient-descent-style updates to approximately find the optimal $y^*(x_k)$-solution in the $k$-th outer iteration. 
To accelerate convergence, we apply a ``warm start'' technique inspired by \cite{ji2021bilevel}: the output from the end of the previous round of inner iterations serves as the initial point for the current round of inner iterations. 
After $D$ inner iterations, $y_k^D$ will be used as an estimation of $y^*(x_k)$.
In the $k$-th iteration of the outer-loop, we offer two options to compute the hypergradient for each UL objective $s\in [S]$.
Option 1 performs $N$ steps of the conjugate-gradient method (CG) to solve the linear equation system 
$\nabla^2_yg(x_k, y_k^D)v = \nabla_yf^{(s)}(x_k, y_k^D)$ to obtain  $v_k^{(s),N}$, which is used to compute the hypergradient as:
\begin{align}
& \label{eq:CG} \hat{\nabla}\phi_s(x_k) = \nabla_xf^{(s)}(x_k, y_k^D) - \nabla_{xy}^2g(x_k, y_k^D)v_k^{(s),N}.
\end{align}
In comparison, Option~2 approximates the Hessian inverse with the Neumann series $[\nabla_y^2 g(x_k, y_k^D)]^{-1} \approx \prod_{j=t+1}^{D-1}(I-\nabla_y^2g(x_k,y_k^j))$ and compute the hypergradient as:
\begin{equation}\label{eq:Neumann} 
    \begin{aligned}
        \hat{\nabla}\phi_s(x_k) = & \nabla_xf^{(s)}(x_k, y_k^D) - \alpha \sum_{t=0}^D\nabla_{xy}^2g(x_k, y_k^t) \\
        & \prod_{j=t+1}^{D-1}(I-\nabla_y^2g(x_k,y_k^j))\nabla_y f^{(s)}(x_k,y_k^D). \!\!
    \end{aligned}
\end{equation}

\begin{algorithm}[t]
    \caption{Deterministic \algns.}
    \label{alg:deterministic}
    \begin{algorithmic}[1]
        \STATE \textbf{Input:} The numbers of iteration $K, D, N$, initialization values $x_0, y_0, v_0$, preference $r$, and step-sizes $\alpha, \beta$.
        \FOR{\(k=0,1,\dots, K-1\)}
            \STATE Set $y_k^0 = y_{k-1}^D$ if $k>0$ and $y_0$ otherwise.
            \FOR{\(t=1,2,\dots, D\)}
                \STATE Update $y_k^t = y_k^{t-1} - \alpha\nabla_yg(x_k,y_k^{t-1})$.
            \ENDFOR
            \FOR{$s\in[S]$}
                \STATE \textbf{Option~1: Conjugate-Gradient(CG)}
                \STATE Set $v_k^{(s),0} = v_{k-1}^{(s),N}$ if $k>0$ and $v_0$ otherwise.
                \STATE Solve $\nabla^2_yg(x_k, y_k^D)v = \nabla_yf^{(s)}(x_k, y_k^D)$ for $N$ steps from $v_k^{(s),0}$ to get $v_k^{(s),N}$.
                \STATE Compute $\hat{\nabla}\phi_s(x_k)$ according to Eq.~\eqref{eq:CG}.
                \STATE \textbf{Option~2: Neumann-Series(NS)}
                \STATE Compute $\hat{\nabla}\phi_s(x_k)$ according to Eq.~\eqref{eq:Neumann}.
            \ENDFOR
            \STATE Compute $\lambda_k$ according to \Cref{eq:WC}.
            \STATE Update $x_{k+1} = x_k - \beta \hat{\nabla}\Phi(x_k)(r\odot\lambda_k)$.
        \ENDFOR
    \end{algorithmic}
\end{algorithm}

\begin{algorithm}[t]
    \caption{Stochastic \algns.}
    \label{alg:stoc}
    \begin{algorithmic}[1]
        \STATE \textbf{Input:} The numbers of iteration $K, D, Q$, initialization values $x_0, y_0$, preference $r$, and step-sizes $\alpha, \beta$.
        \FOR{\(k=0,1,\dots, K\)}
            \STATE Set $y_k^0 = y_{k-1}^D$ if $k>0$ and $y_0$ otherwise.
            \FOR{\(t=1,2,\dots, D\)}
                \STATE Draw a sample batch $\mathcal{T}_{t-1}$, and update $y_k^t = y_k^{t-1} - \alpha\nabla_yG(x_k,y_k^{t-1}; \mathcal{T}_{t-1})$.
            \ENDFOR
            \STATE Draw sample batches $\mathcal{D}_G, \mathcal{D}_H$.
            \FOR{$s\in[S]$}
                \STATE Draw sample batch $\mathcal{D}_F^s$.
                \STATE Set $v_k^{(s),0} = \nabla_y F^{(s)}(x_k, y_k^D; \mathcal{D}_F^s)$.
                \STATE Compute $v_k^{(s),Q}$ via \cref{alg:v-Q}, and $\hat{\nabla}\phi_s(x_k) = \nabla_x F^{(s)}(x_k, y_k^D; \mathcal{D}_F^s) - \nabla_{xy}^2G(x_k, y_k^D; \mathcal{D}_G)v_k^{(s),Q}$.
            \ENDFOR
            \STATE Compute $\lambda_k$ according to Eq.~\eqref{eq:WC}.
            \STATE Update $x_{k+1} = x_k - \beta \hat{\nabla}\Phi(x_k)(r\odot\lambda_k)$.
        \ENDFOR
    \end{algorithmic}
\end{algorithm}

With the hypergradients obtained from Eqs.~\eqref{eq:CG} or \eqref{eq:Neumann}, the heart of our \alg algorithm is to solve the following convex quadratic programming problem to find the optimal $\lambda$-weights to linearly combine hypergradients: 
\begin{equation}\label{eq:WC}
    \begin{aligned}
        & \min_{\lambda\ge \textbf{0}} \underbrace{\|K(r\odot \lambda)\|]^2}_{\mathrm{MHGD}} - u \underbrace{\lambda^\top (r\odot \Phi(x_k))}_{\mathrm{WC-Scalarization}} \\
        & \text{s.t. } \textbf{1}^\top\lambda = 1,
    \end{aligned}
\end{equation}
where $u>0$ is a tunable parameter, and $K := (\hat{\nabla}\Phi(x_k)^\top\hat{\nabla}\Phi(x_k))^{\frac{1}{2}} \in \mathbb{R}^{S\times S}$. 
Upon obtaining $\lambda_k^u$ by solving Problem~\eqref{eq:WC}, we compute the update direction $d_k(u)$ by taking the gradient of the Lagrangian of Problem~\eqref{eq:WC} to yield $d_k(u) = \hat{\nabla}\Phi(x_k)(r\odot\lambda_k^u)$. 
For brevity, we will simply denote $\lambda_k^u$ and $d_k(u)$ as $\lambda_k$ and $d_k$, respectively, when the context is clear.
Three important remarks are in order on Problem~\eqref{eq:WC}, which is a {\bf new design} {\em unseen} in the literature:

\smallskip
{\bf 1)} The first term in Problem~\eqref{eq:WC} plays the role of searching for a Pareto-stationary solution. 
To see this, it is insightful to recognize that the first term $\|K(r\odot \lambda)\|^2$ in Problem~\eqref{eq:WC} is equivalent to $\|\hat{\nabla}\Phi(x_k)(r\odot \lambda)\|^2$ in the $k$-th iteration (it is advantageous to use $K$ instead of $\hat{\nabla}\Phi(x_k)$ since the dimension of $K$ is typically much smaller ($S \ll p$)). 
Consequently, $\|K(r\odot \lambda)\|^2 = \|\hat{\nabla}\Phi(x_k) \text{diag}(r) \lambda\|^2$, where $\text{diag}(r)$ denotes the diagonal matrix with elements in $r$ on the main diagonal.
$\|K(r\odot \lambda)\|^2$ can be viewed as an $r$-scaled version of $\|\hat{\nabla}\Phi(x_k)\lambda\|^2$. 
Since $r$ is strictly positive, the $\lambda$-vector that minimizes $\|\hat{\nabla}\Phi(x_k)\lambda\|^2$ provides a {\bf common descent direction} $d = \hat{\nabla}\Phi(x_k)(r\odot\lambda)$ that improves all objectives. 
It is worth pointing out that the use of $\|K(r\odot \lambda)\|^2$ instead of the quadratic term $\|\Phi(x_k)\lambda\|^2$ in the original MGDA \cite{desideri2012multiple} is {\bf a major novelty} in this paper.
Indeed, this above Pareto-stationarity intuition can be made rigorous and stated as follows (proved in Appendix D):

\begin{lemma}[Pareto Stationarity]\label{lemma:stationary}
    For any preference $r\in\mathbb{R}^S_{++}$, if the first term in Eq.~\eqref{eq:WC} satisfies $K(r\odot\lambda) = 0$, then Pareto stationarity is achieved.
\end{lemma}

{\bf 2)} The second term $-\lambda^\top (r\odot \Phi(x_k))$ in Problem~\eqref{eq:WC} is used to facilitate the WC-scalarization for systematic Pareto stationarity front exploration.
To see why this term is corresponding to the WC scalarization, note that the primal WC problem can be written as minimizing a scalar $\gamma$ subject to $r\odot\Phi(x_k) \le \gamma \textbf{1}$, which effectively minimizes the weighted loss in $\ell_{\infty}$-norm sense \cite{momma2022multi}.
Then, by taking the Wolfe dual to incorporate Pareto-stationarity conditions, this WC-related objective transforms into maximizing $u\lambda^\top (r\odot \Phi(x_k))$, or equivalently minimizing $-u\lambda^\top (r\odot \Phi(x_k))$.

\smallskip
{\bf 3)} By choosing an appropriate $u$-value, the objective function in Problem~\eqref{eq:WC} strikes a balance between minimizing Pareto-stationarity gap under preference $r$ and systematically explore Pareto-stationarity front following the $r$-weighted WC scalarization.
Specifically, if Pareto-stationarity front exploration is of higher priority, one can choose a larger $u$-value. 
Otherwise, a small $u$-value favors minimizing the Pareto-stationarity gap of the obtained solution.
Geometric insight of Problem~\eqref{eq:WC} is provided in Appendix C due to space limitation.

\smallskip
{\bf 4)} In addition to the new algorithmic design of Problem~\eqref{eq:WC} (as shown in \Cref{lemma:PGMNL}), we also emphasize that the overall framework design of \alg approach is a {\em consequence of the special problem structure in MOBL.} 
Specifically, we note that a family of BLO algorithms (namely SOBA-based approaches update UL and LL variables $x$ and $y$ simultaneously in each iteration, \cite{dagreou2022framework,liu2023averaged}) achieves efficient convergence rate in the context of BLO. 
However, such SOBA-based approaches are not applicable in MOBL when multiple objectives are taken into account, since the $v$-variables lose their dual variable meaning as their counterpart in the single-objective setting. 
In contrast, our \alg method handles LL and UL variables in an alternating manner to avoid the explicit use of dual variables. This overcomes the limitation of the SOBA-based approaches and is specifically designed for MOBL problems.

\subsection{The Stochastic \alg Algorithm}\label{subsec:alg-stoc}

The stochastic \alg is illustrated in \Cref{alg:stoc}, which shares a similar structure and is based on the same intuition as in \Cref{alg:deterministic}. 
The key differences of the stochastic \alg algorithm include: 
(1) In the $t$-th inner loop iteration, the $y$-updates are conducted in a stochastic gradient descent fashion with a sample batch $\mathcal{T}_{t-1}$ and similar warm start technique as in deterministic \algns.
(2) In the $k$-th outer loop iteration, the Hessian-vector product $v_k^{(s),Q}$ required for hypergradient computation is inspired by \cite{ji2021bilevel}, which also uses the Neumann series technique for Hessian inverse approximation as shown in \Cref{alg:v-Q}. 
However, the computation of $v_k^{(s),Q}$ is significantly more complex compared to the deterministic case and care must be taken to tame oracle complexity. 
Similar to \cite{ji2021bilevel}, we select mutually independent datasets $\mathcal{B}_j$ with exponentially shrinking sizes in \Cref{alg:v-Q}. 

\begin{algorithm}[t]
	\caption{Hessian-vector Product $v_k^{(s),Q}$ Computation}\label{alg:v-Q}
	\begin{algorithmic}[1]
        \STATE \textbf{Input:} The number of iteration $Q$, initialization values $v_k^{(s),Q}$, samples $\mathcal{D}_H = \{\mathcal{B}_j\}_{j=1}^Q$ and step-sizes $\eta$.
        \FOR{\(j=1,2,\dots, Q\)}
            \STATE Use $\mathcal{B}_j$ to compute $G_j(y) = y - \eta\nabla_yG(x,y; \mathcal{B}_j)$.
        \ENDFOR
        \STATE Set $\nu^{(s),Q} = v_k^{(s),0}$.
        \FOR{\(i=Q, Q-1, \dotsc, 1\)}
            \STATE Compute $\nu^{(s),i-1} = \nu^{(s),i} - \eta \nabla_y^2G(x,y;\mathcal{B}_i)\nu^{(s),i}$.
        \ENDFOR
        \STATE \textbf{Output:} $v_k^{(s),Q} = \eta \sum_{i=0}^Q \nu^{(s),i-1}$.
	\end{algorithmic}
\end{algorithm}

\section{Pareto-Stationarity Convergence Analysis}\label{sec:analysis}

In this section, we start by introducing several assumptions required for our analysis, which are followed by the main finite-time convergence results of our \alg algorithms.

\begin{assumption}[LL Objective Function]\label{ass:convex}
    For any $x\in \mathbb{R}^p$, $g(x, \cdot)$ in the deterministic setting and $G(x,\cdot;\zeta)$ in the stochastic setting are $\mu_g$-strongly-convex with respect to $y$ for some constant $\mu_g>0$.
\end{assumption}

\noindent
We note that Assumption~\ref{ass:convex} is standard and has been widely adopted in the literature \cite{ji2021bilevel,yang2021provably,liu2023prometheus,hu2023contextual,yang2023achievingoepsilon15complexityhessianjacobianfree,gong2024accelerated}. 
On one hand, strong convexity ensures that the Hessian inverse $\nabla_y^2g(x,y)^{-1}$ well-defined and that the solution to the LL subproblem can be uniquely determined; on the other hand, it enables the quantification of the error between the approximated and ground-truth hypergradients.

\begin{assumption}[Smoothness]\label{ass:Lipchitz}
    There exist positive constants $M$, $L$, $\tau$, and $\rho$, such that for any $s\in[S]$ and any $z,z' \in \mathbb{R}^{p+q}$, the following Lipchitz continuity conditions hold under the deterministic setting:
    \begin{align*}
            & |f^{(s)}(z) - f^{(s)}(z')| \le M \|z - z'\|, \\
            & \|\nabla f^{(s)}(z) - \nabla f^{(s)}(z')\| \le L \|z - z'\|, \\
            & \|\nabla g(z) - \nabla g(z')\| \le L \|z - z'\|, \\
            & \|\nabla_{xy}^2 g(z) - \nabla_{xy}^2 g(z')\| \le \tau \|z - z'\|, \\
            & \|\nabla_y^2 g(z) - \nabla_y^2 g(z')\| \le \rho \|z - z'\|, 
    \end{align*}
    where $\|\cdot\|$ denotes the $\ell_2$-induced norm (abbreviated as $\|\cdot\|$ in the sequel for simplicity). 
    In addition, the stochastic functions $F^{(s)}(z;\xi)$ and $G(z;\zeta)$ also satisfy their corresponding smoothness assumptions for any sample pairs $\xi$ and $\zeta$.
\end{assumption}

We note that the above smoothness assumptions are also standard and have been widely adopted in the literature for convergence analysis \cite{ghadimi2018approximation,ji2021bilevel,qiu2023diamond,yang2021provably,lin2024non}. 
Lastly, we make the following assumption on the variance of the LL objective function as follows:

\begin{assumption}[Bounded LL Variance]\label{ass:variance}
    The LL stochastic gradient $\nabla G(z;\zeta)$ has a variance bounded by a constant $\sigma>0$, i.e., $\mathbb{E}_{\zeta}\|\nabla G(z;\zeta) - \nabla g(z)\|^2 \le \sigma^2$.
\end{assumption}

{\bf 1) Convergence Rate of the Deterministic \algns:}
We let $\bar{d}_k = \nabla\Phi(x_k)\text{diag}(r)\lambda_k$ for any $k\in \{0, \dotsc, K\}$. 
Also, we let $\kappa=\frac{L}{\mu_g}$ denote the condition number and define $L_{\phi} := L + \frac{2L^2 + \tau M^2}{\mu_g} + \frac{\rho LM + L^3 + \tau LM}{\mu_g^2} + \frac{\rho L^2M}{\mu_g^3} = \Theta(\kappa^3)$. 
We first state a key result to establish the convergence of our \alg method, which is new for MOBL problems and not in the MOO and BLO literature.

\begin{lemma}[Preference-Guided Minimum-Norm Lemma]\label{lemma:PGMNL}
    Let $\lambda_k^u$ be obtained from solving Problem~\eqref{eq:WC} for some $u>0$ and thus the update direction is $d_k(u) := \hat{\nabla}\Phi(x_k)(r\odot\lambda_k^u)$.
    Let $r_{\text{max}}$ be the maximal element of $r$. 
    For any $s\in[S]$ and $r \in \Delta_S^+$, there exists a sufficiently small $u$ such that the following holds: $\|d_k(u)\|^2 \le 2r_{\text{max}} \langle d_k(u), \hat{\nabla}\phi_s(x_k) \rangle$.
\end{lemma}

\Cref{lemma:PGMNL} plays an important role in the subsequent analysis. While the relationship between the derived direction $d_k$ and the gradient $\hat{\nabla}\phi_s(x_k)$ in MGDA is well understood, it only holds in conventional MOO settings. 
Therefore, adopting the newly designed Problem~\eqref{eq:WC} is essential for MOBL problems. However, this new design also implies a {\em new} common descent direction $d_k(u)$, for which {\em no existing works} have characterized its properties. 
To this end, we establish \Cref{lemma:PGMNL} to reveal how $d_k(u)$ correlates with $r$ and $\hat{\nabla}\phi_s$. 
The result proved in \Cref{lemma:PGMNL} further leads to our finite-time convergence rate result, which is stated as follows:

\begin{theorem}[Finite-Time Convergence Rate of Deterministic \algns]\label{thm:WC_det}
    Under \Cref{ass:convex,ass:Lipchitz}, for any preference $r \in \Delta_S^{++}$ and $\epsilon>0$, by choosing $\alpha \le \frac{1}{L}$, $\beta = \min\{\frac{1}{2(1+L_{\phi})r_{\text{max}}}, \frac{1}{3L_{\phi}}\}$, $D \ge \Theta(\kappa\log\frac{1}{\epsilon})$, $N \ge \Theta(\sqrt{\kappa}\log\frac{1}{\epsilon})$, and $K \ge \Theta(\kappa^3\epsilon^{-1})$, the following results of \Cref{alg:deterministic} hold: 
    
    \smallskip
    $\bullet$ The iterates under the NS option satisfy:
        \begin{align*}
            & \frac{1}{K}\sum_{k=0}^{K-1} \|\bar{d}_k\|^2 \le \frac{24r_{\text{max}}L_{\phi}(\phi_s(x_0) - \phi_s(x_{K}))}{K} \\
            & + (12+24r_{\text{max}}L_{\phi}) \cdot \Bigg[\Bigg(\frac{4M^2(\tau\mu_g+L\rho)^2}{\mu_g^4}(1-\alpha\mu_g)^{D-1} \\
            & + \frac{2L^2(L+\mu_g)^2(1-\alpha\mu_g)^D}{\mu_g^2}\Bigg)\chi + \frac{L^2M^2(1-\alpha\mu_g)^{2D}}{\mu_g^2} \Bigg],
        \end{align*}

    $\bullet$ The iterates under the CG option satisfy:
    \begin{align*}
            \frac{1}{K}\sum_{k=0}^{K-1}  & \|\bar{d}_k\|^2 \le \frac{12r_{\text{max}}L_{\phi}\Big(\phi_s(x_0) - \phi_s(x_{K}) + \delta_{D,N}\chi_0\Big)}{K} \\
            & + 24(r_{\text{max}}L_{\phi}+1)M^2\delta_{D,N}\Omega^\frac{3}{2} + 6\delta_{D,N}\chi_0,
    \end{align*}
    where $\delta_{D,N} := \Gamma(1-\alpha\mu_g)^D + 6L^2\kappa\left(\frac{\sqrt{\kappa}-1}{\sqrt{\kappa}+1}\right)^{2N}, \Omega := 4\beta^2\left(\kappa^2 + \frac{ML}{\mu_g^2} + \frac{ML\kappa}{\mu_g^2}\right)^2, \chi_0 := \|y_0 - y^*(x_0)\|^2 + \|v_0 - v^*_0\|^2$, and $\Gamma$ and $\chi$ are positive constants.
\end{theorem}

The following oracle complexity results immediately follow from Theorem~\ref{thm:WC_det}:
\begin{corollary}[Oracle Complexity of Deterministic \alg]
    To achieve an $\epsilon$-Pareto stationary point, by choosing $D$ and $N$ as Appendix D, \Cref{alg:deterministic} achieves the following oracle complexities:
    
    $\bullet$ NS: $\mathsf{Gc}(f, \epsilon) \!=\! \mathcal{O}(\kappa^3\epsilon^{-1}S)$, $\mathsf{Gc}(g, \epsilon) \!=\! \widetilde{\mathcal{O}}(\kappa^4\epsilon^{-1})$, $\mathsf{JV}(g, \epsilon) \!=\! \widetilde{\mathcal{O}}(\kappa^4\epsilon^{-1}S)$, $\mathsf{HV}(g, \epsilon) \!=\! \widetilde{\mathcal{O}}(\kappa^4\epsilon^{-1}S)$. 
    
    $\bullet$ CG: \! $\mathsf{Gc}(f, \epsilon) \!\!=\!\! \mathcal{O}(\kappa^3\epsilon^{-1}S)$, $\mathsf{Gc}(g, \epsilon) \!=\! \widetilde{\mathcal{O}}(\kappa^4\epsilon^{-1})$, $\mathsf{JV}(g, \epsilon) \!\!=\!\! \mathcal{O}(\kappa^3\epsilon^{-1}S)$, $\mathsf{HV}(g, \epsilon) \!\!=\!\! \widetilde{\mathcal{O}}(\kappa^{3.5}\epsilon^{-1}S)$.
\end{corollary}

Compared to single-objective BLO approaches (e.g., \cite{ji2021bilevel}), both options in \Cref{alg:deterministic} achieve the {\em same} convergence rate of $\mathcal{O}(1/K)$. 
Also, all oracle complexity bounds increase by a factor of $S$,
which is due to the fact that oracle samples are needed for each of the $S$ UL objectives.
Notably, \Cref{thm:WC_det} indicates that, even with more complex preference incorporation, our \alg approach still converges to a neighborhood of a preference-guided Pareto stationary point. 
This highly non-trivial preference-guided Pareto-stationarity convergence, which is similar to that of the conventional MGDA method, critically depends on our new algorithmic design in Problem~\eqref{eq:WC} -- a main contribution of this work as pointed out previously.

\smallskip
{\bf 2) Convergence rate of the stochastic \algns:}
With the same notation as in the deterministic setting, we state the result of the stochastic \alg algorithm as follows:

\begin{theorem}[Finite-Time Convergence Rate of Stochastic \algns]\label{thm:WC_stoc}
    Under \Cref{ass:convex,ass:Lipchitz,ass:variance}, for any preference $r \in \Delta_S^{++}$ and $\epsilon>0$, by choosing $\alpha = \frac{2}{L+\mu_g}$, $\beta = \min\{\frac{1}{2r_{\text{max}}(1+L_{\phi})}, \frac{1}{3L_{\phi}}\}$, $\eta \le \frac{1}{L}$, $D = \Theta(\kappa\log(\epsilon^{-1}))$, and $K = \Theta(\frac{\kappa^3}{\epsilon})$, the following results of \Cref{alg:stoc} hold:
    \begin{equation*}
        \begin{aligned}
            \frac{1}{K} & \sum_{k=1}^K\mathbb{E}\|\bar{d}_k\|^2 \le \gamma\cdot\frac{3M^2(\kappa^2+1)}{1-\gamma}\Big(1 + 2L_{\phi}r_{\text{max}}\Big) \\
            & + \gamma\cdot\Big(6\nu+\frac{12L_{\phi}r_{\text{max}}\nu}{1-\gamma}\Big)\mathbb{E}\|y_0 - y^*(x_0)\|^2 \\
            & + \mathcal{O}\Big(\frac{L_{\phi}}{K} + \frac{\kappa^8\sigma^2}{T} + \frac{\kappa^5}{D_g} + \frac{\kappa^5}{D_f} + \frac{\kappa^5}{B} + \kappa^5(1-\eta\mu_g)^{2Q}\Big),
        \end{aligned}
    \end{equation*}
    where $\nu = \Theta(\kappa^4)$, $\gamma = 16\left(\frac{L-\mu_g}{L+\mu_g}\right)^{2D}\frac{\beta^2 L^2\nu}{\mu_g^2}\frac{r_{\text{max}}}{1-r_{\text{max}}}$; and the batch-sizes satisfy $|\mathcal{B}_{Q+1-j}| = BQ(1-\eta\mu_g)^{j-1}$ for $j\in[Q]$, $BQ(1-\eta\mu_g)^{Q-1} \ge 1$, $Q = \Theta(\kappa\log(\kappa^2\epsilon^{-1}))$.
\end{theorem}

The following oracle complexity results immediately follow from Theorem~\ref{thm:WC_stoc}:

\begin{corollary}[Oracle Complexity of Stochastic \alg]
    To achieve an $\epsilon$-Pareto stationary point, by choosing parameters as Appendix D, \Cref{alg:stoc} achieves the following oracle complexities: $\mathsf{Gc}(f, \epsilon) \!=\! \mathcal{O}(\kappa^8\epsilon^{-2}S)$, $\mathsf{Gc}(g, \epsilon) \!=\! \widetilde{\mathcal{O}}(\kappa^{12}\epsilon^{-2})$, $\mathsf{JV}(g, \epsilon) \!=\! \mathcal{O}(\kappa^8\epsilon^{-2}S)$, $\mathsf{HV}(g, \epsilon) \!=\! \widetilde{\mathcal{O}}(\kappa^9\epsilon^{-2}S)$.
\end{corollary}

\Cref{thm:WC_stoc} shows that, as the numbers of iterations and batch sizes increase, \Cref{alg:stoc} converges to a neighborhood of a Pareto stationary solution at rate $\mathcal{O}(1/K)$. 
Compared to the single-objective BLO approaches (e.g., \cite{ji2021bilevel}), the oracle complexity results in \Cref{thm:WC_stoc} increase by a factor $\kappa^3$. 
This is because stochastic MOBL methods require more data to address the bilevel error propagation issue of each UL objective and mitigate the discrepancy between approximate and true hypergradients. 

Lastly, we note that preference information may not always be available in certain MOBL applications.
For such MOBL problems, we can still establish their Pareto-stationarity convergence rate results in terms of $\bar{d}_k$ under the special case without $r$.
In fact, we show that the Pareto-stationarity convergence rate in this special case can be further tightened thanks to the simplified structure of the problem. 
Due to space limitations, we present this result in Appendix D.

\section{Numerical Experiments}\label{sec:experiments}

\begin{figure}[t]
    \centering
    \includegraphics[width=0.375\textwidth]{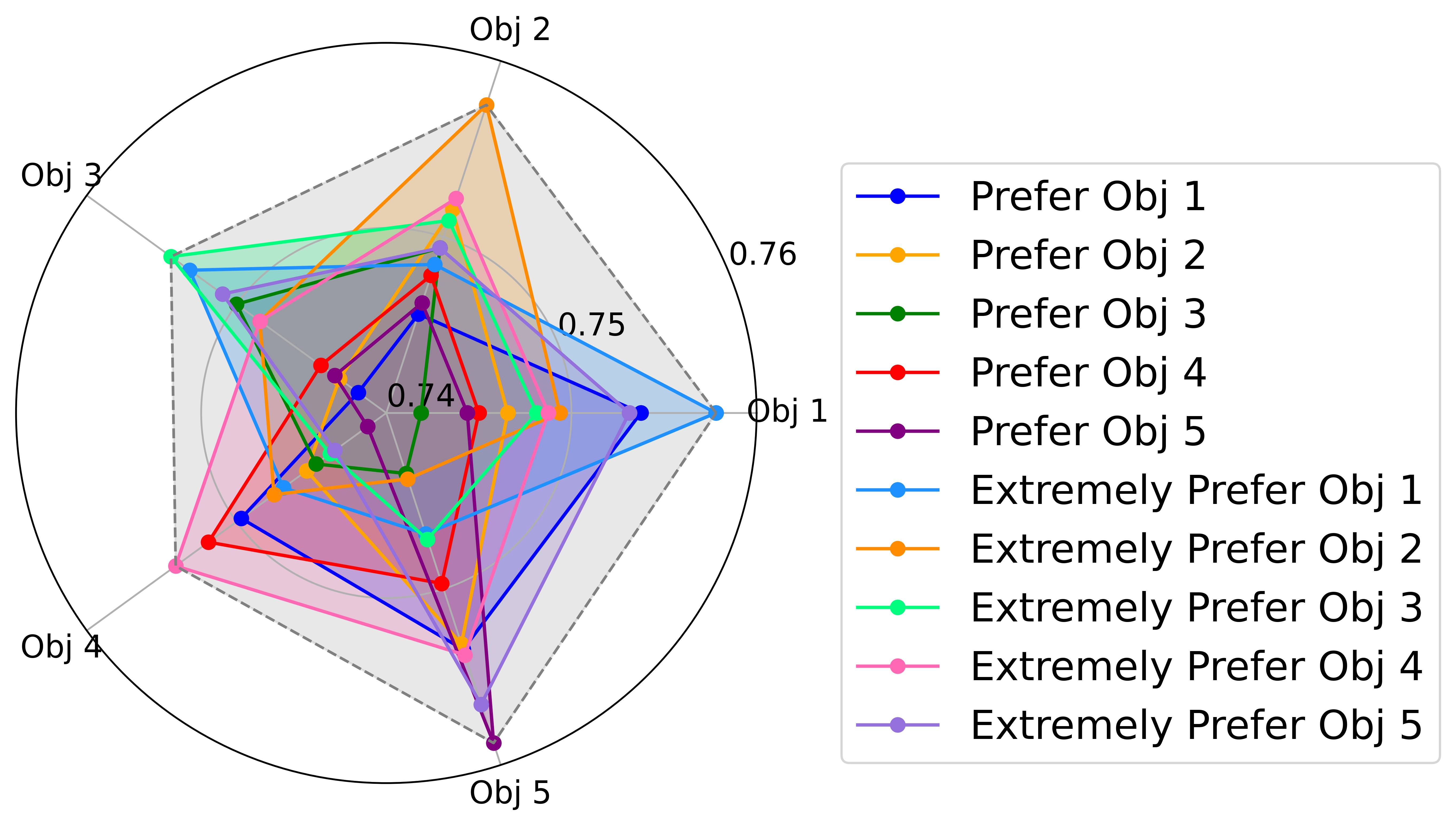}
    \caption{Pareto front exploration of \Cref{alg:deterministic}.}
    \label{fig:det-pref}
\end{figure}
\begin{figure}[t]
    \centering
    \includegraphics[width=0.375\textwidth]{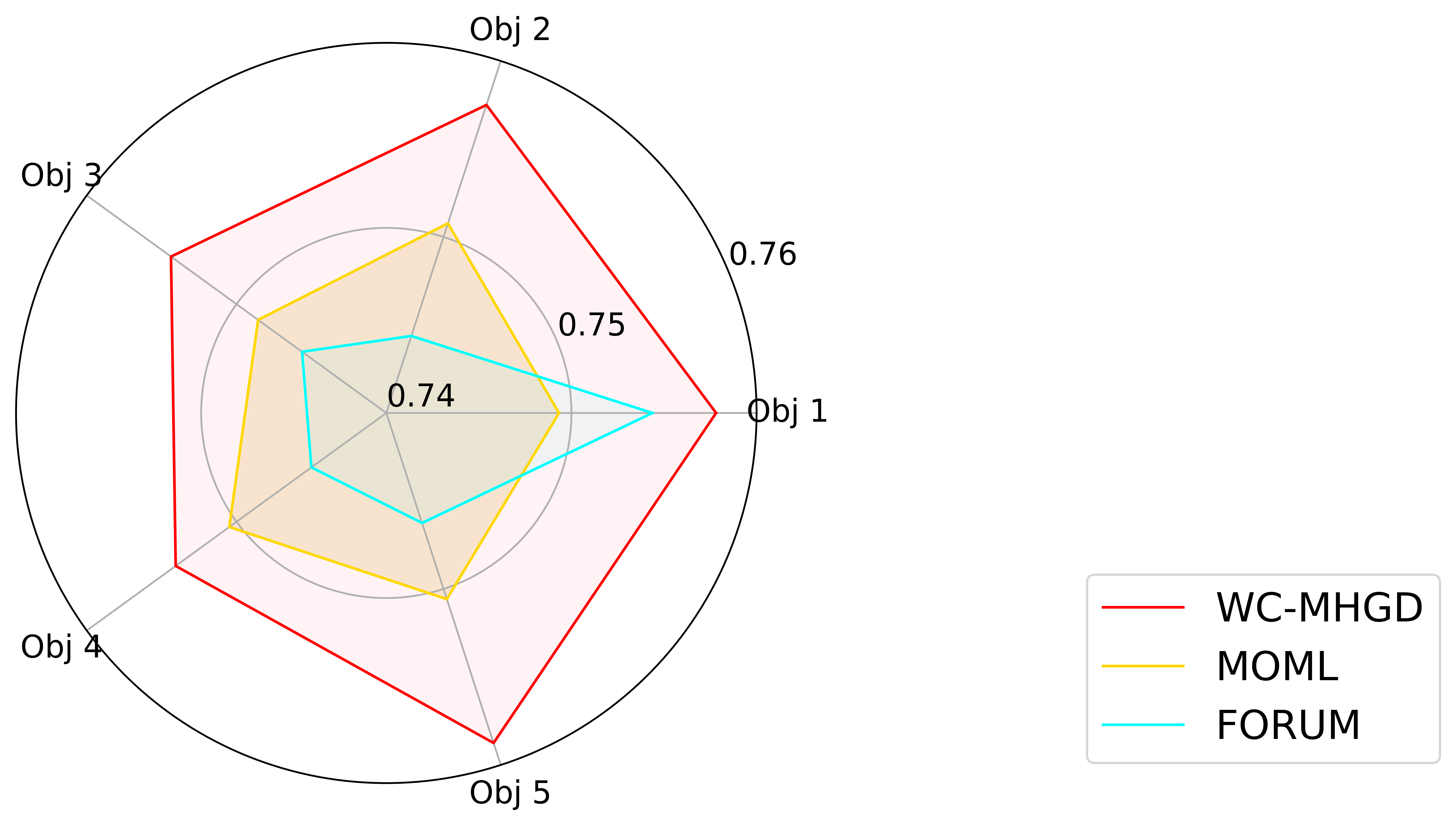}
    \caption{Comparison between algorithms.}
    \label{fig:compare}
\end{figure}

\textbf{1) Experiment settings:} 
{\em 1-a) The Deterministic Setting Experiments:}
We use a $5$-way $5$-shot meta-learning task to test \Cref{alg:deterministic} on the \texttt{FC-100}
dataset \cite{oreshkin2018tadam}, thus we have $S = 5$ objectives. 
For these tasks, we consider multiple models comprising shared parameters and task-specific parameters. 
The corresponding MOBL problem involves minimizing the training loss by first determining the task-specific LL parameters and then finding the optimal shared LL parameters based on the LL solutions to minimize the validation loss. 
We set $(K, D) = (500, 32)$ and $u = 10$ by default. 
We set the preference vector $r$ as: 
i) $r_s = 0.8$ for some $s \in [S]$ and $r_{s'} = 0.05$, $\forall s'\ne s$;
and ii) $r_s = 0.96$ for some $s \in [S]$ and $r_{s'} = 0.01$, $\forall s'\ne s$.
Those objectives with $r_s =0.8$ and $r_s=0.96$ are referred to as {\em ``preferred''} and {\em ``extremely preferred,''} respectively.
Each experiment is averaged over $4$ trials.

\textit{1-b) The Stochastic Setting Experiments:}
We use a data hyper-cleaning task to test \Cref{alg:stoc} on the \texttt{MNIST} dataset \cite{lecun1998gradient}. 
Under 5 different corruption rates $p$ (hence $S=5$ objectives), multiple models share the same UL regularization parameter based on the updates of their $p$-specific LL parameters. 
We set $(K, D) = (150, 200)$ and $u = 10$ by default. 
Each experiment is averaged over $5$ trials.
The detailed setups can be found in Appendix E.

\smallskip
\textbf{2) Experiment Results:}
Fig.~\ref{fig:det-pref} shows how \Cref{alg:deterministic} converges to preference-guided Pareto stationary points.
We can see that the accuracy of the more preferred objectives are consistently higher than those of the less preferred objectives.
Moreover, for those extremely preferred objectives, this trend becomes more pronounced.
The outer gray area demonstrates that our \alg algorithm effectively explores a large Pareto stationarity footprint.
Fig.~\ref{fig:compare} highlights that Pareto exploration allows \Cref{alg:deterministic} to cover a larger portion of Pareto front, compared to MOML \cite{ye2021multi} and FORUM \cite{ye2024first}. 
Due to space limitation, we relegate further experimental details and results in Appendix E.

\section{Conclusion}
In this paper, we investigated MOBL problems and proposed a unifying \alg algorithmic framework, which offers finite-time Pareto-stationarity convergence guarantees, low oracle complexities, and systematic Pareto-stationarity front exploration capability. 
Our numerical results further verified the effectiveness of our proposed algorithms. 

\section*{Acknowledgement}
This work is supported in part by Meta Research Award INB3267726, NSF grants CAREER CNS-2110259, CNS-2112471, and ECCS-233104, DARPA grants YFA D24AP00265 and HR0011-25-2-0019, and ONR grant N00014-24-1-2729.

\bibliography{aaai2026.bib,diamond_ref.bib}

\appendix
\onecolumn
\input{appendix.tex}

\end{document}

%% file: appendix.tex
\setcounter{secnumdepth}{2} 

\section{Notations}

We summarize the notations throughout this paper in \Cref{tab:notation}.

\begin{table}[htb]
\begin{tabularx}{\textwidth}{@{}p{0.4\textwidth}p{0.5\textwidth}@{}}
\toprule
  \textbf{Notations} & \textbf{Definitions} \\
  $d_k$ with preferences & $d_k = \hat{\nabla}\Phi(x_k)(r\odot\lambda_k)$ \\
  $\bar{d}_k$ with preferences & $\bar{d}_k = \nabla\Phi(x_k)(r\odot\lambda_k)$ \\
  $d_k$ without preferences & $d_k = \hat{\nabla}\Phi(x_k)\lambda_k = \sum_{s=1}^S\lambda_k^{(s)}\hat{\nabla}\phi_s(x_k)$ \\
  $\bar{d}_k$ without preferences & $\bar{d}_k = \nabla\Phi(x_k)\lambda_k = \sum_{s=1}^S\lambda_k^{(s)}\nabla\phi_s(x_k)$ \\
  $[S]$ & Set of integers $\{1, \dotsc, S\}$ \\
  $\|\cdot\|$ & $l_2$-norm \\
  $\hat{\nabla}\Phi(x) \in \mathbb{R}^{p\times S}$ & $\hat{\nabla}\Phi(x)$ is a $p$ by $S$ matrix\\
  $r\in \mathbb{R}^S_{++}$ & $r$ is a strictly positive vector in $\mathbb{R}^S$ \\
  $\Delta_S^+$ & Simplex in $S$ dimension \\
  $\Delta_S^{++}$ & Strictly positive simplex in $S$ dimension \\
  $e_s$ & A vector with $1$ in the $s$-th position and $0$ otherwise \\
  $\textbf{1}$ & A vector with $1$ as all elements \\
  $\odot$ & Hadamard product notation \\
\bottomrule
\end{tabularx}
\vspace{1em}
\caption{Summarized notation table in the paper.}\label{tab:notation}
\end{table}

\section{Related work and comparison between \alg and existing algorithms}\label{sec:compare}

In this section, we provide an overview on two closely related research areas, namely MOO, BLO, and the comparison between our \alg approach and recent works in MOBL.

\textbf{1) Multi-objective Optimization (MOO):} The research on MOO dates back to \cite{sawaragi1985theory} and has been actively studied to this date (see, e.g., \cite{chankong2008multiobjective,hwang2012multiple,gunantara2018review} for overviews). 
In recent years, thanks to the rise of ML problems with first-order oracles being readily accessible, gradient-based MOO approaches have gained increasing attention. 
Numerous gradient-based MOO algorithms have been proposed \cite{desideri2012multiple,mercier2018stochastic,miettinen1995interactive,coello2007evolutionary,elsken2018efficient,xu2024psmgd,momma2022multi,mahapatra2020multi,lin2024smooth,chen2024ferero}, with applications spanning multi-task learning \cite{sener2018multi,momma2022multi,crawshaw2020multi,lin2019pareto}, feature selection \cite{peitz2024multi,al2024multi,gong2015unsupervised}, multi-objective training and clustering \cite{mossalam2016multi,acskan2014svm,alok2015new,gonzalez2020improving}, architecture search \cite{jin2007evolutionary}, hyper-parameter optimization \cite{karl2023multi}, data imputation \cite{lobato2015multi}, to name a few. 
However, none of these gradient-based MOO approaches considered bilevel-type problem structure. 
In contrast, our work proposes the first set of multi-objective bilevel algorithms designed to achieve Pareto-stationary solutions and explore Pareto front guided by preferences, with theoretical finite-time convergence guarantees.

\textbf{2) Bilevel Optimization (BLO):} 
Similar to MOO, BLO also has a long research and, to our knowledge, the earliest work on BLO appeared in~\cite{bracken1973mathematical}. 
Over the years, several BLO approaches have been proposed, including 1) penalizing the UL function with the optimality conditions of the LL problem~\cite{shi2005extended, mehra2021penalty}; 2) reformulating the BLO problem as a single-level problem by replacing the LL problem with its optimality conditions~\cite{colson2007overview, kunapuli2008classification};
and 3) utilizing hypergradient-based techniques to iteratively approximate the (stochastic) gradient of the UL problem.
In the ML literature, hypergradient-based BLO approaches have gained the most attention in recent years also thanks to their exploitation of first-order oracles in ML models.
To date, numerous hypergradient-based bilevel optimization algorithms have been proposed, including but not limited to \cite{ghadimi2018approximation,ji2021bilevel,yang2021provably,sow2022convergence,yang2023achievingoepsilon15complexityhessianjacobianfree,arbel2021amortized,kwon2023fully,dagreou2022framework,hong2023two,khanduri2021near}, with applications in signal processing \cite{abdi2002space,crockett2022bilevel,chen2020stackelberg,sun2018learning}, adversarial training \cite{zuo2021adversarial,yang2021robust}, model-agnostic meta-learning \cite{finn2017model,arjovsky2019invariant}, efficient machine learning 
\cite{borsos2020coresets,zhang2022advancing}, and model pruning \cite{liu2020generic,liu2018darts}. 
However, these existing hypergradient-based approaches are all designed for single UL objective setting.
In this work, we generalize hypergradient BLO algorithms to the MOO settings.

\textbf{3) Detailed comparison:} We summarize the comparison between our \alg approach and some existing BLO and MOBL algorithms in \Cref{tab:comparison} as follows.

\begin{table}[htb]
\renewcommand{\arraystretch}{1}
\centering
\begin{tabular}{|>{\centering\arraybackslash}p{6.5cm}|>{\centering\arraybackslash}p{1.25cm}|>{\centering\arraybackslash}p{2cm}|
                >{\centering\arraybackslash}p{1.7cm}|>{\centering\arraybackslash}p{1.8cm}|>{\centering\arraybackslash}p{1.8cm}|}
\hline
\textbf{Algorithm} & \textbf{Problem} & \textbf{Scenario} & \textbf{Finite-time} & \textbf{Rate} & \textbf{Preference} \\
\hline
AID\cite{ji2021bilevel} & BLO & Deterministic & Yes & $\mathcal{O}(K^{-1})$ & - \\
ITD\cite{ji2021bilevel} & BLO & Deterministic & Yes & $\mathcal{O}(K^{-1})$ & - \\
SOBA\cite{dagreou2022framework} & BLO & Stochastic & Yes & $\mathcal{O}(K^{-\frac{1}{2}})$ & - \\
STABLE\cite{chen2022single} & BLO & Stochastic & Yes & $\mathcal{O}(K^{-\frac{1}{2}})$ & - \\
stocBiO\cite{ji2021bilevel} & BLO & Stochastic & Yes & $\mathcal{O}(K^{-1})$ & - \\
AmIGO\cite{arbel2021amortized} & BLO & Stochastic & Yes & $\mathcal{O}(K^{-1})$ & - \\
\hline
BLMS\cite{li2024bi} & MOBL & Deterministic & No & - & No \\
PSP-BLEMO\cite{wang2024pareto} & MOBL & Deterministic & No & - & No \\
MOML\cite{ye2021multi} & MOBL & Deterministic & No & - & No \\
gMOBA\cite{yang2024gradient} & MOBL & Deterministic & No & - & No \\
FORUM\cite{ye2024first} & MOBL & Deterministic & Yes & $\mathcal{O}(K^{-\frac{1}{4}})$ & No \\
MoCo\cite{fernando2022mitigating} & MOBL & Stochastic & Yes & $\mathcal{O}(SK^{-\frac{1}{2}})$ & No \\
\hline
\textbf{\alg}(\textbf{ours}) & \textbf{MOBL} & \textbf{Deterministic} & \textbf{Yes} & $\mathcal{O}(K^{-1})$ & \textbf{Yes} \\
\textbf{\alg}(\textbf{ours}) & \textbf{MOBL} & \textbf{Stochastic} & \textbf{Yes} & $\mathcal{O}(K^{-1})$ & \textbf{Yes} \\
\hline
\end{tabular}
\vspace{1em}
\caption{Comparison between \alg and existing BLO and MOBL algorithms. \textit{Problem} indicates whether the work considers BLO or MOBL problem, \textit{Scenario} specifies whether the setting is deterministic or stochastic, \textit{Finite-time} means whether the work provides a finite-time convergence guarantee or not, \textit{Rate} means the theoretical finite-time convergence rate (if exists), and \textit{Preference} means whether the work take user preference into account.}
\label{tab:comparison}
\end{table}

\section{Geometric interpretation of \alg design}\label{sec:app-geo}

\begin{figure}[H]
    \centering
    \includegraphics[width=0.6\linewidth]{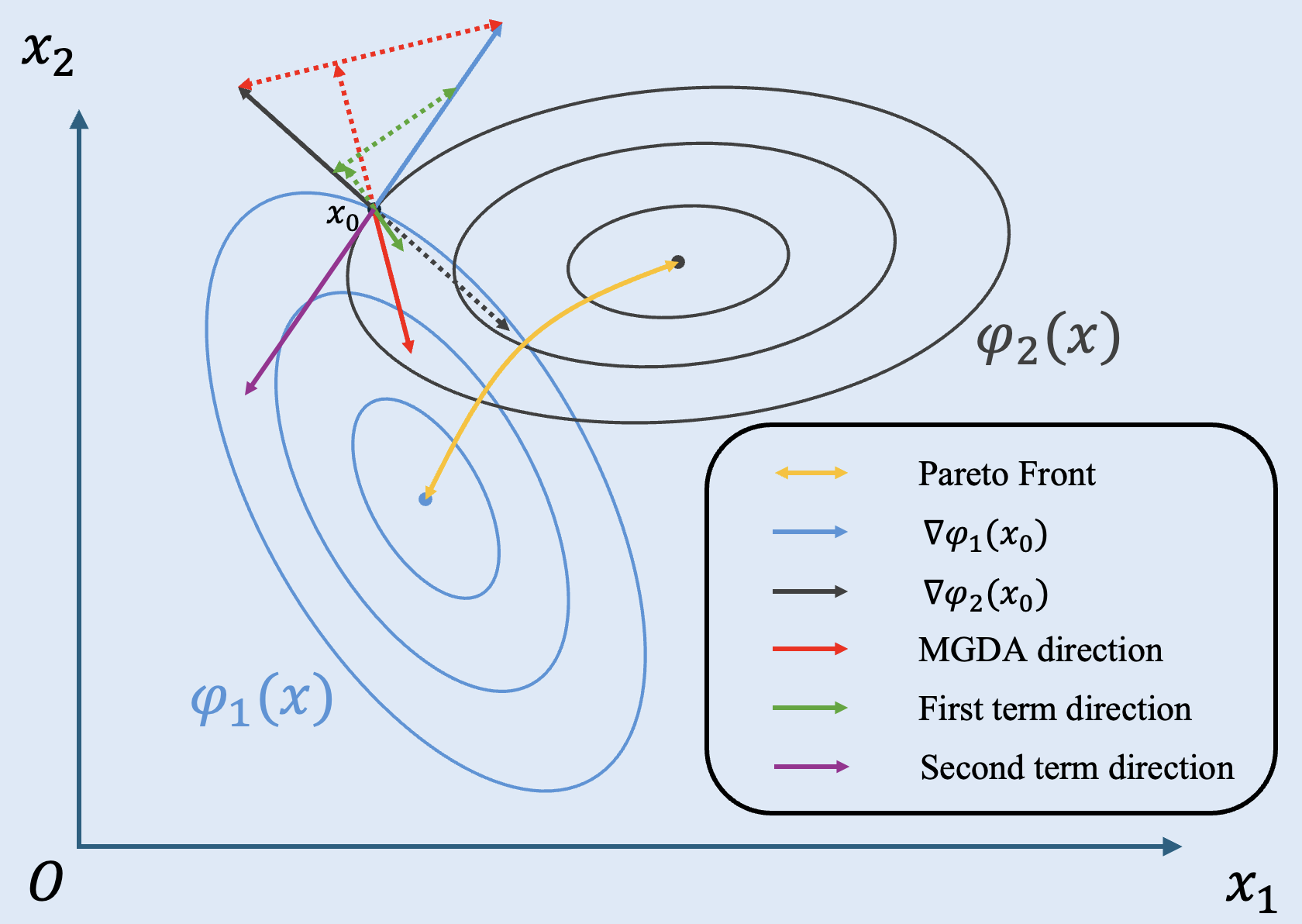}
    \caption{Geometric interpretation of Eq.~\eqref{eq:WC}.}
    \label{fig:WC}
    \vspace{-.2in}
\end{figure}

To better understand Eq.~\eqref{eq:WC}, we provide a {\em geometric interpretation} of illustrated in Fig.~\ref{fig:WC}. 
Consider a scenario with two objectives, represented by the blue and dark gray contours, respectively. 
Suppose $\phi_1(x)$ is more preferred, implying $r_1 > r_2$. 
We also assume that two objectives are of comparable scaling. 
We can observe the following facts: \textbf{(1)} At current $x_0$, MGDA (red arrow) minimizes $\|\lambda\cdot\nabla\phi_1(x_0) + (1-\lambda)\cdot\nabla\phi_2(x_0)\|^2$ to approach Pareto front, meaning it finds a point $x'$ on the line segment formed by the endpoints of $\nabla\phi_1(x_0)$ and $\nabla\phi_2(x_0)$ to minimize the length of $x'-x_0$. \textbf{(2)} Minimizing the first term of \Cref{eq:WC} corresponds to finding a point $x''$ on the green dashed line segment to minimize the length of $x''-x_0$, i.e., to minimize $\|\lambda \cdot r_1\nabla\phi_1(x_0) + (1-\lambda)\cdot r_2\nabla\phi_2(x_0))\|^2$. As a result, the generated direction (green arrow) also converges to Pareto front. \textbf{(3)} As for the second term, minimizing it is equivalent to maximizing $\lambda\cdot r_1\phi_1(x_0) + (1-\lambda)\cdot r_2\phi_2(x_0)$. Since two objectives share a similar scale, the preference implies $r_1\phi_1(x_0) > r_2\phi_2(x_0)$. Therefore, the maximization naturally leads to selecting $\lambda=1$ since $\lambda\in\Delta_2^+$. This emphasizes the purple arrow in \Cref{fig:WC}, which is known as the steepest descent direction for $\phi_1(x_0)$.

\textit{Remark 1.} As mentioned earlier, the primal WC problem aims to minimize the weighted loss in $l_{\infty}$-norm sense by minimizing $\gamma$, such that $r\odot\Phi(x) \le \gamma \textbf{1}$. We argue that the second term in \Cref{eq:WC} aligns with this idea. This term consistently points toward the steepest descent direction of the objective $s$ with maximal $r_s\phi_s(x)$, depicted as the purple arrow in \Cref{fig:WC}. Therefore, it effectively control the dominant term in $r\odot\Phi(x)$, thus continuously reduces its $l_{\infty}$-norm.

\textit{Remark 2.} Another remark is that the positive constant $u$ in \Cref{eq:WC} plays a role as a trade-off controller. Since $d_k$ is a combination of the green and purple arrows in \Cref{fig:WC}, it's obvious that a larger $u$ yields a solution more closely aligned with the preference; while a smaller $u$ ensures a tighter adherence to the stationarity condition.

\textit{Remark 3.} It is also worth noting that, for simplicity in this geometric interpretation, we replace hyper-gradients with true gradients. This simplification highlights that since inaccuracies caused by bilevel structures can lead to deviations in the descent direction, such deviations necessitate careful algorithm design to mitigate their impact.

\section{Theoretical proofs}\label{sec:app_analysis}

In this section, we begin by proving two key lemmas presented in \Cref{sec:alg} and \Cref{sec:analysis}. Next, we provide a detailed analysis of the two main theoretical results: \Cref{thm:WC_det} and \Cref{thm:WC_stoc}. Finally, we present an additional finding: if the MOBL problem is not provided by any user preference vector $r$, which is exactly the fourth type of philosophy mentioned in \Cref{sec:statement}, then this scenario can be treated as a special case of our \alg framework.

\subsection{Proof of \Cref{lemma:stationary}}\label{sec:pf-stationary-lemma}
\begin{proof}
    We first give the Pareto stationarity condition: If $\lambda$ satisfies the following conditions:
    \begin{equation}\label{eq:stationary}
        K\lambda=0, \textbf{1}^\top\lambda=1, \lambda\ge 0,
    \end{equation}
    then it provides Pareto stationarity \cite{desideri2012multiple,sener2018multi,momma2022multi}. According to our assumptions, $\lambda$, on the other hand, satisfies the following conditions:
    \begin{equation}
        K(r \odot\lambda)=K\text{diag}(r)\lambda=0, \textbf{1}^\top\lambda=1, \lambda\ge 0.
    \end{equation}
    We denote $\text{diag}(r)\lambda = \tilde{\lambda}$, which implies $K\tilde{\lambda}=0$ and $\tilde{\lambda}\ge 0$. Since $\lambda \ge 0$ and $\textbf{1}^\top\lambda=1$, there must exist at least one positive element $\lambda_i$ in $\lambda$. Moreover, since $r$ is strictly positive, we can further attain $\tilde{\lambda}_i > 0$. Now, we rescale as follows:
    \begin{equation}
        \hat{\lambda} = \frac{\tilde{\lambda}}{\textbf{1}^\top\tilde{\lambda}},
    \end{equation}
    which is valid since $\textbf{1}^\top\tilde{\lambda} \ge \tilde{\lambda}_i >0$. Therefore, we check Pareto stationarity condition \Cref{eq:stationary} on $\hat{\lambda}$ to end the proof:
    \begin{equation}
        K\hat{\lambda}=0, \textbf{1}^\top\hat{\lambda}=1, \hat{\lambda}\ge 0.
    \end{equation}
\end{proof}

\textit{Remark.} It is obvious that during the iterations prior to convergence, satisfying $K(r \odot\lambda)=0$ is almost impossible to be satisfied. To this end, we minimize $\|K(r \odot\lambda)\|^2$ to yield stationarity in \Cref{eq:WC}, which serves as a reasonable and practical surrogate of $K(r \odot\lambda)=0$ for achieving Pareto stationarity.

\subsection{Proof of \Cref{lemma:PGMNL}}

\begin{proof}
    For any feasible $\lambda$, we denote its corresponding direction as $d = \hat{\nabla}\Phi(x_k)(r\odot\lambda)$. Let $d = d_k +v$. It is obvious that \Cref{eq:WC} is a convex optimization problem. Therefore, for any $\epsilon\in[0,1]$, we know that the weight vector of $d_k + \epsilon v$ is also feasible, and its corresponding weight vector $\lambda(\epsilon) = \lambda_k + \epsilon(\lambda - \lambda_k)$, since:
    \begin{equation*}
        \begin{aligned}
            d_k + \epsilon v & = d_k + \epsilon (d-d_k) \\
            & = (1-\epsilon) d_k + \epsilon d \\
            & = (1-\epsilon) \hat{\nabla}\Phi(x_k)(r\odot\lambda_k) + \epsilon\hat{\nabla}\Phi(x_k)(r\odot\lambda) \\
            & = (1-\epsilon) \hat{\nabla}\Phi(x_k)\text{diag}(r)\lambda_k + \epsilon\hat{\nabla}\Phi(x_k)\text{diag}(r)\lambda \\
            & = \hat{\nabla}\Phi(x_k)\text{diag}(r)\left[\lambda_k + \epsilon(\lambda - \lambda_k)\right].
        \end{aligned}
    \end{equation*}

    Since $\lambda_k$ is the optima of \Cref{eq:WC}, we have:
    \begin{equation*}
        \|K(r\odot \lambda_k)\|^2 - u\lambda_k^\top (r\odot \Phi(x_k)) \le \|K(r\odot \lambda(\epsilon))\|^2 - u\lambda(\epsilon)^\top (r\odot \Phi(x_k)),
    \end{equation*}
    which, since $\|K\lambda\|^2 = \|\hat{\nabla}\Phi(x_k)\lambda\|^2$, is equivalent to:
    \begin{equation*}
        \|d_k\|^2 - u\lambda_k^\top (r\odot \Phi(x_k)) \le \|d_k+\epsilon v\|^2 - u(\lambda_k + \epsilon(\lambda - \lambda_k))^\top (r\odot \Phi(x_k)).
    \end{equation*}
    Thus, after decomposing the first term in RHS and rearranging the inequality, we obtain:
    \begin{equation*}
        2\epsilon\langle d_k, v\rangle + \epsilon^2 \|v\|^2 \ge \epsilon u (\lambda - \lambda_k)^\top (r\odot \Phi(x_k)).
    \end{equation*}
    From last inequality, we can observe that $(\lambda - \lambda_k)$ is bounded, thus, we can select a small enough $u$, such that $u (\lambda - \lambda_k)^\top (r\odot \Phi(x_k)) \le \|d_k\|^2$. Simultaneously, we let $\epsilon$ go to $0$ to get:
    \begin{equation}\label{eq:WC-lemma}
        2\langle d_k, d\rangle \ge \|d_k\|^2.
    \end{equation}

    Now, we note that for any $s\in[S]$, we have $\hat{\nabla}\phi_s(x_k) = \hat{\nabla}\Phi(x_k)e_s$, where $e_s$ is the vector with element $1$ in the $s$-th position and $0$ in all other positions. It is trivial that $e_s$ is feasible for \Cref{eq:WC} for any $s\in[S]$. Let $r_s$ be the $s$-th element of preference $r$, then we find the corresponding direction of $e_s$ has the following form:
    \begin{equation*}
        \hat{\nabla}\Phi(x_k)(r\odot e_s) = \hat{\nabla}\Phi(x_k) \text{diag}(r)e_s = r_s \hat{\nabla}\phi_s(x_k).
    \end{equation*}
    Substituting this result into \Cref{eq:WC-lemma}, we have:
    \begin{equation*}
        \|d_k\|^2 \le 2r_s \langle d_k, \hat{\nabla}\phi_s(x_k) \rangle.
    \end{equation*}
    Therefore, for any $s\in[S]$, we have: $\|d_k\|^2 \le 2r_{\text{max}} \langle d_k, \hat{\nabla}\phi_s(x_k) \rangle$.
\end{proof}

\subsection{Proof of \Cref{thm:WC_det}}\label{sec:app_analysis_det}

We prove \Cref{thm:WC_det} in this section. To make it more clear, we prove \Cref{thm:WC_det} under option NS and option CG separately.

\begin{proof}
    (\textbf{\Cref{thm:WC_det} under option NS}) We first give the following lemma, inspired by \cite{ji2021bilevel}, which quantifies the gap of true gradients and hyper-gradients when leveraging NS approximation. In other words, it ensures the error caused by bilevel structure can be well controlled, thus won't lead to more inaccuracy.
    \begin{lemma}\label{lemma:WC-NS}
        When using \Cref{alg:deterministic} with NS to compute $\hat{\nabla}\phi_s(x_k)$ with $\alpha \le \frac{1}{L}$, we have the following result for each $s\in\mathcal{S}$ and $k\in\{0,1,\dotsc,K-1\}$:
        \begin{equation*}
            \begin{aligned}
                & \|\nabla\phi_s(x_k) - \hat{\nabla}\phi_s(x_k)\|\\
                \le & \left(\frac{L(L+\mu_g)(1-\alpha\mu_g)^{\frac{D}{2}}}{\mu_g} +\frac{2M(\tau\mu_g+L\rho)}{\mu_g^2}(1-\alpha\mu_g)^{\frac{D-1}{2}}\right)\|y_k^0 - y^*(x_k)\| \\
                & + \frac{LM(1-\alpha\mu_g)^D}{\mu_g}. \\
            \end{aligned}
        \end{equation*}
    \end{lemma}

    From the \Cref{ass:Lipchitz}, we attain that the upper level function $\phi_s(x)$ is also smooth for every $s\in\mathcal{S}$, and we denote the positive Lipschitz constant as $L_{\phi} = L + \frac{2L^2 + \tau M^2}{\mu_g} + \frac{\rho LM + L^3 + \tau LM}{\mu_g^2} + \frac{\rho L^2M}{\mu_g^3} = \Theta(\kappa^3)$. Thus, we have the following descent lemma for each $s\in\mathcal{S}$:
    \begin{equation}\label{eq:descent_lemma}
        \phi_s(x_{k+1}) \le \phi_s(x_k) + \langle\nabla\phi_s(x_k), x_{k+1}-x_k\rangle + \frac{L_{\phi}}{2}\|x_{k+1}-x_k\|^2,
    \end{equation}
    where $k\in\{0,1,\dotsc, K-1\}$. Then ,we plug the update in \Cref{alg:deterministic} with NS and $d_k$ into \Cref{eq:descent_lemma} to get:
    \begin{equation}\label{eq:descent_lemma-2}
        \begin{aligned}
            \phi_s(x_{k+1}) & \le \phi_s(x_k) + \langle\nabla\phi_s(x_k), -\beta d_k\rangle + \frac{L_{\phi}}{2}\|-\beta d_k\|^2 \\
            & = \phi_s(x_k) + \langle\nabla\phi_s(x_k) - \hat{\nabla}\phi_s(x_k), -\beta d_k\rangle - \beta\langle\hat{\nabla}\phi_s(x_k), d_k\rangle + \frac{L_{\phi}}{2}\beta^2\|d_k\|^2 \\
            & \le \phi_s(x_k) + \frac{1}{2}\|\nabla\phi_s(x_k) - \hat{\nabla}\phi_s(x_k)\|^2 + \frac{\beta^2}{2}\| d_k\|^2 - \beta\langle\hat{\nabla}\phi_s(x_k), d_k\rangle + \frac{L_{\phi}}{2}\beta^2\|d_k\|^2,
        \end{aligned}
    \end{equation}
    where the last inequality is due to Cauchy–Schwarz inequality. Note that $d_k$ is computed by solving \Cref{eq:WC}, we leverage \Cref{lemma:PGMNL} to get the following inequality for every $s\in\mathcal{S}$:
    \begin{equation*}
        \phi_s(x_{k+1}) \le \phi_s(x_k)  + \frac{1}{2}\|\nabla\phi_s(x_k) - \hat{\nabla}\phi_s(x_k)\|^2 -\beta\left(\frac{1}{2r_{\text{max}}} - \frac{1}{2}\beta - \frac{L_{\phi}}{2}\beta\right)\|d_k\|^2,
    \end{equation*}
    which can be further rearranged as:
    \begin{equation}
        \beta\left(\frac{1}{2r_{\text{max}}} - \frac{1}{2}\beta - \frac{L_{\phi}}{2}\beta\right)\|d_k\|^2 \le \phi_s(x_k) - \phi_s(x_{k+1}) + \frac{1}{2}\|\nabla\phi_s(x_k) - \hat{\nabla}\phi_s(x_k)\|^2.
    \end{equation}

    Then, we telescope the last inequality for $k$ from $0$ to $K-1$, with \Cref{lemma:WC-NS}:
    \begin{equation}\label{eq:NS-hypre}
        \begin{aligned}
            & \beta\left(\frac{1}{2r_{\text{max}}} - \frac{1}{2}\beta - \frac{L_{\phi}}{2}\beta\right)\frac{1}{K}\sum_{k=0}^{K-1}\|d_k\|^2 \\
            \le & \frac{\phi_s(x_0) - \phi_s(x_{K})}{K} + \frac{1}{2K}\sum_{k=0}^{K-1}\|\nabla\phi_s(x_k) - \hat{\nabla}\phi_s(x_k)\|^2 \\
            \le & \frac{\phi_s(x_0) - \phi_s(x_{K})}{K} + \frac{L^2M^2(1-\alpha\mu_g)^{2D}}{\mu_g^2} \\
            & +\left(\frac{2L^2(L+\mu_g)^2(1-\alpha\mu_g)^D}{\mu_g^2} +\frac{4M^2(\tau\mu_g+L\rho)^2}{\mu_g^4}(1-\alpha\mu_g)^{D-1}\right)\chi,
        \end{aligned}
    \end{equation}
    where $\chi$ is a constant. On the other hand, we need to bound the true-gradient using the above result for hyper-gradient. Note the properties of $\lambda_k$ for any $k$: $\lambda_k^{(s)} \ge 0$, and $\sum_s \lambda_k^{(s)} = 1$ hold for any $k\in\{0,1,\dotsc,K-1\}$. We first apply the triangle inequality:
    \begin{equation}\label{eq:triangle-hypre-true-1}
        \begin{aligned}
            \|\bar{d}_k\|^2 & \le 2\|d_k\|^2 + 2\|\bar{d}_k - d_k\|^2 \\
            & = 2\|d_k\|^2 + 2\|\sum_{s=1}^S\lambda_k^{(s)}r_s\left(\nabla\phi_s(x_k) - \hat{\nabla}\phi_s(x_k)\right)\|^2 \\
            & \le 2\|d_k\|^2 + 6\max_s\|\nabla\phi_s(x_k) - \hat{\nabla}\phi_s(x_k)\|^2,
        \end{aligned}
    \end{equation}
    where the last inequality is due to the following reasons:
    \begin{equation}\label{eq:triangle-hypre-true-2}
        \begin{aligned}
            & \|\sum_{s=1}^S\lambda_k^{(s)}r_s\left(\nabla\phi_s(x_k) - \hat{\nabla}\phi_s(x_k)\right)\|^2 \\
            = & \sum_{s=1}^S{\lambda_k^{(s)}}^2r_s^2\|\nabla\phi_s(x_k) - \hat{\nabla}\phi_s(x_k)\|^2 \\
            & + 2\sum_{s_1\neq s_2}\lambda_k^{(s_1)}\lambda_k^{(s_2)}r_{s_1}r_{s_2}\|\nabla\phi_{s_1}(x_k) - \hat{\nabla}\phi_{s_1}(x_k)\|\cdot\|\nabla\phi_{s_2}(x_k) - \hat{\nabla}\phi_{s_2}(x_k)\| \\
            \le & \sum_{s=1}^S{\lambda_k^{(s)}}^2\|\nabla\phi_s(x_k) - \hat{\nabla}\phi_s(x_k)\|^2 \\
            & + 2\sum_{s_1\neq s_2}\lambda_k^{(s_1)}\lambda_k^{(s_2)}\|\nabla\phi_{s_1}(x_k) - \hat{\nabla}\phi_{s_1}(x_k)\|\cdot\|\nabla\phi_{s_2}(x_k) - \hat{\nabla}\phi_{s_2}(x_k)\| \\
            \le & \max_s \|\nabla\phi_s(x_k) - \hat{\nabla}\phi_s(x_k)\|^2 \sum_{s=1}^S{\lambda_k^{(s)}}^2 \\
            & + \max_s \|\nabla\phi_s(x_k) - \hat{\nabla}\phi_s(x_k)\|^2 2\sum_{s_1\neq s_2}\lambda_k^{(s_1)}\lambda_k^{(s_2)} \\
            \le & \max_s \|\nabla\phi_s(x_k) - \hat{\nabla}\phi_s(x_k)\|^2 \sum_{s=1}^S\lambda_k^{(s)} + 2\max_s \|\nabla\phi_s(x_k) - \hat{\nabla}\phi_s(x_k)\|^2 \\
            \le & 3\max_s \|\nabla\phi_s(x_k) - \hat{\nabla}\phi_s(x_k)\|^2,
        \end{aligned}
    \end{equation}
    where the first inequality is due to the conditions of preference vector $r$: since $r\in\mathbb{R}^S_{++}$, and $\textbf{1}^\top r = 1$, we know that each element $r_s$ is positive, and strictly less than $1$. Now, we combine \Cref{eq:NS-hypre} and \Cref{eq:triangle-hypre-true-1} to obtain:
    \begin{equation}
        \begin{aligned}
            & \frac{\beta\left(\frac{1}{2r_{\text{max}}} - \frac{1}{2}\beta - \frac{L_{\phi}}{2}\beta\right)}{K}\sum_{k=0}^{K-1}\|\bar{d}_k\|^2 \\
            \le & \frac{2\beta\left(\frac{1}{2r_{\text{max}}} - \frac{1}{2}\beta - \frac{L_{\phi}}{2}\beta\right)}{K}\sum_{k=0}^{K-1}\|d_k\|^2 + 6\beta\left(\frac{1}{2r_{\text{max}}} - \frac{1}{2}\beta - \frac{L_{\phi}}{2}\beta\right)\max_{s,k} \|\nabla\phi_s(x_k) - \hat{\nabla}\phi_s(x_k)\|^2 \\
            \le & \frac{2(\phi_s(x_0) - \phi_s(x_{K}))}{K} + \left(12\beta\left(\frac{1}{2r_{\text{max}}} - \frac{1}{2}\beta - \frac{L_{\phi}}{2}\beta\right) + 2\right)\cdot \\
            & \left[\left(\frac{2L^2(L+\mu_g)^2(1-\alpha\mu_g)^D}{\mu_g^2} +\frac{4M^2(\tau\mu_g+L\rho)^2}{\mu_g^4}(1-\alpha\mu_g)^{D-1}\right)\chi + \frac{L^2M^2(1-\alpha\mu_g)^{2D}}{\mu_g^2} \right].
        \end{aligned}
    \end{equation}

    Therefore, selecting $\beta = \min\{\frac{1}{2(1+L_{\phi})r_{\text{max}}}, \frac{1}{3L_{\phi}}\}$, we attain:
    \begin{equation}\label{eq:NS-stationary}
        \begin{aligned}
            & \frac{1}{K}\sum_{k=0}^{K-1}\|\bar{d}_k\|^2 \le \frac{24r_{\text{max}}L_{\phi}(\phi_s(x_0) - \phi_s(x_{K}))}{K} + (12+24r_{\text{max}}L_{\phi}) \cdot \\
            & \left[\left(\frac{2L^2(L+\mu_g)^2(1-\alpha\mu_g)^D}{\mu_g^2} +\frac{4M^2(\tau\mu_g+L\rho)^2}{\mu_g^4}(1-\alpha\mu_g)^{D-1}\right)\chi + \frac{L^2M^2(1-\alpha\mu_g)^{2D}}{\mu_g^2} \right].
        \end{aligned}
    \end{equation}
    
    As the preference $r$ satisfies $\textbf{1}^\top r=1$, which implies that $r_{\text{max}} < 1$, thus, with this condition, we can further re-write the result as:
    \begin{equation}
        \begin{aligned}
            & \frac{1}{K}\sum_{k=0}^{K-1}\|\bar{d}_k\|^2 \le \frac{24L_{\phi}(\phi_s(x_0) - \phi_s(x_{K}))}{K} + (12+24L_{\phi}) \cdot \\
            & \left[\left(\frac{2L^2(L+\mu_g)^2(1-\alpha\mu_g)^D}{\mu_g^2} +\frac{4M^2(\tau\mu_g+L\rho)^2}{\mu_g^4}(1-\alpha\mu_g)^{D-1}\right)\chi + \frac{L^2M^2(1-\alpha\mu_g)^{2D}}{\mu_g^2} \right].
        \end{aligned}
    \end{equation}
    
    In order to achieve an $\epsilon$-stationary Pareto point in \Cref{eq:NS-stationary}, we select parameters as follows:
    \begin{equation*}
        \begin{aligned}
            & \alpha \le \frac{1}{L}, \\
            & D \ge \frac{\log[(24+16r_{\text{max}}L_{\phi})(2\chi(\kappa^2(L+\mu_g)^2(1-\alpha\mu_g) + \frac{4M^2(\tau\mu_g+L\rho)^2)}{\mu_g^4}) + M^2\kappa^2(1-\alpha\mu_g)^{D+1})\epsilon^{-1}]}{\log(\frac{1}{1-\alpha\mu_g})} \\
        & = \Theta(\kappa\log\frac{1}{\epsilon}), \\
            & K = \Theta(\kappa^3\epsilon^{-1}),
        \end{aligned}
    \end{equation*}
    then, we have the following complexity results:
    \begin{equation*}
        \begin{aligned}
            & \mathsf{Gc}(f, \epsilon) = 2K\cdot S = \mathcal{O}(\kappa^3\epsilon^{-1}S), \\
            & \mathsf{Gc}(g, \epsilon) = KD = \widetilde{\mathcal{O}}(\kappa^4\epsilon^{-1}), \\
            & \mathsf{JV}(g, \epsilon) = KD\cdot S = \widetilde{\mathcal{O}}(\kappa^4\epsilon^{-1}S), \\
            & \mathsf{HV}(g, \epsilon) = KD\cdot S = \widetilde{\mathcal{O}}(\kappa^4\epsilon^{-1}S). \\
        \end{aligned}
    \end{equation*}
        
\end{proof}

Now, we prove the result of \Cref{alg:deterministic} with Conjugate-Gradient (CG) method. Following that, we discuss the insight of the theoretical result.

\begin{proof}
    (\textbf{\Cref{thm:WC_det} under option CG}) As we mentioned before, the insight of \Cref{alg:deterministic} equipped with CG and NS is the same. The main distinction comes from the different control of the gap between true gradients and hyper-gradients. Here, we first define some notations, and state a supporting lemma for the analysis. We define the following values:
    \begin{equation}
        \begin{aligned}
            & \Gamma = 3L^2 + \frac{3\tau^2M^2}{\mu_g^2} + 6L^2(1 + \sqrt{\kappa})^2(\kappa + \frac{\rho M}{\mu_g^2})^2, \\
            & \delta_{D,N} = \Gamma(1-\alpha\mu_g)^D + 6L^2\kappa\left(\frac{\sqrt{\kappa}-1}{\sqrt{\kappa}+1}\right)^{2N}, \\
            & \Omega = 4\beta^2\left(\kappa^2 + \frac{ML}{\mu_g^2} + \frac{ML\kappa}{\mu_g^2}\right)^2, \\
            & \chi_0 = \|y_0 - y^*(x_0)\|^2 + \|v_0 - v^*_0\|^2.
        \end{aligned}
    \end{equation}
    
    The following lemma, inspired by \cite{ji2021bilevel}, plays a similar role as \Cref{lemma:WC-NS}. It is worth noting that even if it shares the same form with \textit{Lemma 5} in \cite{ji2021bilevel}, the details are quite different: Since MOO with preference vector $r$ is introduced here, the analysis should carefully handle these distinctions. This is reflected in the fact that the choices of $D$ and $N$ are different.
    \begin{lemma}\label{lemma:WC-CG}
        Suppose \Cref{ass:convex} and \Cref{ass:Lipchitz} hold. Selecting $D$, $N$ large enough (will be detailed described later), for any $s\in\mathcal{S}$, and $k\in\{0,1,\dotsc,K-1\}$, we have the following inequality when computing $\hat{\nabla}\phi_s(x_k)$ using \Cref{alg:deterministic} with CG:
        \begin{equation*}
            \|\nabla\phi_s(x_k) - \hat{\nabla}\phi_s(x_k)\|^2 \le \delta_{D,N}\left(\left(\frac{1}{2}\right)^k \chi_0 + \Omega\sum_{j=0}^{k-1}\left(\frac{1}{2}\right)^{k-1-j}\|\nabla\phi_s(x_j)\|^2\right).
        \end{equation*}
    \end{lemma}

    We follow \Cref{eq:descent_lemma} to get the similar result here for any $s\in[S]$ and $k\in\{0,\dotsc,K-1\}$:
    \begin{equation}
        \begin{aligned}
            \phi_s(x_{k+1}) & \le \phi_s(x_k) + \langle\nabla\phi_s(x_k), -\beta d_k\rangle + \frac{L_{\phi}}{2}\|-\beta d_k\|^2 \\
            & = \phi_s(x_k) + \langle\nabla\phi_s(x_k) - \hat{\nabla}\phi_s(x_k), -\beta d_k\rangle - \beta\langle\hat{\nabla}\phi_s(x_k), d_k\rangle + \frac{L_{\phi}}{2}\beta^2\|d_k\|^2 \\
            & \le \phi_s(x_k) + \frac{1}{2}\|\nabla\phi_s(x_k) - \hat{\nabla}\phi_s(x_k)\|^2 + \frac{\beta^2}{2}\| d_k\|^2 - \beta\langle\hat{\nabla}\phi_s(x_k), d_k\rangle + \frac{L_{\phi}}{2}\beta^2\|d_k\|^2
        \end{aligned}
    \end{equation}
    Plug \Cref{lemma:PGMNL} into this inequality and rearrange the terms, then we have:
    \begin{equation}
        \beta\left(\frac{1}{2r_{\text{max}}} - \frac{1}{2}\beta - \frac{L_{\phi}}{2}\beta\right)\|d_k\|^2 \le \phi_s(x_k) - \phi_s(x_{k+1}) + \frac{1}{2}\|\nabla\phi_s(x_k) - \hat{\nabla}\phi_s(x_k)\|^2.
    \end{equation}
    Now, we can apply \Cref{lemma:WC-CG} to upper bound LHS:
    \begin{equation}
        \begin{aligned}
            & \beta\left(\frac{1}{2r_{\text{max}}} - \frac{1}{2}\beta - \frac{L_{\phi}}{2}\beta\right)\sum_{k=0}^{K-1}\|d_k\|^2 \\
            \le & \phi_s(x_0) - \phi_s(x_{K}) + \frac{1}{2} \delta_{D,N}\sum_{k=0}^{K-1}\left(\left(\frac{1}{2}\right)^k \chi_0 + \Omega\sum_{j=0}^{k-1}\left(\frac{1}{2}\right)^{k-1-j}\|\nabla\phi_s(x_j)\|^2\right) \\
            \le & \phi_s(x_0) - \phi_s(x_{K}) + \delta_{D,N}\chi_0 + \delta_{D,N}\Omega\sum_{k=0}^{K-1}\sum_{j=0}^{k-1}\left(\frac{1}{2}\right)^{k-j}\|\nabla\phi_s(x_j)\|^2 \\
            \le & \phi_s(x_0) - \phi_s(x_{K}) + \delta_{D,N}\chi_0 + \delta_{D,N}\Omega\sum_{k=0}^{K-1} \|\nabla\phi_s(x_k)\|^2 \\
            \le & \phi_s(x_0) - \phi_s(x_{K}) + \delta_{D,N}\chi_0 + 2KM^2\delta_{D,N}\Omega(\kappa^2+1),
        \end{aligned}
    \end{equation}
    where the last inequality is due to \Cref{eq:stoc-proof-2}. Therefore, select $\beta = \min\{\frac{1}{2(1+L_{\phi})r_{\text{max}}}, \frac{1}{3L_{\phi}}\}$, then we have:
    \begin{equation}
        \frac{1}{K}\sum_{k=0}^{K-1}\|d_k\|^2 \le \frac{12r_{\text{max}}L_{\phi}\Big(\phi_s(x_0) - \phi_s(x_{K}) + \delta_{D,N}\chi_0\Big)}{K} + 24r_{\text{max}}L_{\phi}M^2\delta_{D,N}\Omega(\kappa^2+1).
    \end{equation}

    Now, by \Cref{eq:triangle-hypre-true-1}, we have:
    \begin{equation}
        \begin{aligned}
            & \frac{1}{K}\sum_{k=0}^{K-1}\|\bar{d}_k\|^2 \\
            \le & \frac{2}{K}\sum_{k=0}^{K-1}\|d_k\|^2 + 6\max_{s,k} \|\nabla\phi_s(x_k) - \hat{\nabla}\phi_s(x_k)\|^2 \\
            \le & \frac{2}{K}\sum_{k=0}^{K-1}\|d_k\|^2 + 6\delta_{D,N}\max_{s,k}\left(\left(\frac{1}{2}\right)^k \chi_0 + \Omega\sum_{j=0}^{k-1}\left(\frac{1}{2}\right)^{k-1-j}\|\nabla\phi_s(x_j)\|^2\right) \\
            \le & \frac{2}{K}\sum_{k=0}^{K-1}\|d_k\|^2 + 6\delta_{D,N}\chi_0 + 24M^2\delta_{D,N}\Omega(\kappa^2+1) \\
            \le & \frac{12r_{\text{max}}L_{\phi}\Big(\phi_s(x_0) - \phi_s(x_{K}) + \delta_{D,N}\chi_0\Big)}{K} + 24(r_{\text{max}}L_{\phi}+1)M^2\delta_{D,N}\Omega(\kappa^2+1) + 6\delta_{D,N}\chi_0.
        \end{aligned}
    \end{equation}

    The desired result can be attained since we have $\kappa^2+1 \le \sqrt{\Omega}$. In order to achieve an $\epsilon$-stationary Pareto point, we select parameters as follows:
    \begin{equation*}
        \begin{aligned}
            & \alpha \le \frac{1}{L}, \\
            & D \ge \frac{\log[24\Gamma(\chi_0+4M^2\Omega(\kappa^2+1)(r_{\text{max}}L_{\phi}+1))\epsilon^{-1}]}{\log(\frac{1}{1-\alpha\mu_g})} = \Theta(\kappa\log\frac{1}{\epsilon}), \\
            & N \ge \frac{\log[144L^2\kappa(\chi_0+4M^2\Omega(\kappa^2+1)(r_{\text{max}}L_{\phi}+1))\epsilon^{-1}]}{2\log(\frac{\sqrt{\kappa}+1}{\sqrt{\kappa}-1})} = \Theta(\sqrt{\kappa}\log\frac{1}{\epsilon}), \\
            & K = \Theta(\kappa^3\epsilon^{-1}),
        \end{aligned}
    \end{equation*}
    then, we have the following complexity results:
    \begin{equation*}
        \begin{aligned}
            & \mathsf{Gc}(f, \epsilon) = 2K\cdot S = \mathcal{O}(\kappa^3\epsilon^{-1}S), \\
            & \mathsf{Gc}(g, \epsilon) = KD = \widetilde{\mathcal{O}}(\kappa^4\epsilon^{-1}), \\
            & \mathsf{JV}(g, \epsilon) = K\cdot S = \mathcal{O}(\kappa^3\epsilon^{-1}S), \\
            & \mathsf{HV}(g, \epsilon) = KN\cdot S = \widetilde{\mathcal{O}}(\kappa^{3.5}\epsilon^{-1}S). \\
        \end{aligned}
    \end{equation*}
    
\end{proof}

We observe that \Cref{alg:deterministic} using CG achieves the same convergence rate as \Cref{alg:deterministic} with the NS approximation. Furthermore, as demonstrated in complexities of two methods, the oracle complexity of the CG method is even more favorable. However, we claim that both approaches offer distinct advantages in different aspects: on the one hand, CG achieves greater efficiency in terms of oracle complexity; on the other hand, \cite{xiao2023communication} highlights that the NS approximation is more communication-efficient in federated settings.

\subsection{Proof of \Cref{thm:WC_stoc}}\label{sec:app_analysis_stoc}

\begin{proof}
    For stochastic case, we first introduce some notations and supporting lemmas for the proof. Specifically, we define:
    \begin{equation}
        \begin{aligned}
            & \theta = 2\left(\frac{L-\mu_g}{L+\mu_g}\right)^{2D}, \\
            & \nu = \left(L + \frac{L}{\mu_g} + \frac{M\tau}{\mu_g} + \frac{LM\rho}{\mu_g^2}\right)^2, \\
            & \gamma = 8\theta\frac{\beta^2 L^2\nu}{\mu_g^2}\frac{r_{\text{max}}}{1-r_{\text{max}}}, \\
            & \chi = \frac{4L^2M^2}{\mu_g^2D_g} + \left(\frac{8L^2}{\mu_g^2}+2\right)\frac{M^2}{D_f} + \frac{16\eta^2L^4M^2}{B\mu_g^2} + \frac{16L^2M^2(1-\eta\mu_g)^{2Q}}{\mu_g^2}.
        \end{aligned}
    \end{equation}
    
    We first introduce two lemmas. The first lemma is inspired by \cite{ji2021bilevel}:
    \begin{lemma}\label{lemma:stoc_1}
        Suppose \Cref{ass:convex}, \Cref{ass:Lipchitz}, and \Cref{ass:variance} hold. Then, we have:
        \begin{equation*}
            \mathbb{E}\|\nabla\phi_s(x_k) - \hat{\nabla}\phi_s(x_k)\|^2 \le \chi + \nu \mathbb{E}\|y_k^D - y^*(x_k)\|^2.
        \end{equation*}
    \end{lemma}

    The second lemma is stated and proved as follows:
    \begin{lemma}\label{lemma:stoc_2}
        Suppose \Cref{ass:convex}, \Cref{ass:Lipchitz}, and \Cref{ass:variance} hold. Selecting $D$ such that $\theta < 1$ and setting $\alpha = \frac{2}{L+\mu_g}$, we have the following inequality for any $s\in[S]$:
        \begin{equation*}
            \mathbb{E}\|y_k^D - y^*(x_k)\|^2 \le \gamma^k \mathbb{E}\|y_0 - y^*(x_0)\|^2 + \frac{\gamma}{4\nu}\sum_{j=0}^{k-1}\gamma^{k-1-j}\mathbb{E}\|\nabla\phi_s(x_j)\|^2 + \frac{1}{1-\gamma}\cdot\left(\frac{\sigma^2}{L\mu_gT} + \frac{\gamma\chi}{4\nu}\right).
        \end{equation*}
    \end{lemma}

    \begin{proof}
        From \cite{ji2021bilevel}, we have the following inequality when selecting $\alpha=\frac{2}{L + \mu_g}$:
        \begin{equation}\label{eq:stoc-lemma-1}
            \mathbb{E}\|y_k^D - y^*(x_k)\|^2 \le \left(\frac{L-\mu_g}{L+\mu_g}\right)^{2D} \mathbb{E}\|y_{k-1}^D - y^*(x_k)\|^2 + \frac{\sigma^2}{L\mu_g T}.
        \end{equation}

        To find the relationship between $\mathbb{E}\|y_k^D - y^*(x_k)\|^2$ and $\mathbb{E}\|y_{k-1}^D - y^*(x_{k-1})\|^2$, we focus on the first term of the RHS in the last inequality. For any $s\in[S]$, we have:
        \begin{equation}\label{eq:stoc-lemma-2}
            \begin{aligned}
                & \mathbb{E}\|y_{k-1}^D - y^*(x_k)\|^2 \\
                \overset{(a)}{\le} & 2 \mathbb{E}\|y_{k-1}^D - y^*(x_{k-1})\|^2 + \frac{2\beta^2L^2}{\mu_g^2}\mathbb{E}\|d_{k-1}\|^2 \\
                \overset{(b)}{\le} & 2 \mathbb{E}\|y_{k-1}^D - y^*(x_{k-1})\|^2 + \frac{2\beta^2L^2}{\mu_g^2}\frac{r_{\text{max}}}{1-r_{\text{max}}}\mathbb{E}\|\hat{\nabla}\phi_s(x_{k-1})\|^2 \\
                \overset{(c)}{\le} & 2 \mathbb{E}\|y_{k-1}^D - y^*(x_{k-1})\|^2 + \frac{4\beta^2L^2}{\mu_g^2}\frac{r_{\text{max}}}{1-r_{\text{max}}}\mathbb{E}\|\nabla\phi_s(x_{k-1})\|^2 + \frac{4\beta^2L^2}{\mu_g^2}\frac{r_{\text{max}}}{1-r_{\text{max}}}\mathbb{E}\|\hat{\nabla}\phi_s(x_{k-1}) - \nabla\phi_s(x_{k-1})\|^2 \\
                \overset{(d)}{\le} & 2 \mathbb{E}\|y_{k-1}^D - y^*(x_{k-1})\|^2 + \frac{4\beta^2L^2}{\mu_g^2}\frac{r_{\text{max}}}{1-r_{\text{max}}}\mathbb{E}\|\nabla\phi_s(x_{k-1})\|^2 + \frac{4\beta^2L^2}{\mu_g^2}\frac{r_{\text{max}}}{1-r_{\text{max}}} (\chi + \nu \mathbb{E}\|y_{k-1}^D - y^*(x_{k-1})\|^2) \\
                = & \left(2 + \frac{4\beta^2L^2\nu}{\mu_g^2}\frac{r_{\text{max}}}{1-r_{\text{max}}}\right) \mathbb{E}\|y_{k-1}^D - y^*(x_{k-1})\|^2 + \frac{4\beta^2L^2}{\mu_g^2}\frac{r_{\text{max}}}{1-r_{\text{max}}}(\mathbb{E}\|\nabla\phi_s(x_{k-1})\|^2 + \chi).
            \end{aligned}
        \end{equation}
        where $(a)$ is from \cite{ghadimi2018approximation}, $(b)$ is due to \Cref{lemma:PGMNL}, $(c)$ leverages triangle inequality, $(d)$ applies \Cref{lemma:stoc_1}. Therefore, we combine \Cref{eq:stoc-lemma-1,eq:stoc-lemma-2} to get:
        \begin{equation}
            \begin{aligned}
                & \mathbb{E}\|y_k^D - y^*(x_k)\|^2 \\
                \le & \gamma \mathbb{E}\|y_{k-1}^D - y^*(x_{k-1})\|^2 + \frac{\gamma}{4\nu}(\mathbb{E}\|\nabla\phi_s(x_{k-1})\|^2 + \chi) + \frac{\sigma^2}{L\mu_g T} \\
                \le & \gamma^k \mathbb{E}\|y_0 - y^*(x_0)\|^2 + \frac{\gamma}{4\nu}\sum_{j=0}^{k-1}\gamma^{k-1-j}\mathbb{E}\|\nabla\phi_s(x_j)\|^2 + \frac{1}{1-\gamma}\cdot\left(\frac{\sigma^2}{L\mu_gT} + \frac{\gamma\chi}{4\nu}\right),
            \end{aligned} 
        \end{equation}
        where the last inequality is due to telescoping, and this ends the proof of lemma.
    \end{proof}

    In the beginning the proof process for \Cref{thm:WC_stoc}, we give the closed form of $v_k^{(s),Q}$ for any $s$ and $k$:
    \begin{equation}
        v_k^{(s),Q} = \eta\sum_{q=-1}^{Q-1}\prod_{j=Q-q}^Q(I - \eta\nabla_y^2G(x_k,y_k^D;\mathcal{B}_j))v_k^{(s),0}.
    \end{equation}
    
    Now, we leverage the descent lemma which comes from \Cref{ass:Lipchitz} to get the following result for each $s\in[S]$, and $k\in[K]$:
    \begin{equation}
        \begin{aligned}
            \phi_s(x_{k+1}) & \le \phi_s(x_k) + \langle\nabla\phi_s(x_k), -\beta d_k\rangle + \frac{L_{\phi}}{2}\|-\beta d_k\|^2 \\
            & = \phi_s(x_k) + \langle\nabla\phi_s(x_k) - \hat{\nabla}\phi_s(x_k), -\beta d_k\rangle - \beta\langle\hat{\nabla}\phi_s(x_k), d_k\rangle + \frac{L_{\phi}}{2}\beta^2\|d_k\|^2 \\
            & \le \phi_s(x_k) + \frac{1}{2}\|\nabla\phi_s(x_k) - \hat{\nabla}\phi_s(x_k)\|^2 + \frac{\beta^2}{2}\| d_k\|^2 - \beta\langle\hat{\nabla}\phi_s(x_k), d_k\rangle + \frac{L_{\phi}}{2}\beta^2\|d_k\|^2.
        \end{aligned}
    \end{equation}
    Note that the update direction $d_k = \hat{\nabla}\Phi(x_k)\text{diag}(r)\lambda_k$ satisfies \Cref{lemma:PGMNL}, thus, we have:
    \begin{equation}
        \phi_s(x_{k+1}) \le \phi_s(x_k)  + \frac{1}{2}\|\nabla\phi_s(x_k) - \hat{\nabla}\phi_s(x_k)\|^2 -\beta\left(\frac{1}{2r_{\text{max}}} - \frac{1}{2}\beta - \frac{L_{\phi}}{2}\beta\right)\|d_k\|^2.
    \end{equation}
    Taking expectation on both sides, and rearranging the inequality, we attain:
    \begin{equation}
        \beta\left(\frac{1}{2r_{\text{max}}} - \frac{1}{2}\beta - \frac{L_{\phi}}{2}\beta\right)\mathbb{E}\|d_k\|^2 \le \mathbb{E}(\phi_s(x_k)) - \mathbb{E}(\phi_s(x_{k+1})) + \frac{1}{2}\mathbb{E}\|\nabla\phi_s(x_k) - \hat{\nabla}\phi_s(x_k)\|^2.
    \end{equation}
    Now, we plug \Cref{lemma:stoc_1,lemma:stoc_2} into the RHS of this inequality:
    \begin{equation}\label{eq:stoc-proof-1}
        \begin{aligned}
            & \beta\left(\frac{1}{2r_{\text{max}}} - \frac{1}{2}\beta - \frac{L_{\phi}}{2}\beta\right)\mathbb{E}\|d_k\|^2 \\
            \le & \mathbb{E}(\phi_s(x_k)) - \mathbb{E}(\phi_s(x_{k+1})) + \frac{1}{2}(\chi + \nu \mathbb{E}\|y_k^D - y^*(x_k)\|^2) \\
            \le & \mathbb{E}(\phi_s(x_k)) - \mathbb{E}(\phi_s(x_{k+1})) + \frac{\nu}{2}\gamma^k\mathbb{E}\|y_0 - y^*(x_0)\|^2 + \frac{1}{8}\sum_{j=0}^{K-1}\gamma^{k-j}\mathbb{E}\|\nabla\phi_s(x_j)\|^2 + \left(\frac{4+\gamma}{8}\chi + \frac{\nu}{1-\gamma}\frac{\sigma^2}{2L\mu_g T}\right).
        \end{aligned}
    \end{equation}
    Note the following facts:
    \begin{equation}\label{eq:stoc-proof-2}
        \begin{aligned}
            &  \mathbb{E}\|\nabla\phi_s(x_k)\|^2 \\
            = & \mathbb{E}\|\nabla_x f^{(s)}(x_k, y^*(x_k)) - \nabla_{xy}^2g(x_k, y^*(x_k))v_k^*\|^2 \\
            = & \mathbb{E}\|\nabla_x f^{(s)}(x_k, y^*(x_k)) - \nabla_{xy}^2g(x_k, y^*(x_k))(\nabla_y^2g(x_k, y^*(x_k)))^{-1}\nabla_yf(x_k, y^*(x_k))\|^2 \\
            \le & 2 \mathbb{E}\|\nabla_x f^{(s)}(x_k, y^*(x_k))\|^2 + 2 \mathbb{E} (\|\nabla_{xy}^2g(x_k, y^*(x_k))\|^2 \cdot \|(\nabla_y^2g(x_k, y^*(x_k)))^{-1}\nabla_yf(x_k, y^*(x_k))\|^2) \\
            \le & 2M^2 (1+\frac{L^2}{\mu_g^2}),
        \end{aligned}
    \end{equation}
    holds for each $s$ and $k$, which is due to \Cref{ass:convex}, \Cref{ass:Lipchitz} and \Cref{ass:variance}, and
    \begin{equation*}
        \sum_{k=1}^{K}\sum_{j=0}^{k-1}\gamma^{k-j-1}\mathbb{E}\|\nabla\phi_s(x_j)\|^2 \le \frac{1}{1-\gamma}\sum_{k=1}^{K}\mathbb{E}\|\nabla\phi_s(x_k)\|^2.
    \end{equation*}
    Thus, we telescope \Cref{eq:stoc-proof-1} for $k$ from $1$ to $K$, then we have:
    \begin{equation}
        \begin{aligned}
            & \beta\left(\frac{1}{2r_{\text{max}}} - \frac{1}{2}\beta - \frac{L_{\phi}}{2}\beta\right) \frac{1}{K}\sum_{k=1}^K\mathbb{E}\|d_k\|^2 \\
            \le & \frac{\mathbb{E}(\phi_s(x_0)) - \mathbb{E}(\phi_s(x_{K}))}{K} + \frac{\nu\gamma}{2(1-\gamma)}\mathbb{E}\|y_0 - y^*(x_0)\|^2 + \left(\frac{4+\gamma}{8}\chi + \frac{\nu}{1-\gamma}\frac{\sigma^2}{2L\mu_g T}\right) + \frac{\gamma M^2(1+\kappa^2)}{4(1-\gamma)}.
        \end{aligned}
    \end{equation}

    Select $\beta \le \frac{1}{2(1+L_{\phi})r_{\text{max}}}$, then we have $\frac{1}{2r_{\text{max}}} - \frac{1}{2}\beta - \frac{L_{\phi}}{2}\beta \ge \frac{1}{4r_{\text{max}}}$. Therefore, we have:
    \begin{equation}\label{eq:stoc-proof-3}
        \begin{aligned}
            & \frac{1}{K}\sum_{k=1}^K\mathbb{E}\|d_k\|^2 \le \frac{4r_{\text{max}}}{\beta} \cdot \\
            & \Bigg[\frac{\mathbb{E}(\phi_s(x_0)) - \mathbb{E}(\phi_s(x_{K}))}{K} + \frac{\nu\gamma}{2(1-\gamma)}\mathbb{E}\|y_0 - y^*(x_0)\|^2 + \left(\frac{4+\gamma}{8}\chi + \frac{\nu}{1-\gamma}\frac{\sigma^2}{2L\mu_g T}\right) + \frac{\gamma M^2(1+\kappa^2)}{4(1-\gamma)}\Bigg].
        \end{aligned}
    \end{equation}
    To upper bound $\frac{1}{K}\sum_{k=1}^K\mathbb{E}\|\bar{d}_k\|^2$, we bridge the final result using triangle inequality:
    \begin{equation}
        \begin{aligned}
            \mathbb{E}\|\bar{d}_k\|^2 & \le 2\mathbb{E}\|d_k\|^2 + 2\mathbb{E}\|\bar{d}_k - d_k\|^2 \\
            & = 2\mathbb{E}\|d_k\|^2 + 2\mathbb{E}\|\sum_{s=1}^S\lambda_k^{(s)}r_s\left(\nabla\phi_s(x_k) - \hat{\nabla}\phi_s(x_k)\right)\|^2 \\
            & = 2\mathbb{E}\|d_k\|^2 + 2\mathbb{E}\Bigg[ \sum_{s=1}^S{\lambda_k^{(s)}}^2r_s^2\|\nabla\phi_s(x_k) - \hat{\nabla}\phi_s(x_k)\|^2 \\
            &\textbf{  }+ 2\sum_{s_1\neq s_2}\lambda_k^{(s_1)}\lambda_k^{(s_2)}r_{s_1}r_{s_2}\|\nabla\phi_{s_1}(x_k) - \hat{\nabla}\phi_{s_1}(x_k)\|\cdot\|\nabla\phi_{s_2}(x_k) - \hat{\nabla}\phi_{s_2}(x_k)\|\Bigg] \\
            & \le 2\mathbb{E}\|d_k\|^2 + 2 \mathbb{E}\Bigg[\max_s \|\nabla\phi_s(x_k) - \hat{\nabla}\phi_s(x_k)\|^2\sum_{s=1}^S{\lambda_k^{(s)}}^2\Bigg] \\
            &\textbf{  }+ 4 \mathbb{E}\Bigg[\max_s \|\nabla\phi_s(x_k) - \hat{\nabla}\phi_s(x_k)\|^2 2\sum_{s_1\neq s_2}\lambda_k^{(s_1)}\lambda_k^{(s_2)}\Bigg] \\
            & \le 2\mathbb{E}\|d_k\|^2 + 6 \mathbb{E}\Bigg(\max_s \|\nabla\phi_s(x_k) - \hat{\nabla}\phi_s(x_k)\|^2 \Bigg).
        \end{aligned}
    \end{equation}
    Thus,
    \begin{equation}
        \frac{1}{K}\sum_{k=1}^K\mathbb{E}\|\bar{d}_k\|^2 \le \frac{2}{K}\sum_{k=1}^K\mathbb{E}\|d_k\|^2 + 6\mathbb{E}\left(\max_{s,k}\|\nabla\phi_s(x_k) - \hat{\nabla}\phi_s(x_k)\|^2\right).
    \end{equation}

    Therefore, we combine \Cref{lemma:stoc_1,lemma:stoc_2} and \Cref{eq:stoc-proof-2,eq:stoc-proof-3} together to get:
    \begin{equation}
        \begin{aligned}
            & \frac{1}{K}\sum_{k=1}^K\mathbb{E}\|\bar{d}_k\|^2 \\
            \le & \frac{2}{K}\sum_{k=1}^K\mathbb{E}\|d_k\|^2 + 6\mathbb{E}\left(\max_{s,k}\|\nabla\phi_s(x_k) - \hat{\nabla}\phi_s(x_k)\|^2\right) \\
            \le & \frac{8r_{\text{max}}}{\beta} \Bigg[\frac{\mathbb{E}(\phi_s(x_0)) - \mathbb{E}(\phi_s(x_{K}))}{K} + \frac{\nu\gamma}{2(1-\gamma)}\mathbb{E}\|y_0 - y^*(x_0)\|^2 + \left(\frac{4+\gamma}{8}\chi + \frac{\nu}{1-\gamma}\frac{\sigma^2}{2L\mu_g T}\right) + \frac{\gamma M^2(1+\kappa^2)}{4(1-\gamma)}\Bigg] \\
            & + 6\chi + 6\nu \max_{k,s}\Bigg(\gamma^k \mathbb{E}\|y_0 - y^*(x_0)\|^2 + \frac{\gamma}{4\nu}\sum_{j=0}^{k-1}\gamma^{k-1-j}\mathbb{E}\|\nabla\phi_s(x_j)\|^2 + \frac{1}{1-\gamma}\cdot\left(\frac{\sigma^2}{L\mu_gT} + \frac{\gamma\chi}{4\nu}\right)\Bigg) \\
            \le & \frac{8r_{\text{max}}}{\beta} \Bigg[\frac{\mathbb{E}(\phi_s(x_0)) - \mathbb{E}(\phi_s(x_{K}))}{K} + \frac{\nu\gamma}{2(1-\gamma)}\mathbb{E}\|y_0 - y^*(x_0)\|^2 + \left(\frac{4+\gamma}{8}\chi + \frac{\nu}{1-\gamma}\frac{\sigma^2}{2L\mu_g T}\right) + \frac{\gamma M^2(1+\kappa^2)}{4(1-\gamma)}\Bigg] \\
            & + 6\chi + 6\nu\gamma \mathbb{E}\|y_0 - y^*(x_0)\|^2 + \frac{6\nu}{1-\gamma}\cdot\left(\frac{\sigma^2}{L\mu_gT} + \frac{\gamma\chi}{4\nu}\right) + \frac{3\gamma}{2(1-\gamma)}\max_{k,s}\mathbb{E}\|\nabla\phi_s(x_k)\|^2 \\
            \le & \frac{8r_{\text{max}}}{\beta} \Bigg[\frac{\mathbb{E}(\phi_s(x_0)) - \mathbb{E}(\phi_s(x_{K}))}{K} + \frac{\nu\gamma}{2(1-\gamma)}\mathbb{E}\|y_0 - y^*(x_0)\|^2 + \left(\frac{4+\gamma}{8}\chi + \frac{\nu}{1-\gamma}\frac{\sigma^2}{2L\mu_g T}\right) + \frac{\gamma M^2(1+\kappa^2)}{4(1-\gamma)}\Bigg] \\
            & + 6\chi + 6\nu\gamma \mathbb{E}\|y_0 - y^*(x_0)\|^2 + \frac{6\nu}{1-\gamma}\cdot\left(\frac{\sigma^2}{L\mu_gT} + \frac{\gamma\chi}{4\nu}\right) + \frac{3M^2(\kappa^2+1)\gamma}{1-\gamma}.
        \end{aligned}
    \end{equation}
    We select $\beta = \min\{\frac{1}{2(1+L_{\phi})r_{\text{max}}}, \frac{1}{3L_{\phi}}\}$, then this result can be further written using order notations as follows:
    \begin{equation}
        \begin{aligned}
            \frac{1}{K}\sum_{k=1}^K \mathbb{E}\|\bar{d}_k\|^2 \le  & \mathcal{O}\Bigg(\frac{L_{\phi}}{K} + \frac{\kappa^8\sigma^2}{T} + \frac{\kappa^5}{D_g} + \frac{\kappa^5}{D_f} + \frac{\kappa^5}{B} + \kappa^5(1-\eta\mu_g)^{2Q}\Bigg) \\
            & + \gamma\cdot\Big(6\nu+\frac{12L_{\phi}r_{\text{max}}\nu}{1-\gamma}\Big)\mathbb{E}\|y_0 - y^*(x_0)\|^2 + \gamma\cdot\frac{3M^2(\kappa^2+1)}{1-\gamma}\Big(1 + 2L_{\phi}r_{\text{max}}\Big).
        \end{aligned}
    \end{equation}

    In order to attain an $\epsilon$-stationary Pareto point, we select parameters as follows:
    \begin{equation*}
        \begin{aligned}
            & \alpha = \frac{2}{L+\mu_g}, \eta \le \frac{1}{L}, K = \Theta(\frac{\kappa^3}{\epsilon}),  D = \Theta(\kappa\log(\epsilon^{-1})),  Q = \Theta(\kappa\log(\kappa^2\epsilon^{-1})), \\
            & T = \Theta(\frac{\kappa^8}{\epsilon}),  D_f = \Theta(\frac{\kappa^5}{\epsilon}),  D_g = \Theta(\frac{\kappa^5}{\epsilon}),  B = \Theta(\frac{\kappa^5}{\epsilon}), \\
        \end{aligned}
    \end{equation*}
    then, the complexity result is:
    \begin{equation*}
        \begin{aligned}
            & \mathsf{Gc}(f, \epsilon) = KD_f\cdot S = \mathcal{O}(\kappa^8\epsilon^{-2}S), \\
            & \mathsf{Gc}(g, \epsilon) = KDT = \widetilde{\mathcal{O}}(\kappa^{12}\epsilon^{-2}), \\
            & \mathsf{JV}(g, \epsilon) = KD_g\cdot S = \mathcal{O}(\kappa^8\epsilon^{-2}S), \\
            & \mathsf{HV}(g, \epsilon) = \frac{KBQ}{\eta\mu_g}\cdot S = \widetilde{\mathcal{O}}(\kappa^9\epsilon^{-2}S). \\
        \end{aligned}
    \end{equation*}
    
\end{proof}

\subsection{Theoretical result for non-preference scenario}\label{sec:app-non}

In this section, we present the algorithm and theoretical results for the non-preference scenario, where the user preference $r$ is either not provided or not critical. This aligns precisely with the fourth type of MOO philosophy discussed in \Cref{sec:statement}, where our objective is to identify an arbitrary Pareto stationary point regardless of the preference. We begin by introducing the algorithm, followed by an analysis that demonstrates how this scenario serves as a special case of the main theories of \alg discussed in \Cref{sec:analysis}.

As shown in \Cref{alg:NP}, we only focus on the deterministic case for brevity. We have previously discussed that the choice between CG and NS primarily affects the gap between true gradients and hyper-gradients. Therefore, we only consider the algorithm equipped with the NS approximation. The key difference between \Cref{alg:NP} and \Cref{alg:deterministic} lies in the computation of $\lambda_k$, for any $k\in\{0,\dotsc,K-1\}$. Since the preference $r$ is not a concern in this scenario, the primal MGDA can be utilized to determine $\lambda_k$ and $d_k$. We state and prove the theoretical result as follows:

\begin{algorithm}[t]
	\caption{Non-Preference Deterministic Algorithm}\label{alg:NP}
	\begin{algorithmic}[1]
        \STATE \textbf{Input:} The numbers of iteration $K, D, N$, initialization values $x_0, y_0, v_0$, and step-sizes $\alpha, \beta$.
        \FOR{\(k=0,1,\dots, K-1\)}
            \STATE Set $y_k^0 = y_{k-1}^D$ if $k>0$ and $y_0$ otherwise.
            \FOR{\(t=1,2,\dots, D\)}
                \STATE Update $y_k^t = y_k^{t-1} - \alpha\nabla_yg(x_k,y_k^{t-1})$.
            \ENDFOR
            \FOR{$s\in[S]$}
                \STATE Compute $\hat{\nabla}\phi_s(x_k)$ according to \Cref{eq:Neumann}.
            \ENDFOR
            \STATE Compute $\lambda_k = \argmin_{\lambda} \| \lambda^{\top} \hat{\nabla}{\Phi}(x_k) \|^2$ s.t. $\sum_{s=1}^S\lambda_k^{(s)} = 1$, $\lambda_k^{(s)}\ge 0, \forall s\in[\mathcal{S}]$ .
            \STATE Update $x_{k+1} = x_k - \beta\sum_{s=1}^S\lambda_k^{(s)}\hat{\nabla}\phi_s(x_k)$.
        \ENDFOR
	\end{algorithmic}
\end{algorithm}

\begin{theorem}
    Suppose \Cref{ass:convex} and \Cref{ass:Lipchitz} hold. For any $\epsilon>0$, when selecting $\alpha \le \frac{1}{L}$, $\beta = \min\{\frac{1}{2(1+L_{\phi})r_{\text{max}}}, \frac{1}{3L_{\phi}}\}$, $D \ge \Theta(\kappa\log\frac{1}{\epsilon})$, and $K = \Theta(\kappa^3\epsilon^{-1})$, \Cref{alg:NP} with NS approximation satisfies:
    \begin{equation*}
        \begin{aligned}
            & \frac{1}{K}\sum_{k=0}^{K-1}\|\bar{d}_k\|^2 \le \frac{8L_{\phi}(\phi_s(x_0) - \phi_s(x_{K}))}{K} + (12+8L_{\phi}) \cdot \\
            & \left[\left(\frac{2L^2(L+\mu_g)^2(1-\alpha\mu_g)^D}{\mu_g^2} +\frac{4M^2(\tau\mu_g+L\rho)^2}{\mu_g^4}(1-\alpha\mu_g)^{D-1}\right)\chi + \frac{L^2M^2(1-\alpha\mu_g)^{2D}}{\mu_g^2} \right].
        \end{aligned}
    \end{equation*}
    where $\chi$ is a positive constant, and the detailed form of $D$ is:
    \begin{equation*}
        \begin{aligned}
            & D \ge \frac{\log[(24+16L_{\phi})(2\chi(\kappa^2(L+\mu_g)^2(1-\alpha\mu_g) + \frac{4M^2(\tau\mu_g+L\rho)^2)}{\mu_g^4}) + M^2\kappa^2(1-\alpha\mu_g)^{D+1})\epsilon^{-1}]}{\log(\frac{1}{1-\alpha\mu_g})} \\
            & = \Theta(\kappa\log\frac{1}{\epsilon}). \\
        \end{aligned}
    \end{equation*}
    To achieve an $\epsilon$-Pareto stationary point, i.e., to let $\frac{1}{K}\sum_{k=0}^{K-1}\|\bar{d}_k\|^2 \le \epsilon$, the oracle complexities are:
    \begin{equation*}
        \begin{aligned}
            & \mathsf{Gc}(f, \epsilon) = \mathcal{O}(\kappa^3\epsilon^{-1}S),
            \mathsf{Gc}(g, \epsilon) = \widetilde{\mathcal{O}}(\kappa^4\epsilon^{-1}), \\
            & \mathsf{JV}(g, \epsilon) = \widetilde{\mathcal{O}}(\kappa^4\epsilon^{-1}S),
            \mathsf{HV}(g, \epsilon) = \widetilde{\mathcal{O}}(\kappa^{4}\epsilon^{-1}S). \\
        \end{aligned}
    \end{equation*}
\end{theorem}

\begin{proof}
    The main process is similar to the proof of \Cref{thm:WC_det}. By \Cref{eq:descent_lemma,eq:descent_lemma-2}, we attain:
    \begin{equation}
        \beta\left(1 - \frac{1}{2}\beta - \frac{L_{\phi}}{2}\beta\right)\|d_k\|^2 \le \phi_s(x_k) - \phi_s(x_{k+1}) + \frac{1}{2}\|\nabla\phi_s(x_k) - \hat{\nabla}\phi_s(x_k)\|^2.
    \end{equation}
    Note that we leverage the property of MGDA instead of our \Cref{lemma:PGMNL} here, i.e., the descent direction $d_k$ derived by MGDA satisfies the following inequality:
    \begin{equation}
        \langle\hat{\nabla}\phi_s(x_k), d_k\rangle \ge \|d_k\|^2.
    \end{equation}

    Then, we telescope the last inequality for $k$ from $0$ to $K-1$, with \Cref{lemma:WC-NS}:
    \begin{equation}
        \begin{aligned}
            & \beta\left(1 - \frac{1}{2}\beta - \frac{L_{\phi}}{2}\beta\right)\frac{1}{K}\sum_{k=0}^{K-1}\|d_k\|^2 \\
            \le & \frac{\phi_s(x_0) - \phi_s(x_{K})}{K} + \frac{1}{2K}\sum_{k=0}^{K-1}\|\nabla\phi_s(x_k) - \hat{\nabla}\phi_s(x_k)\|^2 \\
            \le & \frac{\phi_s(x_0) - \phi_s(x_{K})}{K} + \frac{L^2M^2(1-\alpha\mu_g)^{2D}}{\mu_g^2} \\
            & +\left(\frac{2L^2(L+\mu_g)^2(1-\alpha\mu_g)^D}{\mu_g^2} +\frac{4M^2(\tau\mu_g+L\rho)^2}{\mu_g^4}(1-\alpha\mu_g)^{D-1}\right)\chi,
        \end{aligned}
    \end{equation}
    where $\chi$ is a constant. Using \Cref{eq:triangle-hypre-true-1}, we have:
    \begin{equation}
        \begin{aligned}
            & \frac{\beta\left(1 - \frac{1}{2}\beta - \frac{L_{\phi}}{2}\beta\right)}{K}\sum_{k=0}^{K-1}\|\bar{d}_k\|^2 \\
            \le & \frac{2\beta\left(1 - \frac{1}{2}\beta - \frac{L_{\phi}}{2}\beta\right)}{K}\sum_{k=0}^{K-1}\|d_k\|^2 + 6\beta\left(1 - \frac{1}{2}\beta - \frac{L_{\phi}}{2}\beta\right)\max_{s,k} \|\nabla\phi_s(x_k) - \hat{\nabla}\phi_s(x_k)\|^2 \\
            \le & \frac{2(\phi_s(x_0) - \phi_s(x_{K}))}{K} + \left(12\beta\left(1 - \frac{1}{2}\beta - \frac{L_{\phi}}{2}\beta\right) + 2\right)\cdot \\
            & \left[\left(\frac{2L^2(L+\mu_g)^2(1-\alpha\mu_g)^D}{\mu_g^2} +\frac{4M^2(\tau\mu_g+L\rho)^2}{\mu_g^4}(1-\alpha\mu_g)^{D-1}\right)\chi + \frac{L^2M^2(1-\alpha\mu_g)^{2D}}{\mu_g^2} \right].
        \end{aligned}
    \end{equation}
    Selecting $\beta = \min\{\frac{3}{2(1+L_{\phi})}, \frac{1}{L_{\phi}}\}$, we attain:
    \begin{equation}
        \begin{aligned}
            & \frac{1}{K}\sum_{k=0}^{K-1}\|\bar{d}_k\|^2 \le \frac{8L_{\phi}(\phi_s(x_0) - \phi_s(x_{K}))}{K} + (12+8L_{\phi}) \cdot \\
            & \left[\left(\frac{2L^2(L+\mu_g)^2(1-\alpha\mu_g)^D}{\mu_g^2} +\frac{4M^2(\tau\mu_g+L\rho)^2}{\mu_g^4}(1-\alpha\mu_g)^{D-1}\right)\chi + \frac{L^2M^2(1-\alpha\mu_g)^{2D}}{\mu_g^2} \right].
        \end{aligned}
    \end{equation}
    
    In order to achieve an $\epsilon$-accurate stationary point, we select parameters as follows:
    \begin{equation*}
        \begin{aligned}
            & \alpha \le \frac{1}{L}, \\
            & D \ge \frac{\log[(24+16L_{\phi})(2\chi(\kappa^2(L+\mu_g)^2(1-\alpha\mu_g) + \frac{4M^2(\tau\mu_g+L\rho)^2)}{\mu_g^4}) + M^2\kappa^2(1-\alpha\mu_g)^{D+1})\epsilon^{-1}]}{\log(\frac{1}{1-\alpha\mu_g})} \\
            & = \Theta(\kappa\log\frac{1}{\epsilon}), \\
            & K = \Theta(\kappa^3\epsilon^{-1}),
        \end{aligned}
    \end{equation*}
    then, we have the following complexity results:
    \begin{equation*}
        \begin{aligned}
            & \mathsf{Gc}(f, \epsilon) = 2K\cdot S = \mathcal{O}(\kappa^3\epsilon^{-1}S), \\
            & \mathsf{Gc}(g, \epsilon) = KD = \widetilde{\mathcal{O}}(\kappa^4\epsilon^{-1}), \\
            & \mathsf{JV}(g, \epsilon) = KD\cdot S = \widetilde{\mathcal{O}}(\kappa^4\epsilon^{-1}S), \\
            & \mathsf{HV}(g, \epsilon) = KD\cdot S = \widetilde{\mathcal{O}}(\kappa^4\epsilon^{-1}S). \\
        \end{aligned}
    \end{equation*}
\end{proof}

Comparing this theorem with \Cref{thm:WC_det}, we observe that the convergence rates for both scenarios are of the same order; however, the non-preference case benefits from a tighter coefficient. On the one hand, this distinction arises because the difference between MGDA and \Cref{lemma:PGMNL}, with the former providing a tighter bound on $d_k$. On the other hand, the non-preference scenario can be viewed as a special case of our main theory, highlighting that the design of \Cref{eq:WC} is sufficiently delicate to accommodate user preferences.

\section{Additional numerical experiments}\label{sec:app-exp}

\subsection{Details of numerical experiments for deterministic case}\label{sec:app-exp-det}

All numerical experiments for deterministic case were conducted on a cluster of 4 NVIDIA H100 GPUs (96GB each) using PyTorch's DistributedDataParallel.

\textbf{Formulations.} We implement \Cref{alg:deterministic} on the FC-100 dataset \cite{oreshkin2018tadam} for a meta-learning task \cite{arnold2020learn2learn,hospedales2021meta} in the deterministic case. Specifically, there are $m$ different tasks $\mathcal{T}_i, i \in [m]$, each sampled from a distribution $\mathcal{P}$. For each task $\mathcal{T}_i$, we maintain a model $\mathcal{M}(\phi, w_i)$ along with its corresponding loss function $\mathcal{L}(\phi, w_i; \xi)$, where $\phi$ represents the shared parameter among all tasks, $w_i$ is the parameter specific to task $i$, and $\xi$ denotes the randomness from the dataset. In this scenario, our MOBL problem can be formulated as follows:
\begin{equation}
    \begin{aligned}
        & \min_{\phi}\mathcal{L}_\mathcal{D}(\phi, \widetilde{w}^*) := \Big[\frac{1}{|\mathcal{D}_1|}\sum_{\xi\in\mathcal{D}_1}\mathcal{L}(\phi,w_1^*;\xi),\dotsc, \frac{1}{|\mathcal{D}_m|}\sum_{\xi\in\mathcal{D}_m}\mathcal{L}(\phi,w_m^*;\xi)\Big] \\
        & \text{s.t. } \widetilde{w}^* = \argmin_{\widetilde{w}}\mathcal{L}_{\mathcal{S}}(\phi, \widetilde{w}) := \frac{1}{m}\sum_{i=1}^m\frac{1}{|\mathcal{S}_i|}\sum_{\xi\in \mathcal{S}_i}\mathcal{L}(\phi, w_i;\xi) + \mathcal{R}(w_i),
    \end{aligned}
\end{equation}
where (1) $\widetilde{w} = [w_1, \dotsc, w_m]^\top$, (2) $\mathcal{R}(w_i)$ is a regularizer to ensure the strong convexity of the lower-level objective function, and (3) $\mathcal{S} = \{\mathcal{S}_i\}_{i\in[m]}$ and $\mathcal{D} = \{\mathcal{D}_i\}_{i\in[m]}$ are training datasets and test datasets, respectively. In this formulation, on the one hand, the lower-level minimizes the averaged loss function to determine the optimal task-specific parameters $\widetilde{w}$ given the current shared parameter $\phi$. On the other hand, the upper-level treats the $m$ tasks as $m$ potentially conflicting objectives and seeks the Pareto optimal shared parameter $\phi$ based on the output $\widetilde{w}^*$ from the lower level.

\textbf{Datasets.} FC-100 (Fewshot-CIFAR-100) dataset is a variant of CIFAR-100 dataset \cite{krizhevsky2009learning}. The CIFAR-100 dataset contains $100$ image classes, each with $500$ training images and $100$ test images, all of size $32 \times 32 \times 3$. The FC-100 dataset divides these $100$ classes into $60$ for training, $20$ for validation, and $20$ for testing. These small size sub-datasets make it more challenging for learning.

\textbf{Settings.} We consider a $5$-way $5$-shot meta-learning task, meaning there are $m=5$ tasks, and each support set contains $|\mathcal{S}_i| = 5, \forall i\in[m]$ samples. Inspired by \cite{ji2021bilevel}, $\mathcal{M}(\phi, w_i)$ is chosen to be a $4$-layer Convolutional Neural Network (CNN) equipped with batch normalization, ReLU activation, $2 \times 2$ max pooling, and $64$ filters in each convolutional layer for each $i \in [m]$. Besides, in the previous formulation, we choose Cross-Entropy as our loss function $\mathcal{L}$, and $l_2$ regularizer as our regularizer $\mathcal{R}$.

To implement \Cref{alg:deterministic} with Neumann-series approximation, we set $S=5$, $K=500$, $D=32$, $\alpha=0.1$, $\beta=0.1$, $u=10$, and $r = [0.8, 0.05, 0.05, 0.05, 0.05]^\top$ as our default experiment settings. Each experiment is repeated $4$ times, as we observe that fluctuations between different random seeds are minimal. To perform the ablation study, we vary several parameters as follows:
\begin{itemize}
    \item We set different preference vectors $r$. Since $r \in \mathbb{R}^S_{++}$ and $\textbf{1}^\top r = 1$, we select one of the five elements of $r$ as $0.8$, while the remaining elements are set to $0.05$ to represent user preferences. Therefore, for $5$ different vectors $r$, each of them focuses on one objective.
    \item We set different trade-off controller constants $u$ in \Cref{eq:WC} from the set $\{1, 10, 20\}$. Note that a larger $u$ places greater emphasis on the preference matching issue.
    \item We set different inner-loop iteration counts $D$ from the set $\{8, 16, 32\}$. While the outer-loop iteration count $K$ is fixed, this setup allows us to observe the convergence rate under different values of $D$.
\end{itemize}

\textbf{Additional Results.} The following simulation results further verify the effectiveness of \Cref{alg:deterministic} and the correctness of \Cref{thm:WC_det}. In addition to the aforementioned ablation studies, we also present the numerical results for \Cref{sec:app-non}, where the convergence process is not guided by any user preference $r$.


\begin{figure}[htb]
    \centering
    \begin{subfigure}[t]{0.38\textwidth}
        \centering
        \includegraphics[width=\linewidth]{Figure/Experiment/ML_pref.png}
        \caption{Exploration of Pareto front under the guidance of different preference vectors $r$.}
        \label{fig:det-pref-app}
    \end{subfigure}
    \hspace{0.1cm}
    \begin{subfigure}[t]{0.24\textwidth}
        \centering
        \includegraphics[width=\linewidth]{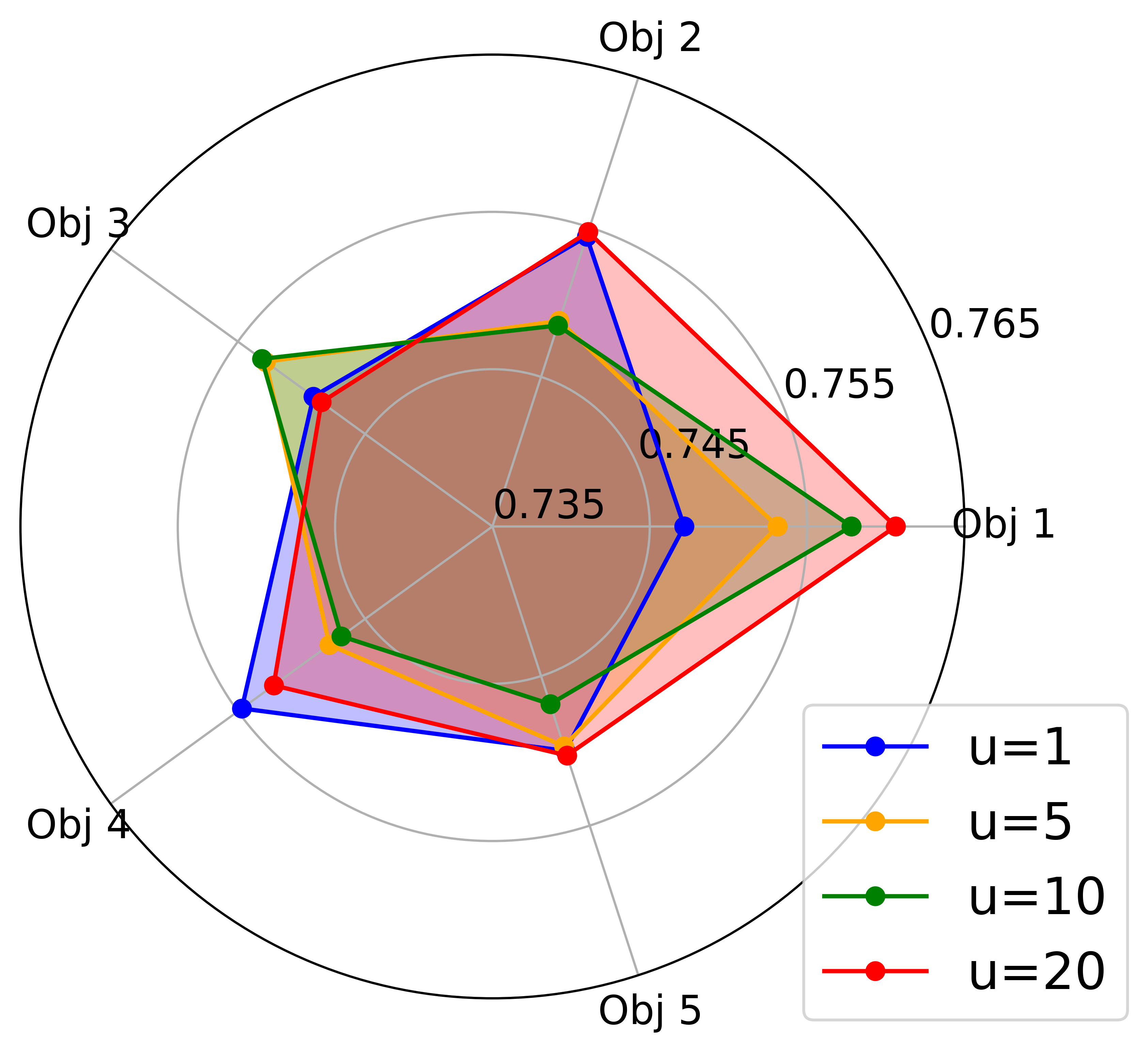}
        \caption{Performance under different constant $u$.}
        \label{fig:det-u}
    \end{subfigure}
    \hspace{0.1cm}
    \begin{subfigure}[t]{0.33\textwidth}
        \centering
        \includegraphics[width=\linewidth]{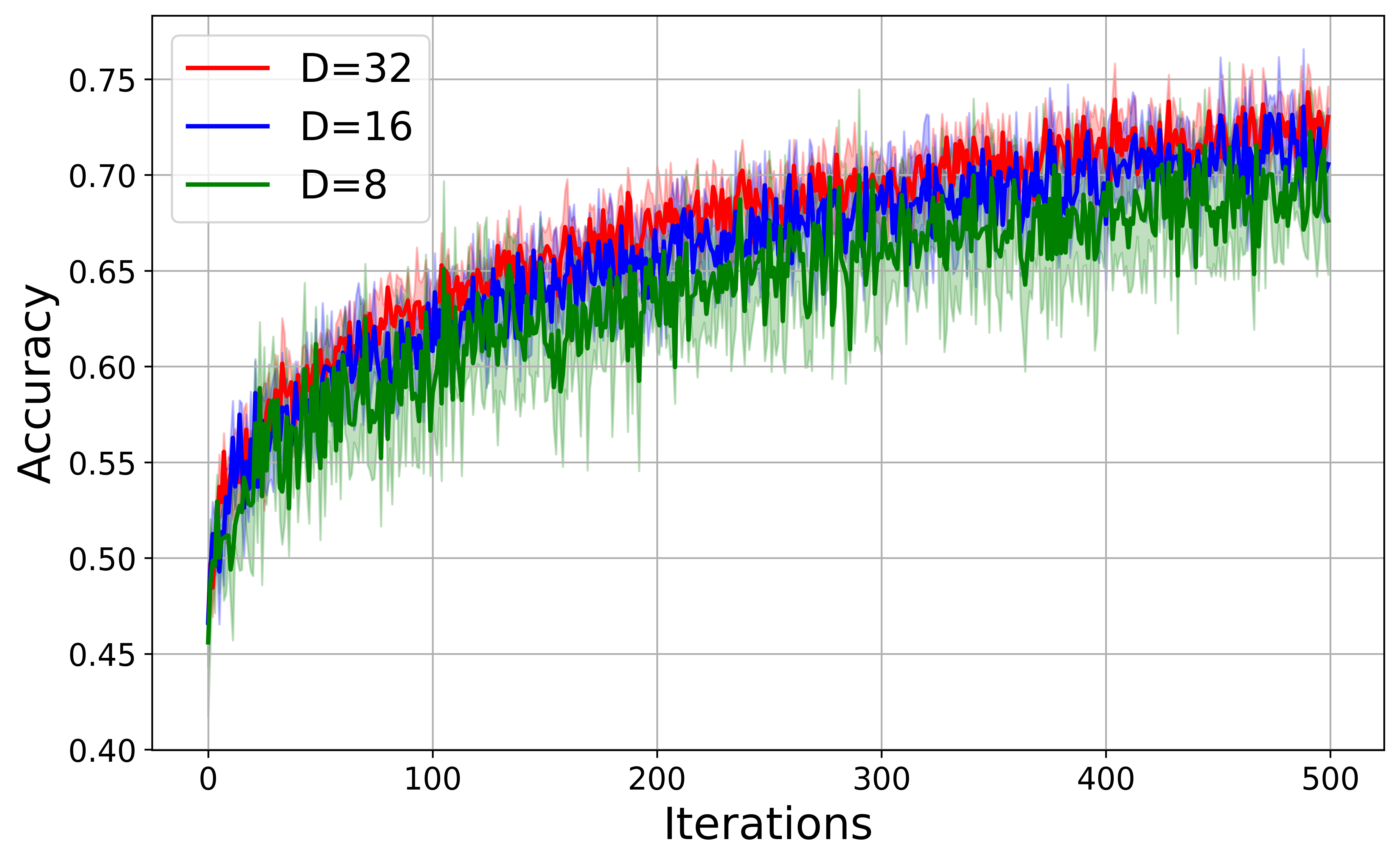}
        \caption{Convergence rate when prefer objective $1$ and $u=10$.}
        \label{fig:ml-converge}
    \end{subfigure}
    \caption{Performances of \Cref{alg:deterministic}.}
    \label{fig:ml-performance}
\end{figure}

Fig.~\ref{fig:det-pref-app} shows how \Cref{alg:deterministic} converges to preference-guided Pareto stationary points. 
We can see that the accuracy of the more preferred objective is consistently higher than those less preferred  objectives. 
Moreover, for those extremely preferred objectives, this trend becomes more pronounced. 
The outer gray area demonstrates that our \alg algorithm effectively explore a large Pareto stationarity footprint.
With the preference vector $r=[0.96, 0.01, 0.01, 0.01, 0.01]^\top$, Fig.~\ref{fig:det-u} shows the accuracy results with varying trade-off parameters $u$ chosen from ${1, 5, 10, 20}$. 
We can see that a larger $u$-value leads to improved performance on the extremely preferred objective $1$, which confirms our theoretical insights of \Cref{eq:WC}.
Fig.~\ref{fig:ml-converge} shows the convergence of \Cref{alg:deterministic}'s accuracy.
We vary the inner-loop iteration count $D$ to examine its impact on the performance of \Cref{alg:deterministic}. 
According to \Cref{thm:WC_det}, 
a larger $D$ results in faster convergence, which is verified in Fig.~\ref{fig:ml-converge}.


\begin{figure}[htb]
    \centering
    \begin{subfigure}[t]{0.32\textwidth}
        \centering
        \includegraphics[width=\linewidth]{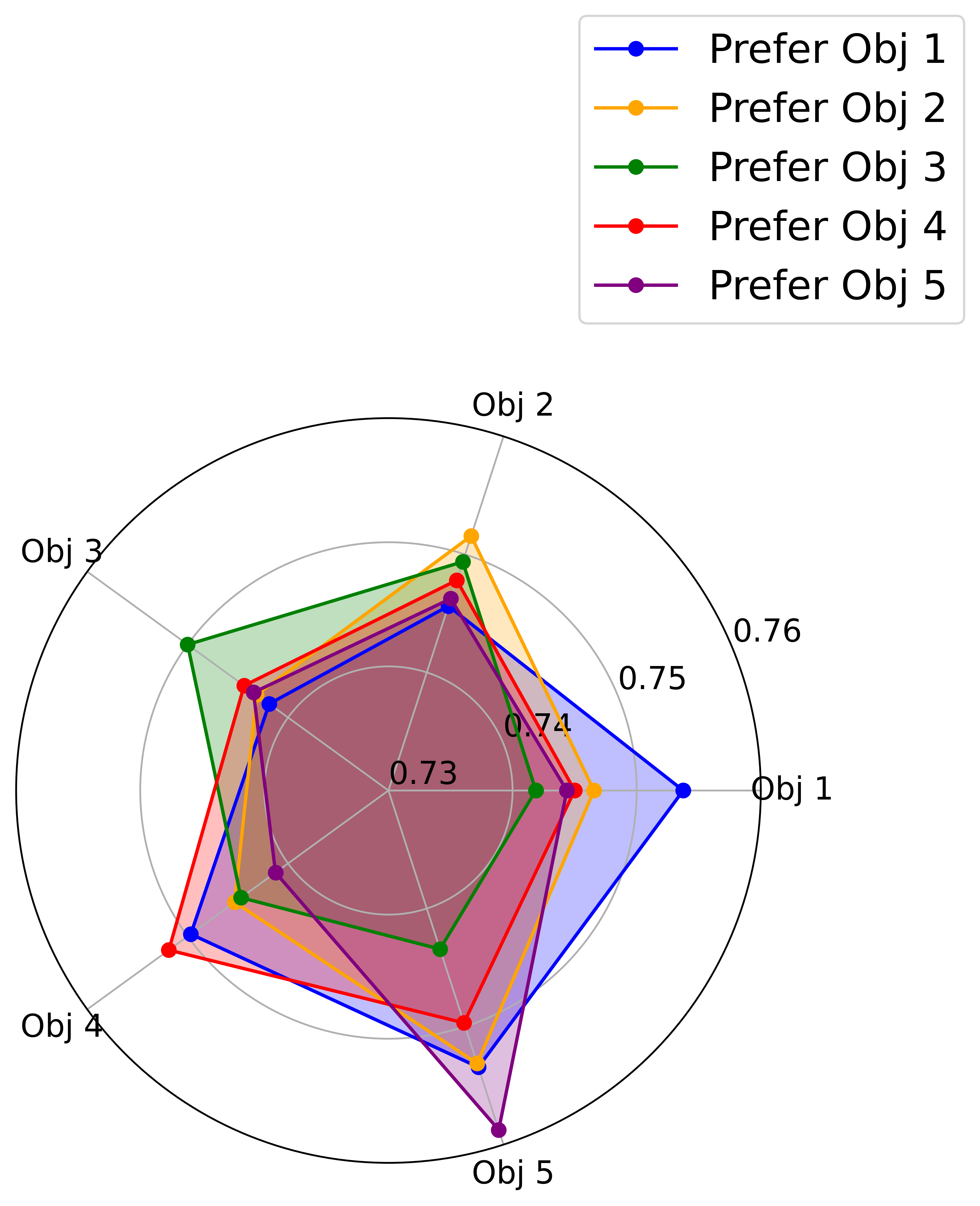}
        \caption{Training Accuracy}
        \label{fig:ml-r-train}
    \end{subfigure}
    \hspace{0.5cm}
    \begin{subfigure}[t]{0.32\textwidth}
        \centering
        \includegraphics[width=\linewidth]{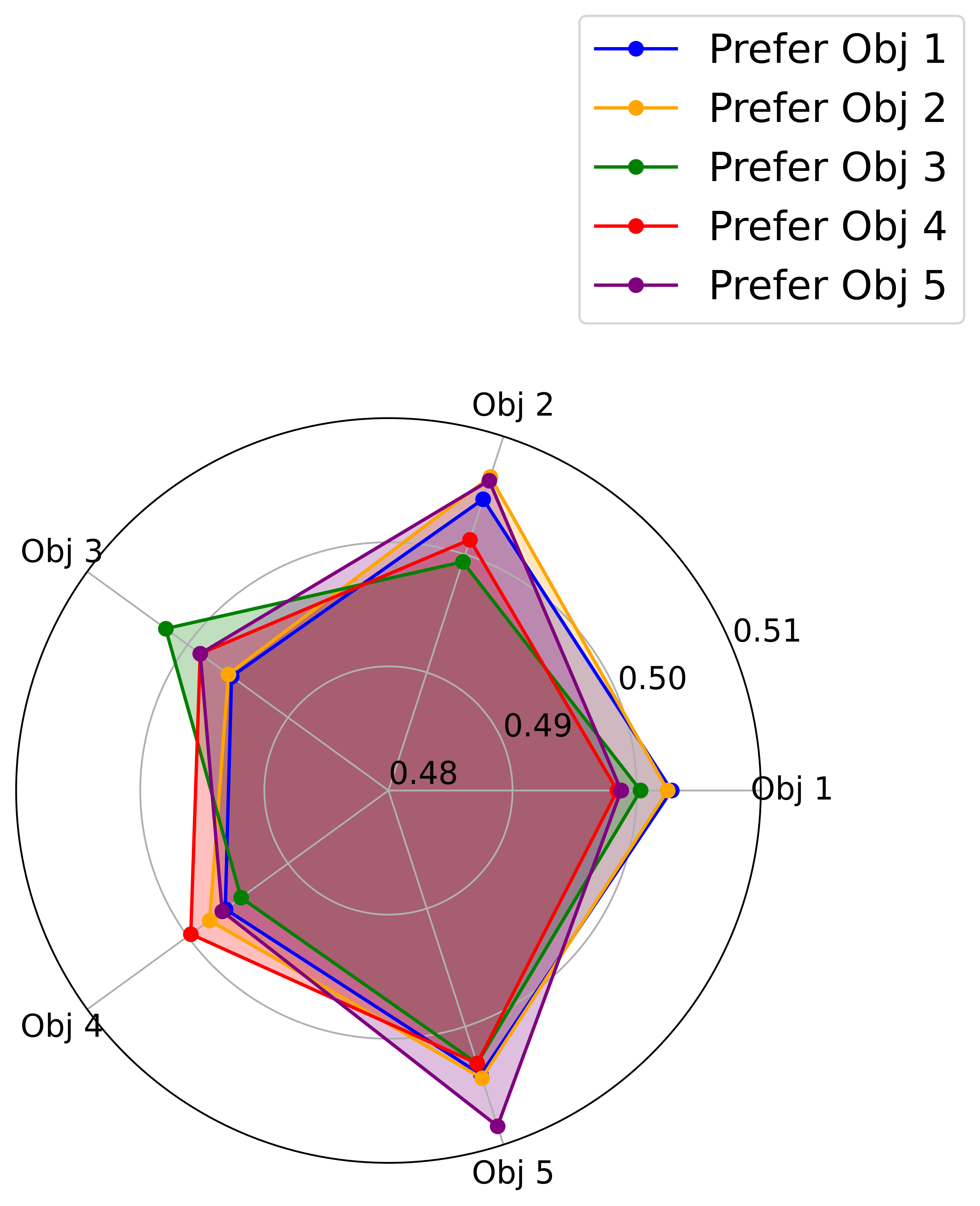}
        \caption{Test Accuracy}
        \label{fig:ml-r-test}
    \end{subfigure}
    \caption{The accuracy with different preference vector $r$.}
    \label{fig:ml-r}
\end{figure}

\Cref{fig:ml-r} clearly illustrates how \Cref{alg:deterministic} converges under the guidance of different preference vectors $r$. Since we assign one of the five objectives as the preferred objective with a value of $0.8$ in each $r$, while the others are set to $0.05$, both \Cref{fig:ml-r-train} and \Cref{fig:ml-r-test} essentially align with the expectations. In other words, the accuracy of objective $i$ is almost always the highest (or at least nearly the highest) when the corresponding $r_i=0.8$ is the highest.


\begin{figure}[htb]
    \centering
    \begin{subfigure}[t]{0.32\textwidth}
        \centering
        \includegraphics[width=\linewidth]{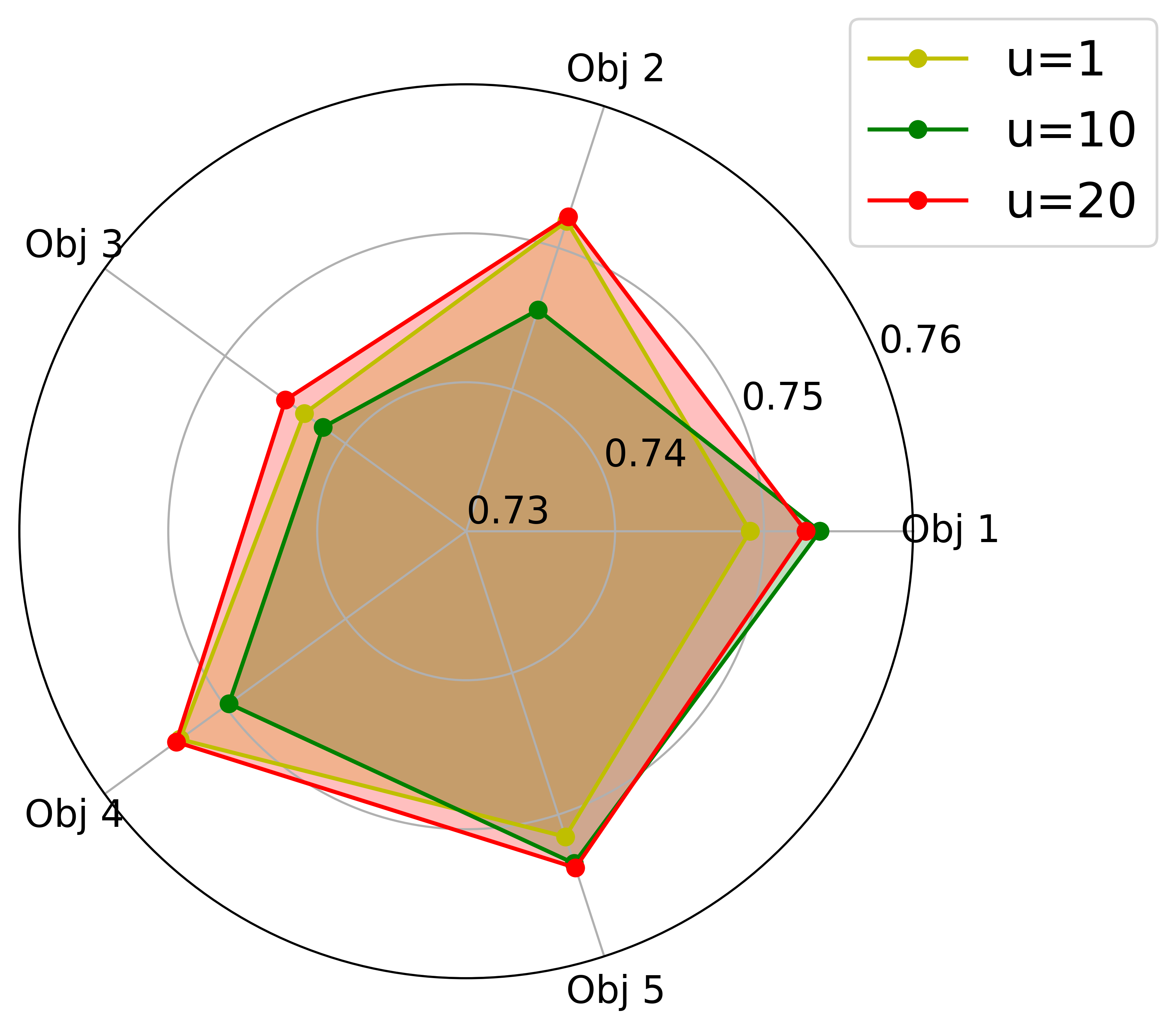}
        \caption{Training Accuracy}
        \label{fig:ml-u-train}
    \end{subfigure}
    \hspace{0.5cm}
    \begin{subfigure}[t]{0.32\textwidth}
        \centering
        \includegraphics[width=\linewidth]{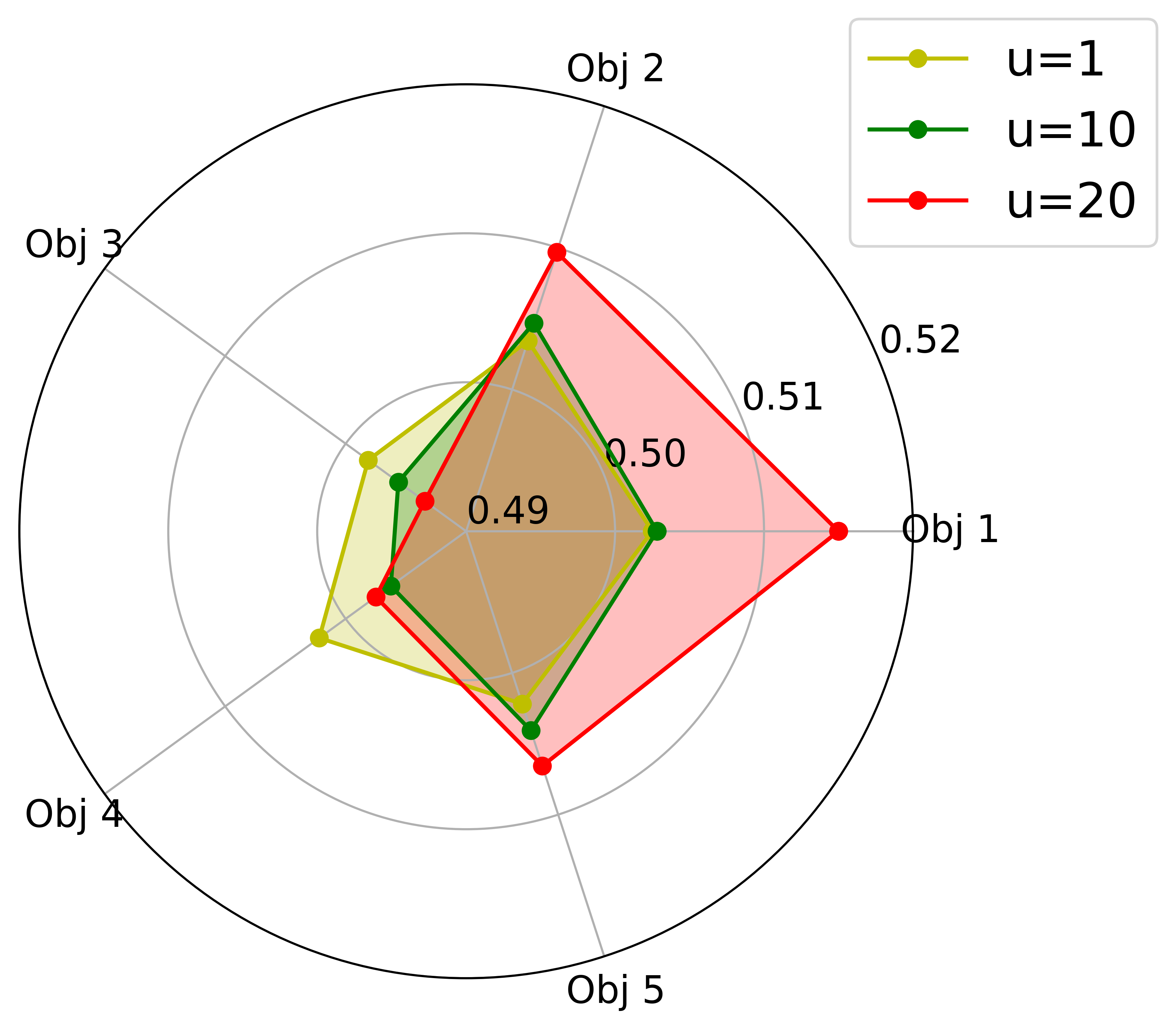}
        \caption{Test Accuracy}
        \label{fig:ml-u-test}
    \end{subfigure}
    \caption{The accuracy with different constant $u$.}
    \label{fig:ml-u}
\end{figure}

\Cref{fig:ml-u} illustrates how the trade-off controller constant $u$ affects the performance of \Cref{alg:deterministic}. As discussed in \Cref{subsec:alg-det}, a larger $u$ places more emphasis on the preference, which is $r = [0.8, 0.05, 0.05, 0.05, 0.05]^\top$ in this case. In other words, the red curve in \Cref{fig:ml-u} is expected to achieve the highest accuracy for objective $1$, as its corresponding preference value of $0.8$ is the largest. \Cref{fig:ml-u-test} aligns perfectly with this expectation, as larger values of $u$ result in higher accuracy for objective $1$. On the other hand, \Cref{fig:ml-u-train} appears to deviate from the aforementioned expectations. However, we argue that \Cref{fig:ml-u-train} is still consistent with our theory for the following reasons: (1) Since a larger $u$ in \Cref{eq:WC} places greater emphasis on the preference, the performances for objective $1$ when $u$ is set to $10$ and $20$ are indeed better than those when $u = 1$. (2) Although a larger $u$ makes the convergence process more guided by preference $r$, even when $u$ is small, the convergence sequence, even if itself is not overly influenced by the preferences, can still potentially fall on any point along the Pareto front, including those guided by the preferences. Considering that the performance of objective $1$ shows little difference between $u = 10$ and $u = 20$, we can conclude that both cases are aligned with the preference vector $r$ to some extent. Therefore, this observation further supports the validity of our theory.


\begin{figure}[htb]
    \centering
    \begin{subfigure}[t]{0.32\textwidth}
        \centering
        \includegraphics[width=\linewidth]{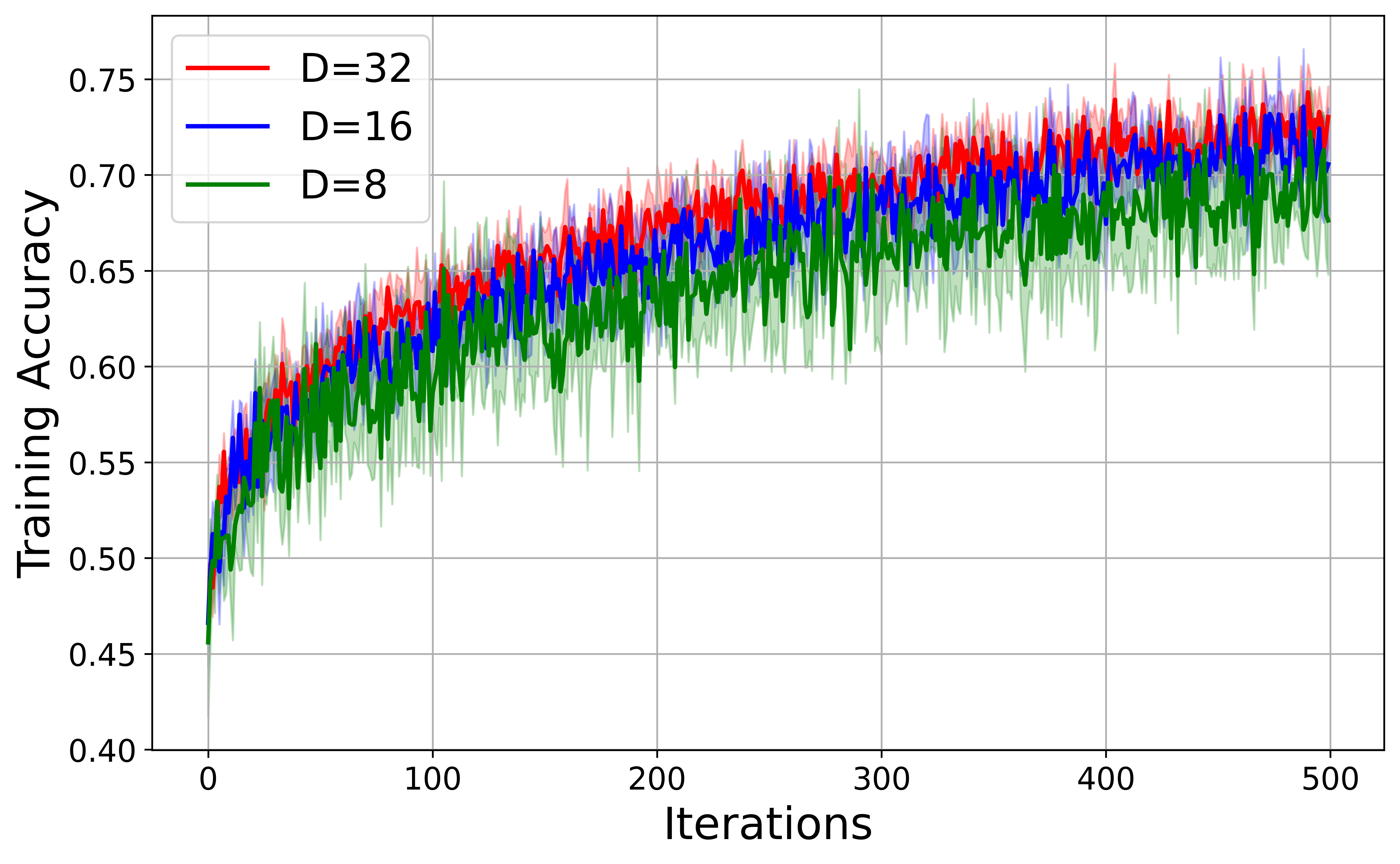}
        \caption{Training Accuracy}
        \label{fig:ml-D-train}
    \end{subfigure}
    \hspace{0.5cm}
    \begin{subfigure}[t]{0.32\textwidth}
        \centering
        \includegraphics[width=\linewidth]{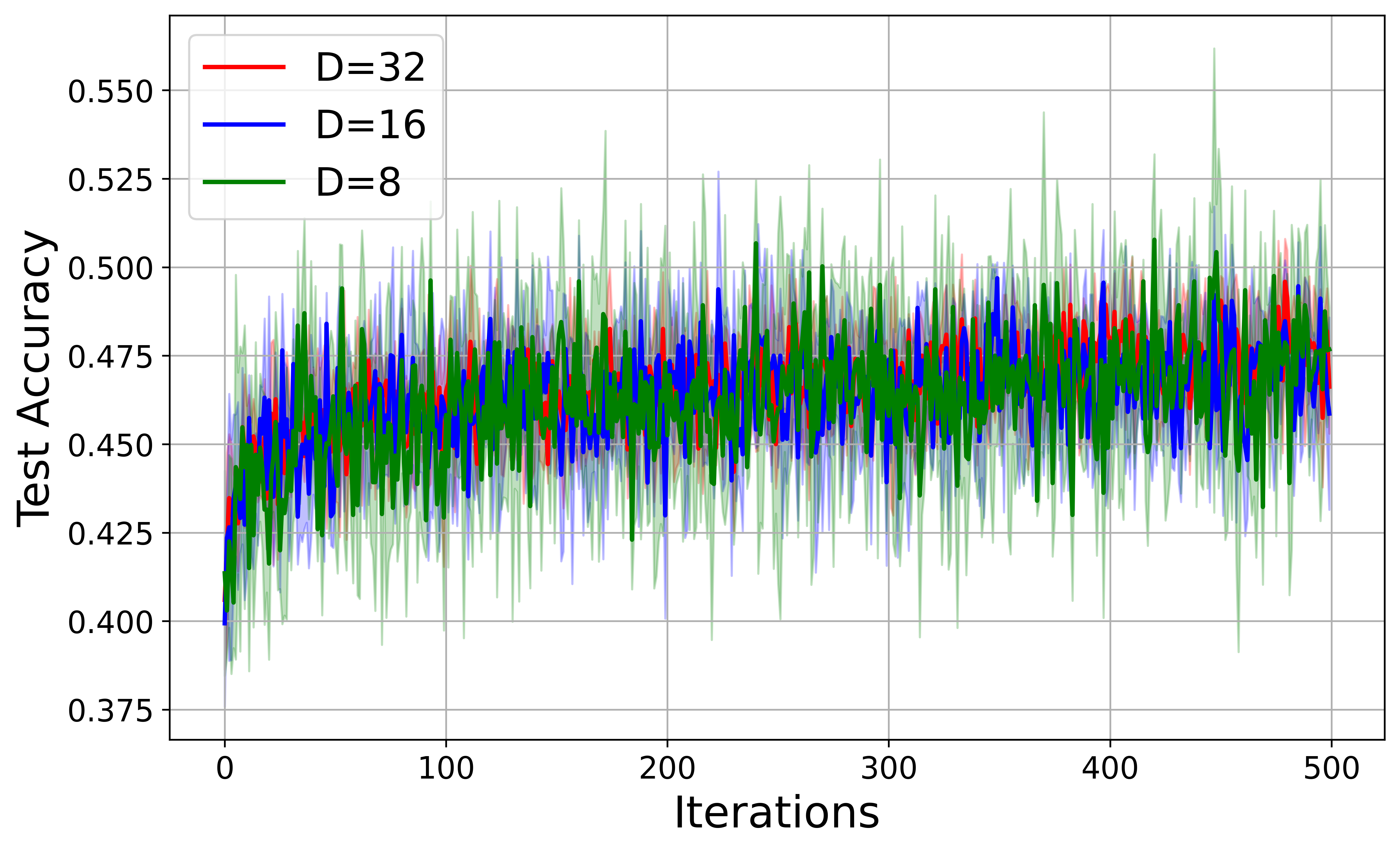}
        \caption{Test Accuracy}
        \label{fig:ml-D-test}
    \end{subfigure}
    \caption{The average accuracy with inner iteration times $D$.}
    \label{fig:ml-D}
\end{figure}

\Cref{fig:ml-D} demonstrates that, given the same number of outer-loop iterations, a larger $D$ results in higher average accuracy. This observation verifies the result in \Cref{thm:WC_det}, where $D$ acts as the exponent of $1 - \alpha \mu_g$, which converges to $0$ as $D$ approaches infinity. Additionally, we include standard error bars, which reveal that smaller values of $D$ lead to greater fluctuations. This aligns with intuition: since insufficient inner-loop iterations fail to accurately determine $y^*$, the process is more susceptible to randomness, resulting in more pronounced oscillations. Finally, it is also worth mentioning that the low accuracies in \Cref{fig:ml-D-test} is due to the few sample nature of meta-learning tasks.


\begin{figure}[htb]
    \centering
    \begin{subfigure}[t]{0.35\textwidth}
        \centering
        \includegraphics[width=\linewidth]{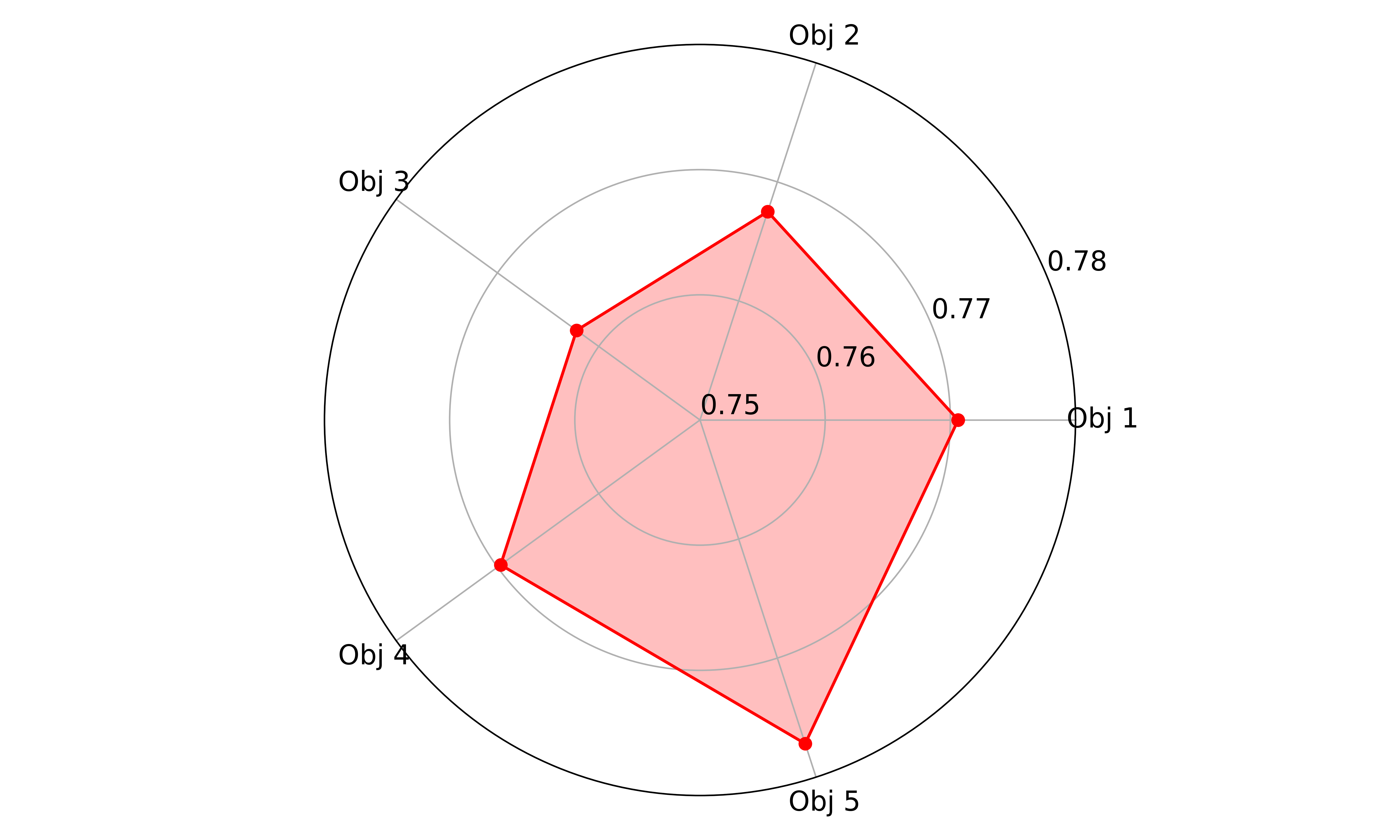}
        \caption{Training Accuracy for $5$ objectives}
        \label{fig:ml-non-train-5}
    \end{subfigure}
    \hspace{0.5cm}
    \begin{subfigure}[t]{0.35\textwidth}
        \centering
        \includegraphics[width=\linewidth]{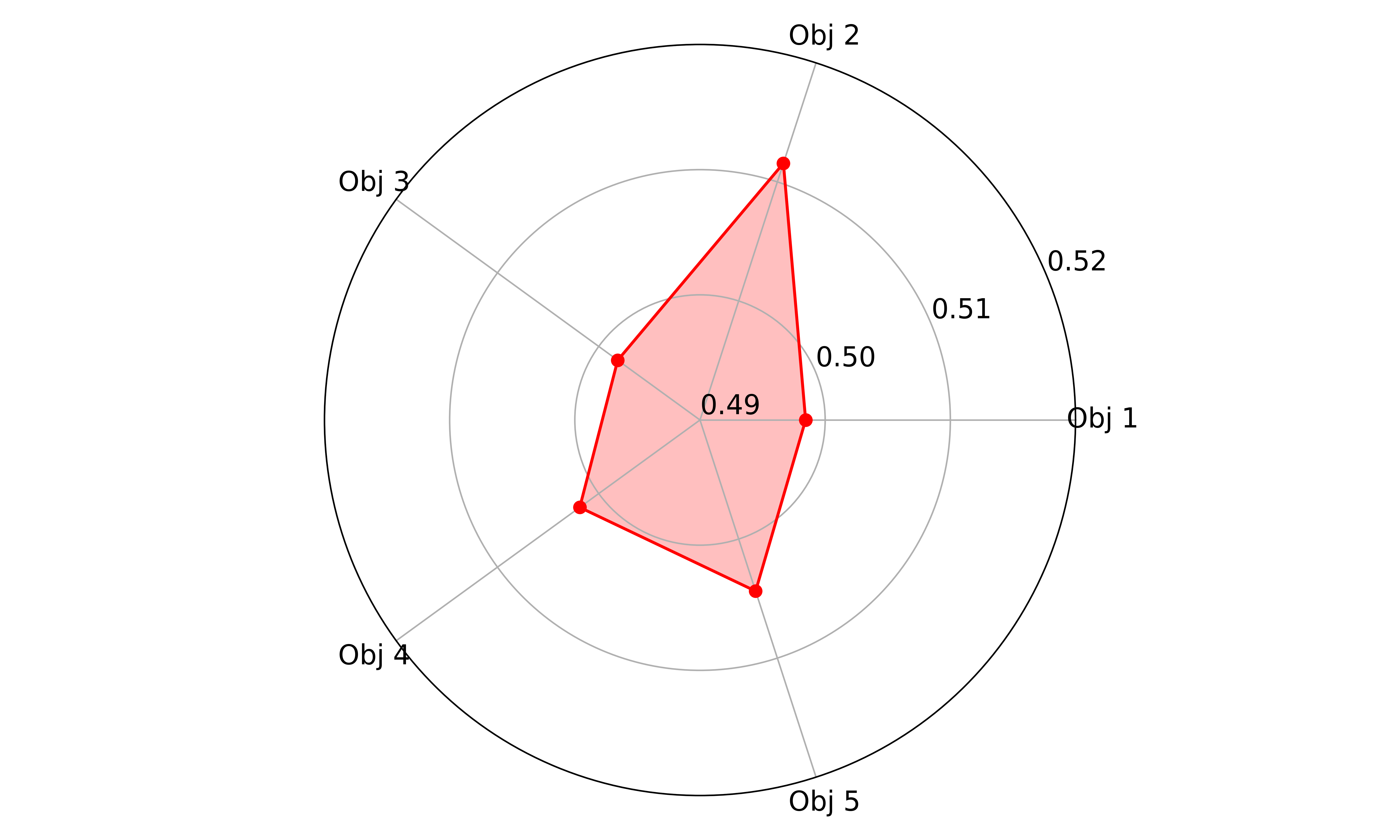}
        \caption{Test Accuracy for $5$ objectives}
        \label{fig:ml-non-test-5}
    \end{subfigure}

    \vspace{0.3cm}

    \begin{subfigure}[t]{0.35\textwidth}
        \centering
        \includegraphics[width=\linewidth]{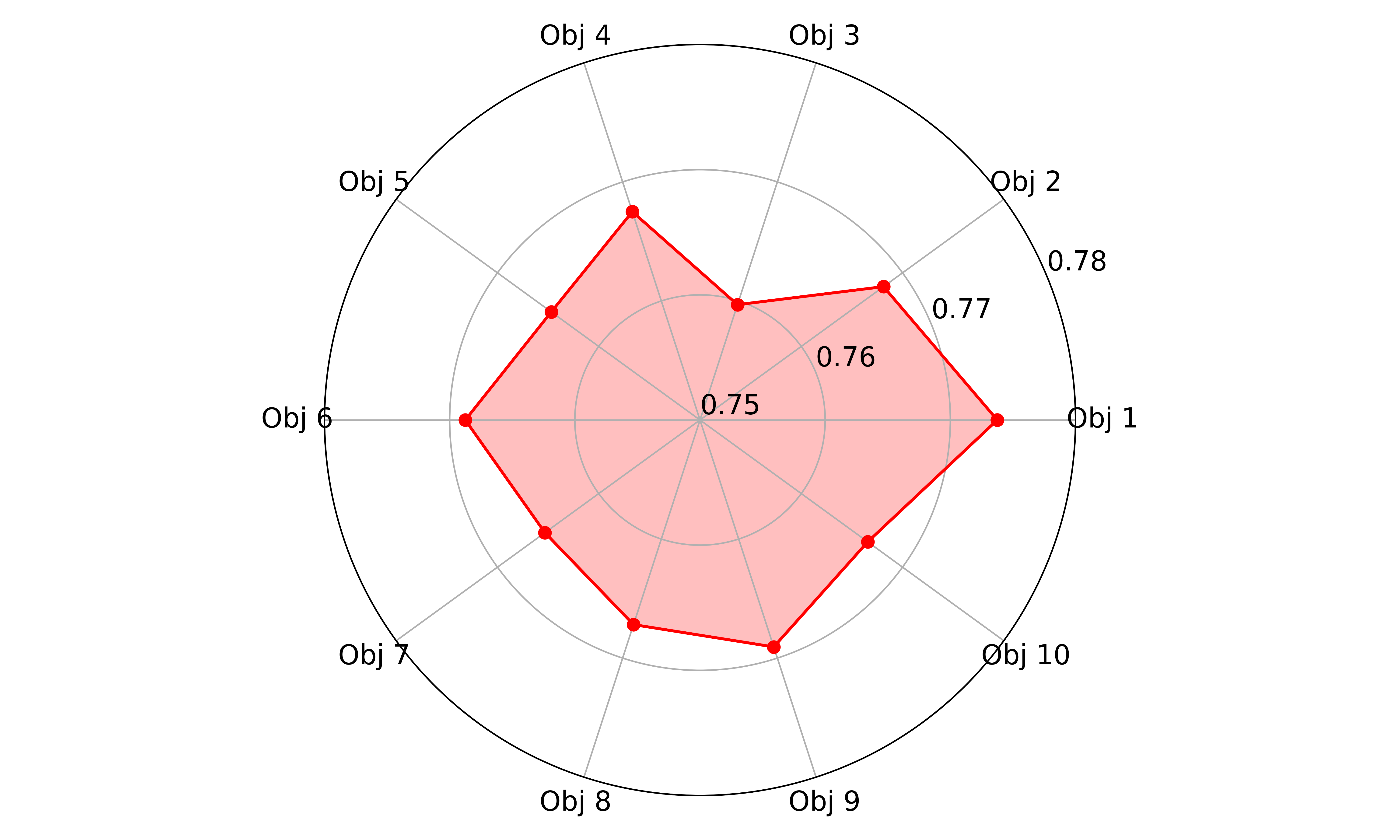}
        \caption{Training Accuracy for $10$ objectives}
        \label{fig:ml-non-train-10}
    \end{subfigure}
    \hspace{0.5cm}
    \begin{subfigure}[t]{0.35\textwidth}
        \centering
        \includegraphics[width=\linewidth]{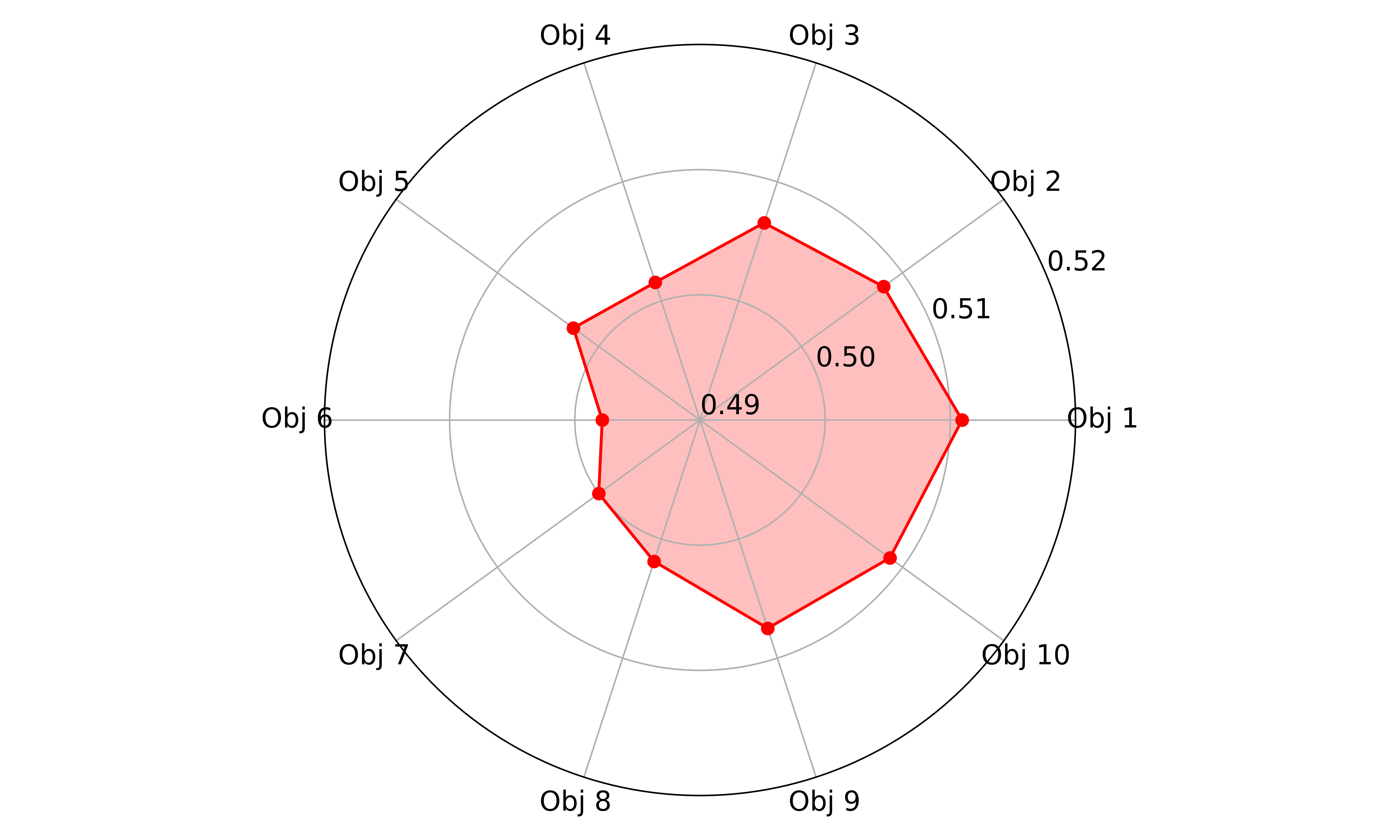}
        \caption{Test Accuracy for $10$ objectives}
        \label{fig:ml-non-test-10}
    \end{subfigure}

    \caption{The accuracy without guidance of preference vector $r$ for $5$ and $10$ objectives.}
    \label{fig:ml-non}
\end{figure}

\Cref{fig:ml-non} validates the results in \Cref{sec:app-non}, showing that the sequence arbitrarily converges to some Pareto stationary point without the guidance of any preference vector $r$. In contrast to \Cref{fig:ml-r}, where the accuracies in both training and testing consistently exhibit stronger performance for the objective corresponding to $r_{\text{max}}$, \Cref{fig:ml-non} obviously suffers from greater randomness. In other words, the better performance for some objective in training process does not directly imply the similar results in test process. This further highlights the significance of the delicately designed approach in \Cref{sec:alg} for effectively addressing user preferences.

As for the baselines \cite{ye2021multi,ye2024first}, we also enumerate the setup as follows. 1) For MOML \cite{ye2021multi}, the parameters are identical to those in our default setting, as \Cref{alg:deterministic} encompasses all of its required parameters. 2) For FORUM \cite{ye2024first}, we set $\rho = 2$ and $\beta_k = \frac{1}{2k}$, with all other parameters matching those in our setup. Notably, we also tested their recommended values ($\rho \in\{ 0.5, 0.7, 0.9\}$), however, $\rho = 2$ yields the best performances.

\subsection{Details of numerical experiments for stochastic case}\label{sec:app-exp-stoc}

All numerical experiments for stochastic case were conducted on a MacBook Pro equipped with an Apple M2 Pro chip, featuring a 8-core CPU, and 16GB of memory.

\textbf{Formulations.} We implement \Cref{alg:stoc} on the MNIST dataset \cite{lecun1998gradient} for a data hyper-cleaning task \cite{franceschi2018bilevel,shaban2019truncated} to verify the theoretical results of \Cref{alg:stoc} in the stochastic case. The objective of the data hyper-cleaning task is to train a classifier using corrupted data under a specified corruption rate $p$, where each label has a probability of $p$ of being incorrect. We consider $m$ tasks under different corruption rates $p$. Each task $i \in [m]$ aims to train a model $\mathcal{M}_{w_i}$, parameterized by $w_i$. Additionally, all $m$ tasks share the same regularization parameter $\lambda$ and aim to minimize the validation loss. Since there are $m$ potentially conflicting tasks, along with the complex interactions between the training and validation processes, this problem can be formulated as an MOBL problem as follows:

\begin{equation}\label{eq:stoc-loss}
    \begin{aligned}
        & \min_{\lambda}\mathcal{L}_\mathcal{D}(\lambda, \widetilde{w}^*) := \Big[\frac{1}{|\mathcal{D}_1|}\sum_{(x_j,y_j)\in\mathcal{D}_1}\mathcal{L}(\mathcal{M}_{w_1}(x_j), y_j),\dotsc, \frac{1}{|\mathcal{D}_m|}\sum_{(x_j,y_j)\in\mathcal{D}_m}\mathcal{L}(\mathcal{M}_{w_m}(x_j), y_j)\Big] \\
        & \text{s.t. } \widetilde{w}^* = \argmin_{\widetilde{w}}\mathcal{L}_{\mathcal{S}}(\lambda, \widetilde{w}) := \frac{1}{m}\sum_{i=1}^m\frac{1}{|\mathcal{S}_i|}\sum_{(x_j,y_j)\in \mathcal{S}_i}\sigma(\lambda_j)\mathcal{L}(\mathcal{M}_{w_i}(x_j), y_j) + \mathcal{R}(w_i),
    \end{aligned}
\end{equation}
where (1) $\widetilde{w} = [w_1, \dotsc, w_m]^\top$, (2) $\mathcal{R}(w_i)$ is a regularizer to ensure the strong convexity of the lower-level objective function, (3) $\mathcal{S} = \{\mathcal{S}_i\}_{i\in[m]}$ and $\mathcal{D} = \{\mathcal{D}_i\}_{i\in[m]}$ are task-specific training datasets and validation datasets, respectively, and (4) $\sigma(\lambda)$ the sigmoid function applied on $\lambda$. The lower-level of this formulation finds the optimal task-specific parameters $\widetilde{w}$ given current regularizer $\lambda$, while the upper-level updates this $\lambda$ based on the output $\widetilde{w}^*$.

\textbf{Datasets.} We conduct our experiments on MNIST \cite{lecun1998gradient}, one of the most widely used datasets in the community. MNIST consists of $60,000$ training images and $10,000$ test images of handwritten digits, each with a size of $28 \times 28 \times 1$. To align with established settings, we use a subset of MNIST inspired by \cite{shaban2019truncated}, which includes $20,000$ training images, $5,000$ validation images, and $10,000$ test images.

\textbf{Settings.} We mainly consider $m=5$, and the corresponding corruption rates are $[0, 0.15, 0.3, 0.45, 0.6]$. We choose $\mathcal{M}_{w_i}$ to be a vector, $\mathcal{L}$ to be Cross-Entropy, $R(w_i) = 0.1\|w_i\|^2, \forall i\in[m]$. To implement \Cref{alg:stoc}, we set $S=5$, $(K, D)=(150,200)$, $\alpha=0.1$, $\beta=0.1$, $\eta=0.5$, $Q=3$, $u=10$, and $r = [0.025, 0.025, 0.025, 0.025, 0.9]^\top$ as our default experiment settings. Each experiment is repeated $5$ times. It is worth mentioning that we consider the metric $1/L$, where $L$ is the upper-level loss value defined in \Cref{eq:stoc-loss}. Therefore, consistent with the accuracy metric in \Cref{sec:experiments,sec:app-exp-det}, a larger covered area indicates a greater extent of exploration by the algorithm. In this section, we provide the following numerical results:

\begin{itemize}
    \item Similarly to \Cref{sec:experiments}, we demonstrate how \Cref{alg:stoc} systematically explores the Pareto front under different preference vectors $r$, and how different constants $u$ impact the performance of \Cref{alg:stoc}.

    \item As discussed in \Cref{sec:app-non}, the non-preference scenario can be viewed as a special case of our MOBL theory, for which we provide a theoretical analysis of the deterministic algorithm. We also note that the stochastic case shares the same underlying insight, and we present corresponding numerical results in this section.

    \item We also provide the numerical comparison between our approach and the MoCo method \citep{fernando2022mitigating}.
    
\end{itemize}


\begin{figure}[htb]
    \centering
    \begin{subfigure}[t]{0.25\textwidth}
        \centering
        \includegraphics[width=\linewidth]{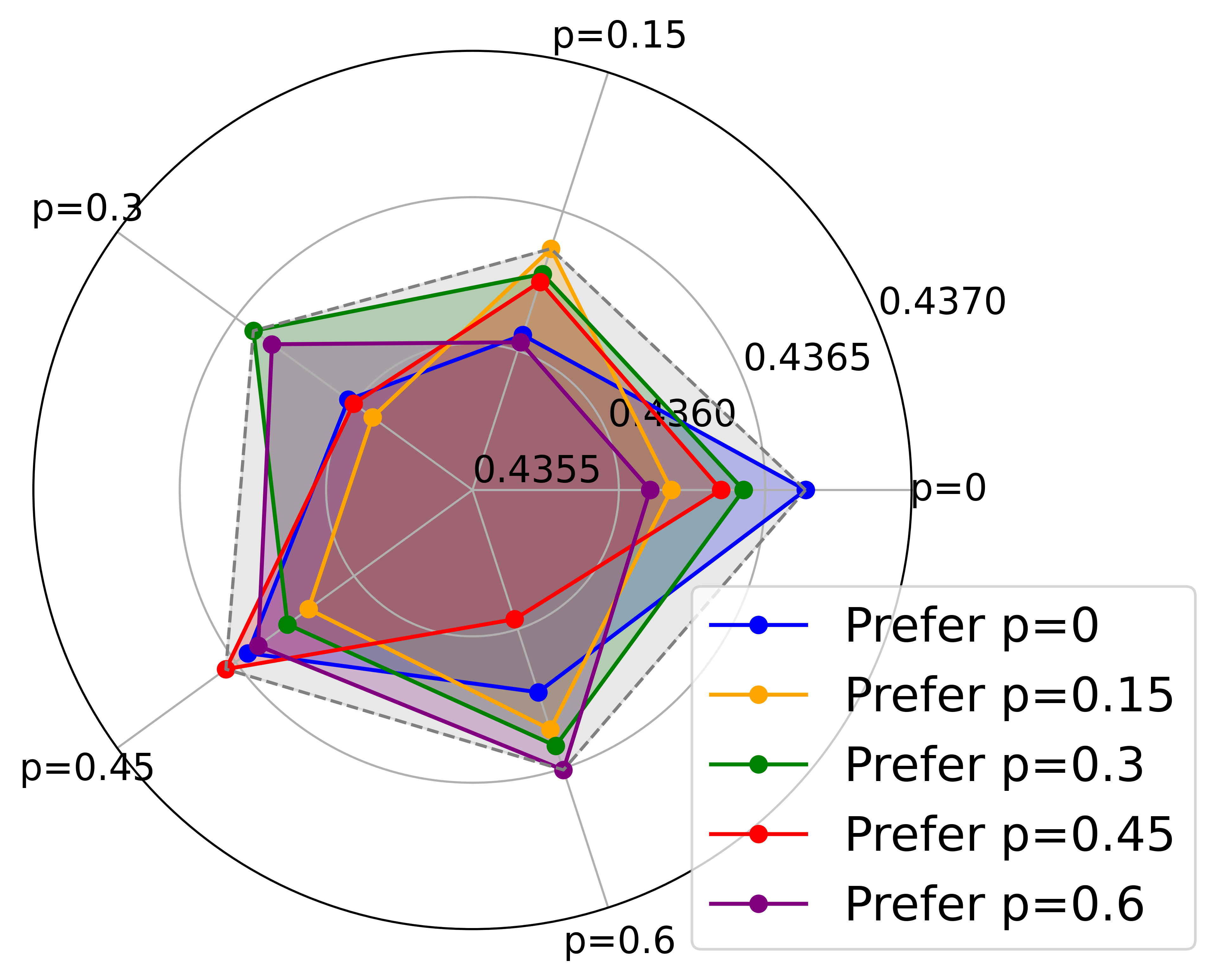}
        \caption{Exploration of Pareto front under different preference vectors $r$.}
        \label{fig:stoc-pref-app}
    \end{subfigure}
    \hspace{0.1cm}
    \begin{subfigure}[t]{0.25\textwidth}
        \centering
        \includegraphics[width=\linewidth]{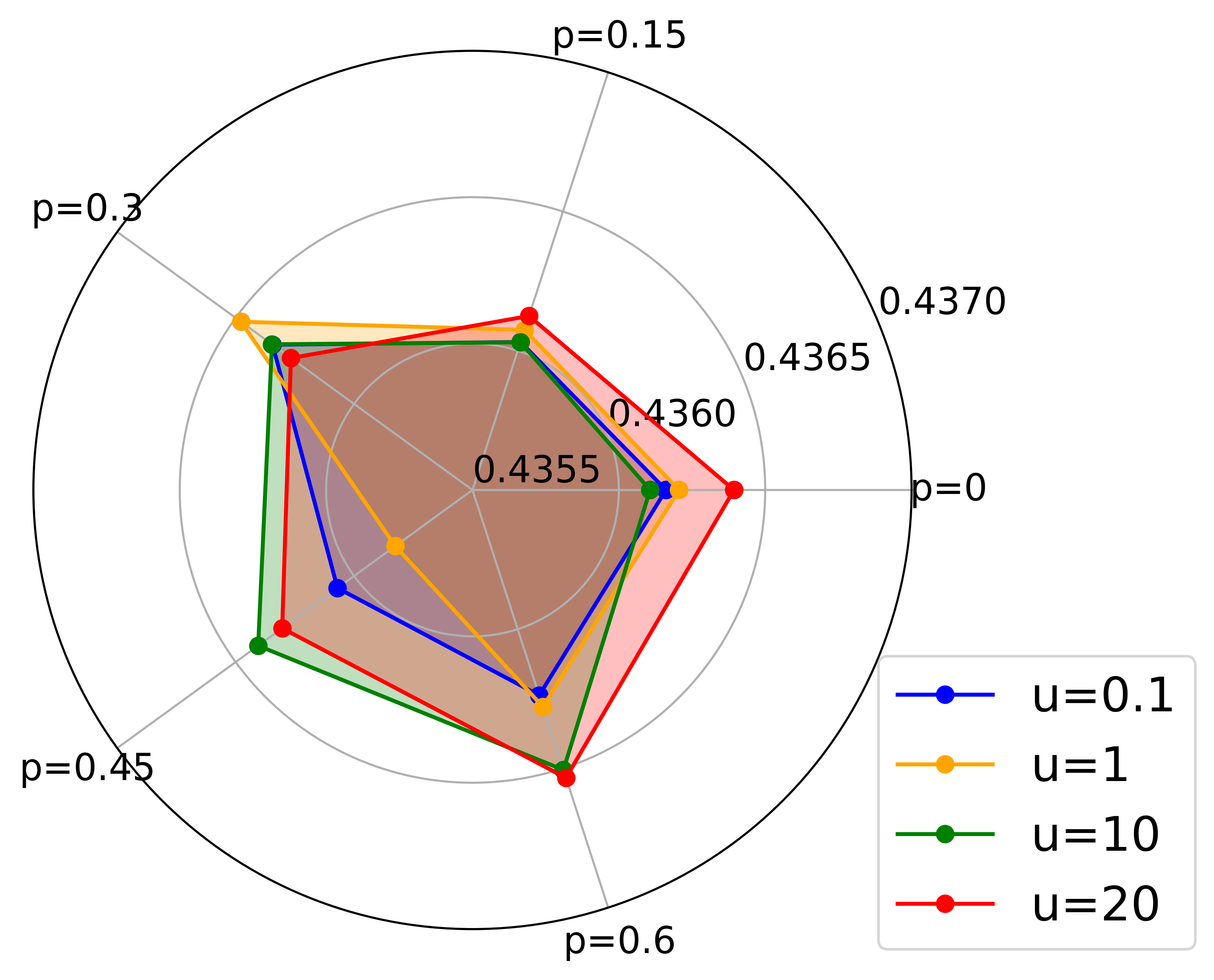}
        \caption{Performance under different constant $u$.}
        \label{fig:stoc-u}
    \end{subfigure}
    \hspace{0.1cm}
    \begin{subfigure}[t]{0.275\textwidth}
        \centering
        \includegraphics[width=\linewidth]{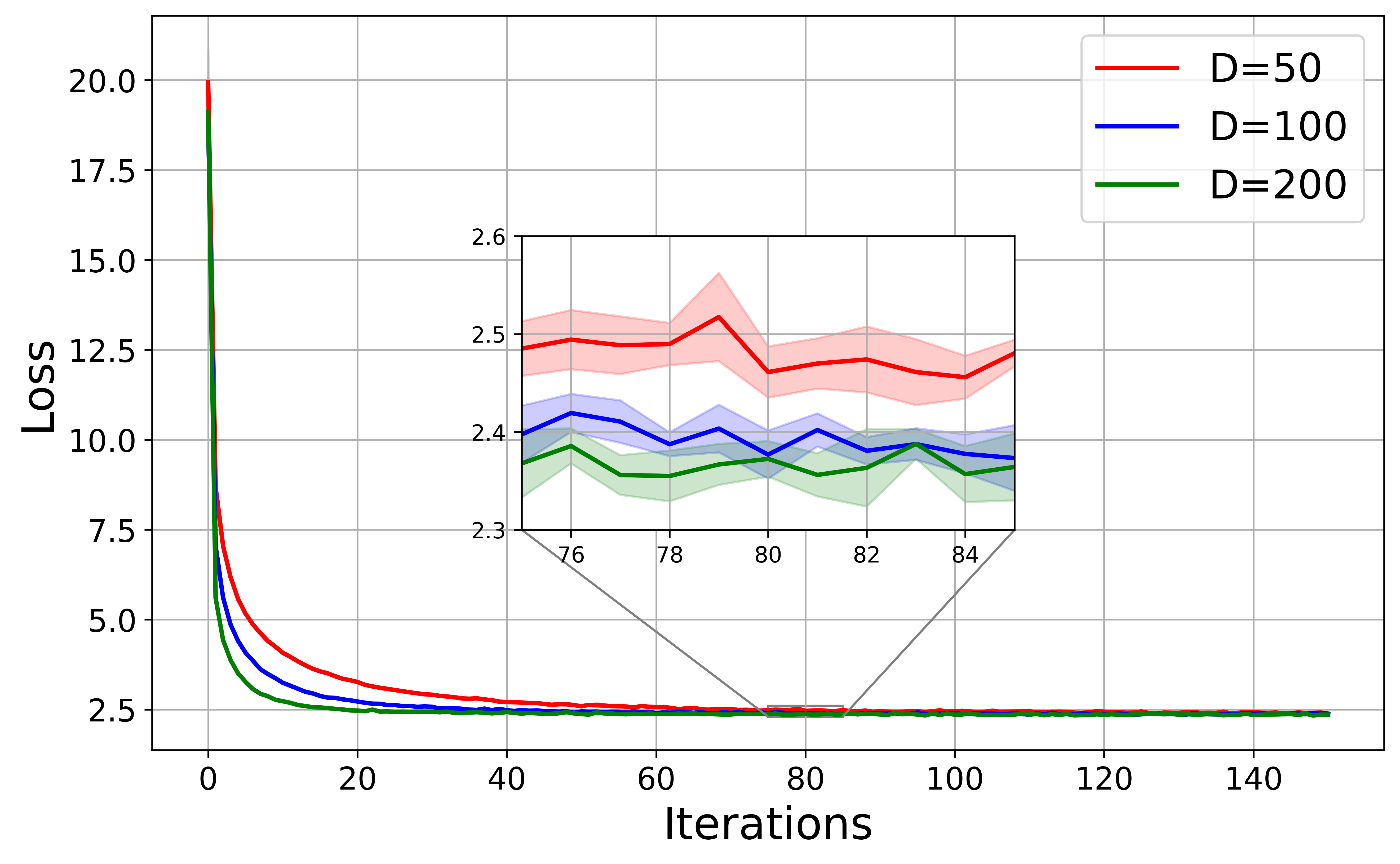}
        \caption{Convergence rate when prefer objective $p = 0.6$ and $u = 10$.}
        \label{fig:stoc-converge}
    \end{subfigure}
    \caption{Performances of \Cref{alg:stoc}.}
    \label{fig:stoc-performance}
\end{figure}

\textbf{Numerical Results.} The following simulation results verify the effectiveness of \Cref{alg:stoc} and the correctness of \Cref{thm:WC_stoc}.

\Cref{fig:stoc-pref-app} demonstrates that the design of \Cref{alg:stoc} effectively enables systematic exploration of the Pareto front under different preference vectors. Specifically, we set $r_p=0.9$ and $r_{p'}=0.025, \forall p'\neq p$ to represent a preference for the task with corruption rate $p$. Consistent with \Cref{fig:det-pref}, when the preference vector favors the task with corruption rate $p$, our algorithm reaches the smallest loss, equivalently, the largest $1/L$, on the corresponding axis. Additionally, the gray dotted curves highlight that the exploration area is significantly larger than that of any single preference-fixed curve, effectively showcasing the capability of \Cref{alg:stoc} to explore the Pareto front comprehensively.

In \Cref{fig:stoc-u}, we fix $r=[0.025, 0.025, 0.025, 0.025, 0.9]^\top$, indicating a strong focus on the task with a corruption rate of $p=0.6$, and vary the trade-off controller constant $u$ from set $\{0.1, 1, 10, 20\}$. As expected, the metric $1/L$ increases with larger $u$, demonstrating a stronger emphasis on the preferred task. This result validates the theoretical insight of \Cref{eq:WC}, where a larger $u$ enhances the alignment with the specified preference vector.

\begin{figure}[htb]
    \centering
    \begin{subfigure}[t]{0.275\textwidth}
        \centering
        \includegraphics[width=\linewidth]{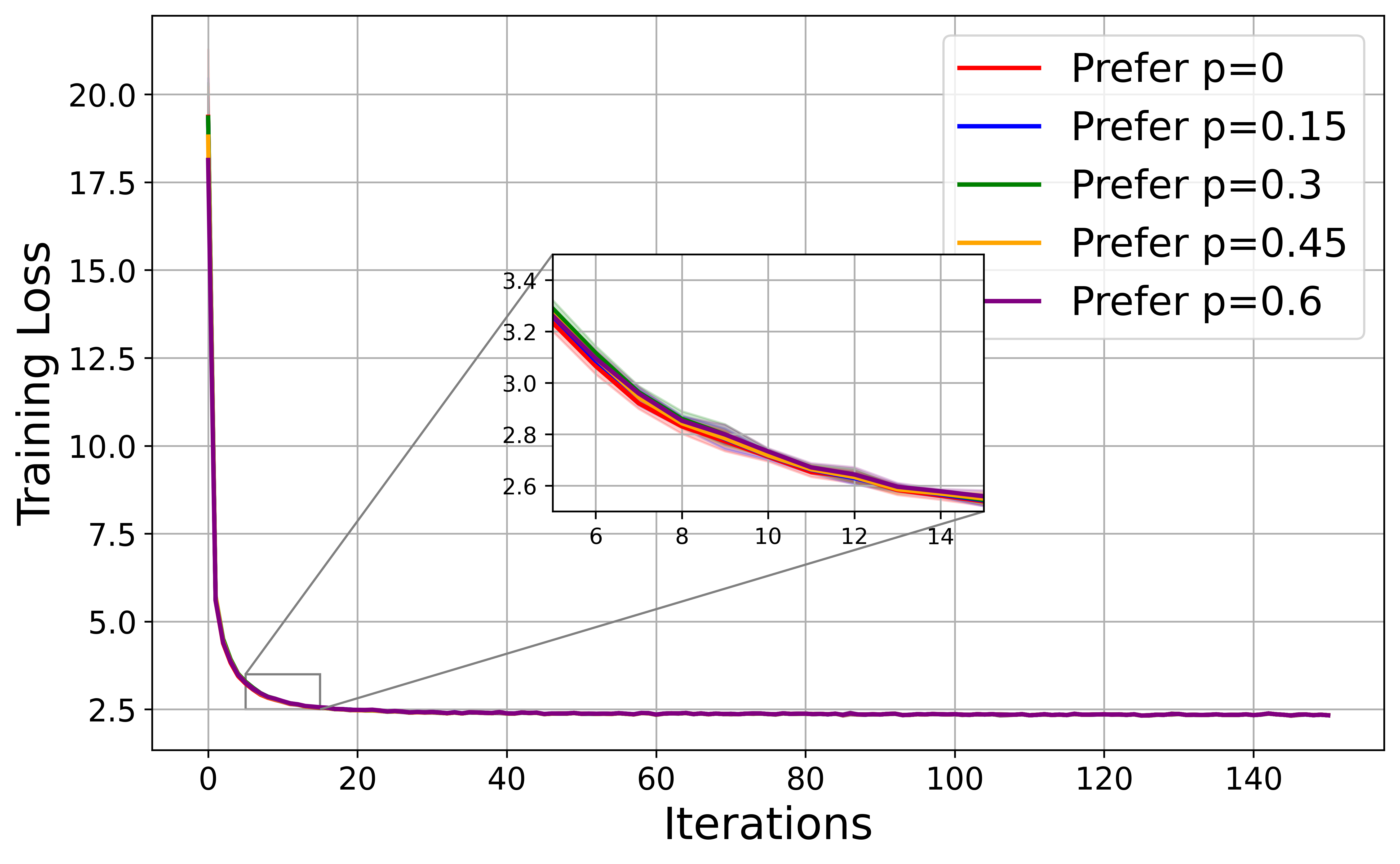}
        \caption{Training Loss for $5$ objectives}
        \label{fig:HO-non-train-5}
    \end{subfigure}
    \hspace{0.5cm}
    \begin{subfigure}[t]{0.275\textwidth}
        \centering
        \includegraphics[width=\linewidth]{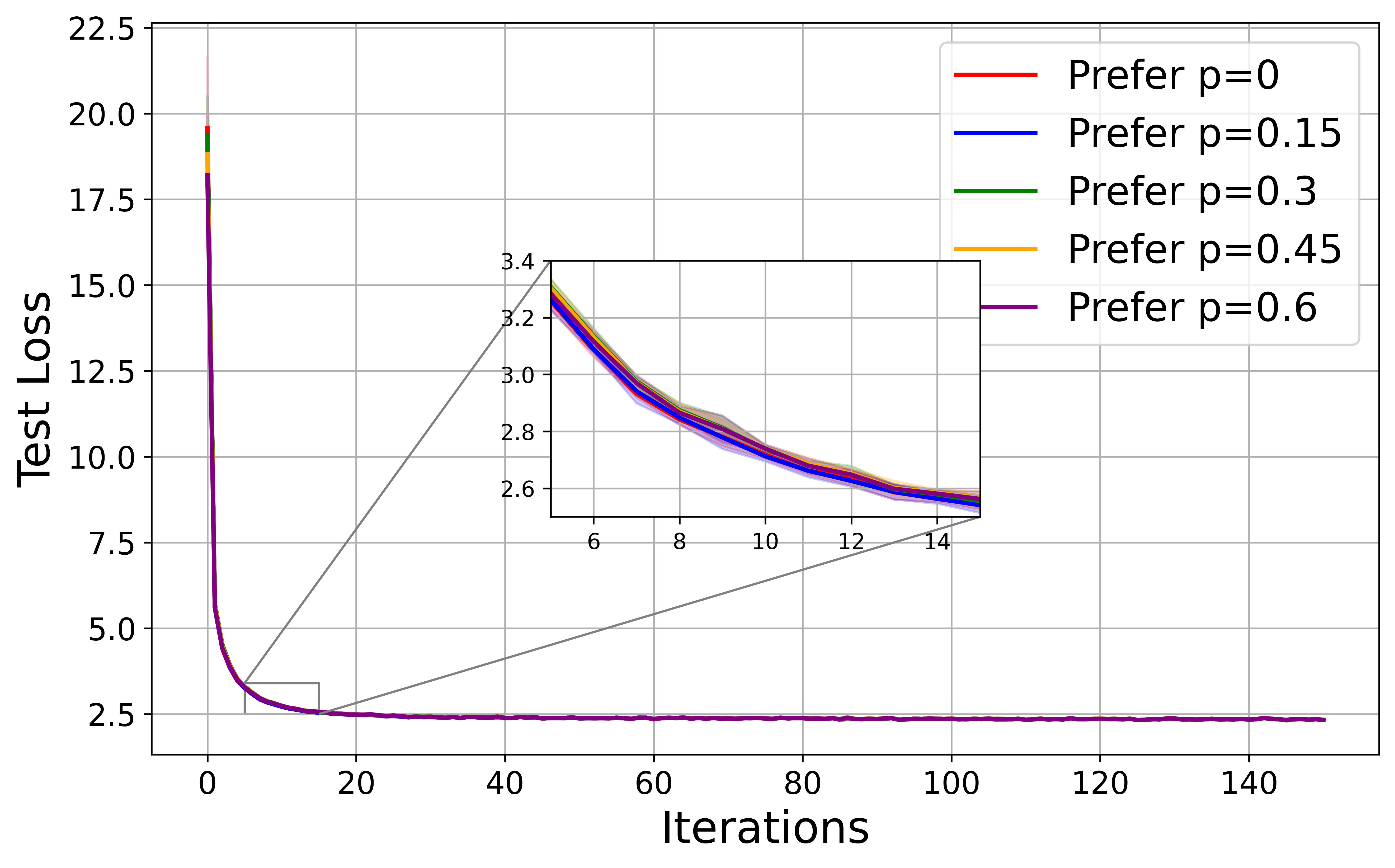}
        \caption{Test Loss for $5$ objectives}
        \label{fig:HO-non-test-5}
    \end{subfigure}


    \begin{subfigure}[t]{0.275\textwidth}
        \centering
        \includegraphics[width=\linewidth]{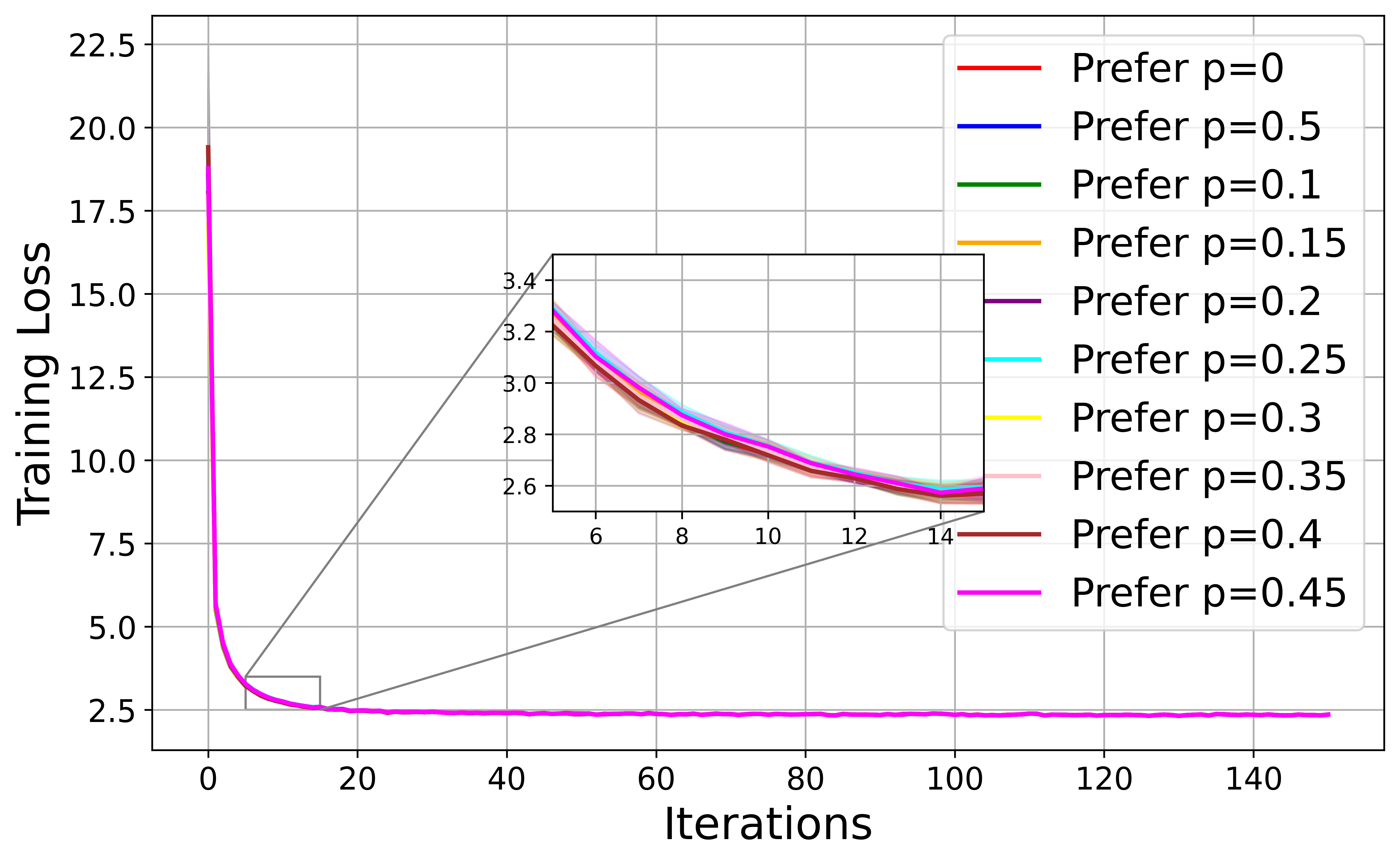}
        \caption{Training Loss for $10$ objectives}
        \label{fig:HO-non-train-10}
    \end{subfigure}
    \hspace{0.5cm}
    \begin{subfigure}[t]{0.275\textwidth}
        \centering
        \includegraphics[width=\linewidth]{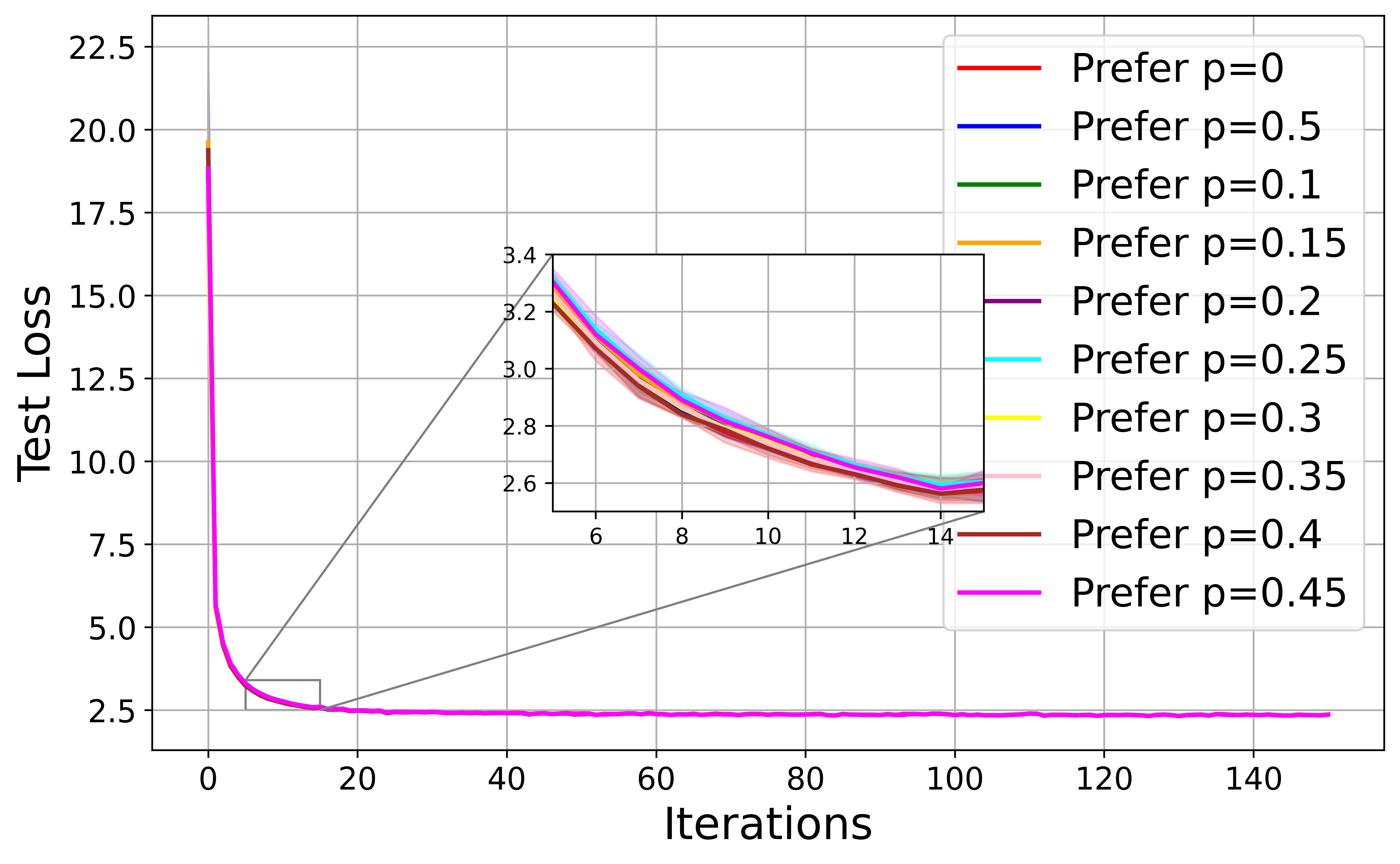}
        \caption{Test Loss for $10$ objectives}
        \label{fig:HO-non-test-10}
    \end{subfigure}

    \caption{The average loss of $5$ and $10$ objectives without the guidance of preference.}
    \label{fig:HO-non-convergence}
\end{figure}

\Cref{fig:stoc-converge} fixes $r=[0.025, 0.025, 0.025, 0.025, 0.9]^\top$ and $u=10$, while varying the inner-loop iteration count $D$ from $\{50, 100, 200\}$, and shows the average loss value of $5$ aforementioned potentially conflicting tasks. All three curves illustrate their respective convergence behaviors. As implied by \Cref{thm:WC_stoc}, $D$ serves as the exponent of $\frac{L-\mu_g}{L+\mu_g}$, which is strictly less than $1$. Consequently, a larger $D$ is expected to result in a smaller loss value, and this trend is precisely reflected in the numerical results shown in \Cref{fig:stoc-converge}.


\Cref{fig:HO-non-convergence} illustrates the convergence behavior of \Cref{alg:NP}, where no preference vector is used for guidance. In this scenario, the sequence converges arbitrarily to a point on the Pareto front. We consider cases with $S=5$ and $S=10$ different corruption rates $p$, respectively. In both settings, the loss values converge in fewer than $K=150$ iterations.


\begin{figure}[htb]
    \centering
    \begin{subfigure}[t]{0.295\textwidth}
        \centering
        \includegraphics[width=\linewidth]{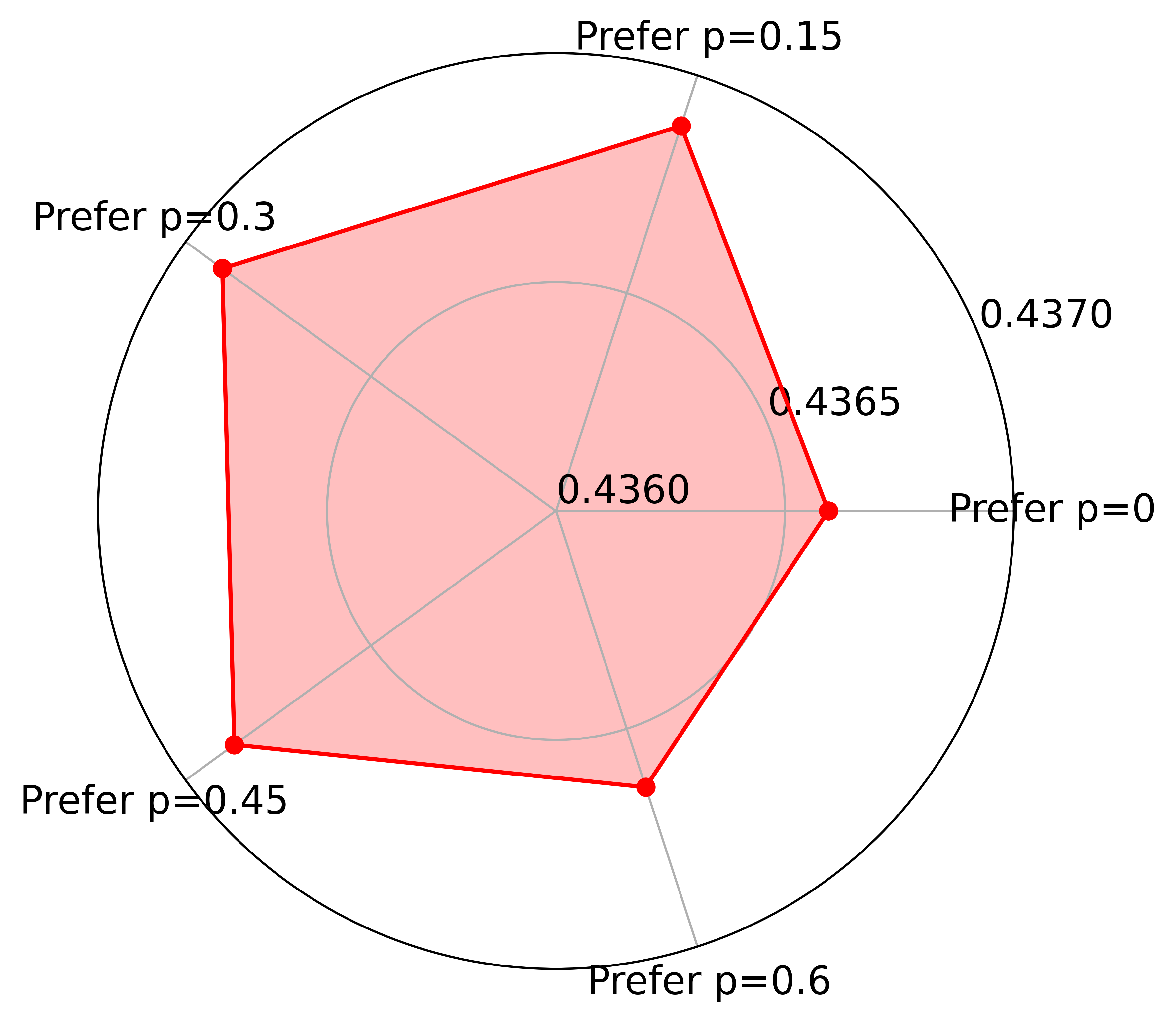}
        \caption{Training Performance for $5$ objectives}
        \label{fig:HO-non-train-radar-5}
    \end{subfigure}
    \hspace{0.5cm}
    \begin{subfigure}[t]{0.295\textwidth}
        \centering
        \includegraphics[width=\linewidth]{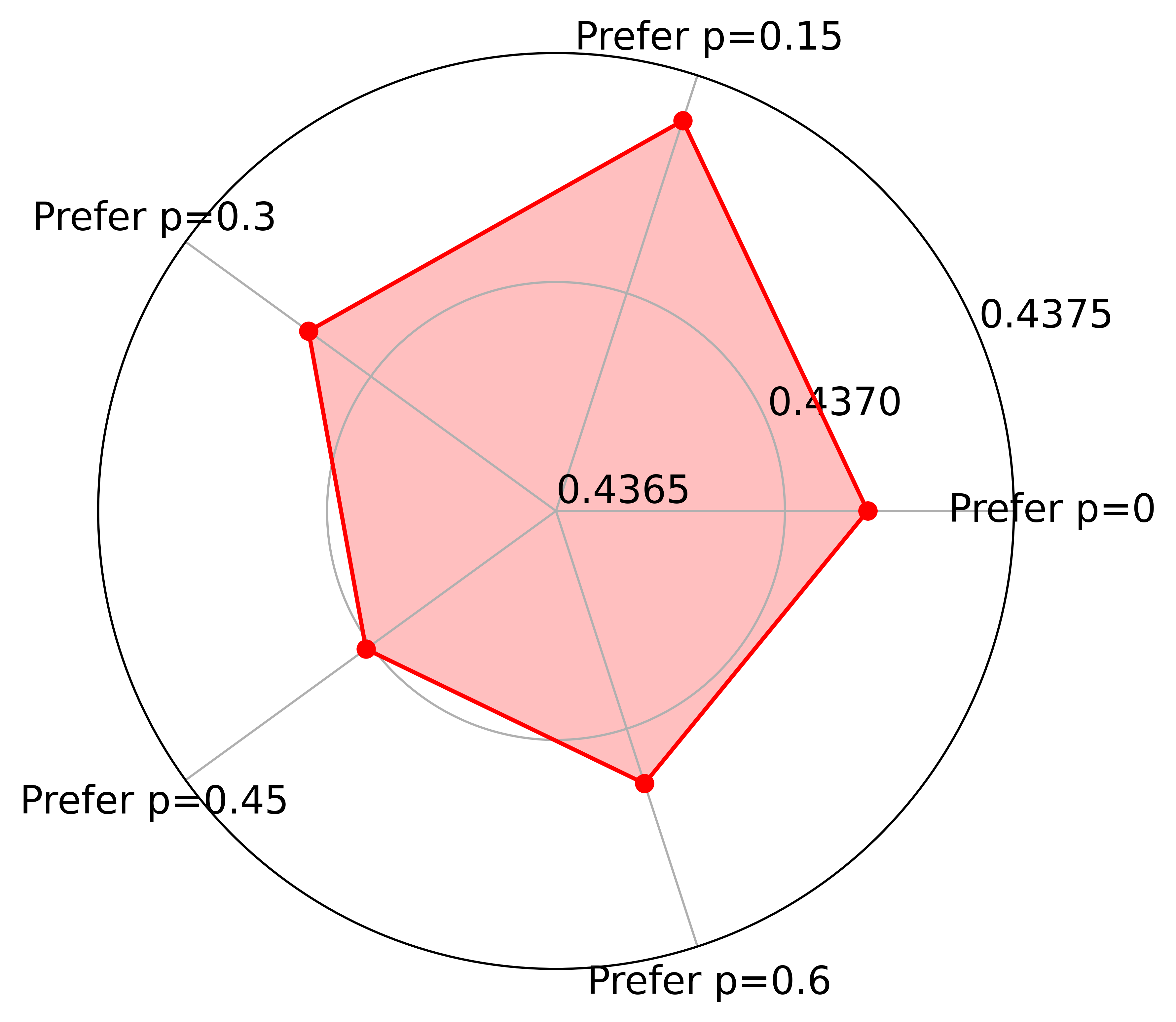}
        \caption{Test Performance for $5$ objectives}
        \label{fig:HO-non-test-radar-5}
    \end{subfigure}


    \begin{subfigure}[t]{0.295\textwidth}
        \centering
        \includegraphics[width=\linewidth]{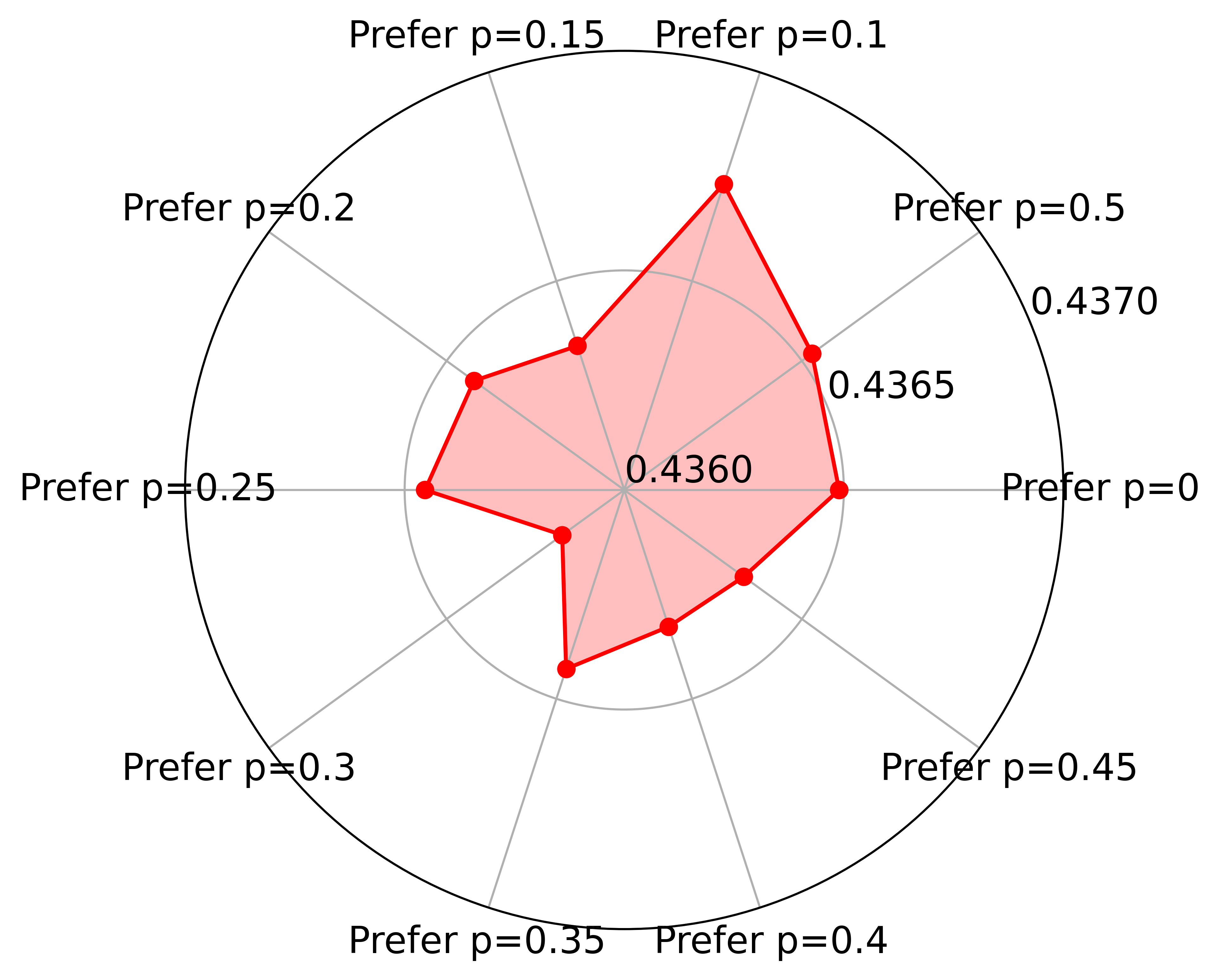}
        \caption{Training Performance for $10$ objectives}
        \label{fig:HO-non-train-radar-10}
    \end{subfigure}
    \hspace{0.5cm}
    \begin{subfigure}[t]{0.295\textwidth}
        \centering
        \includegraphics[width=\linewidth]{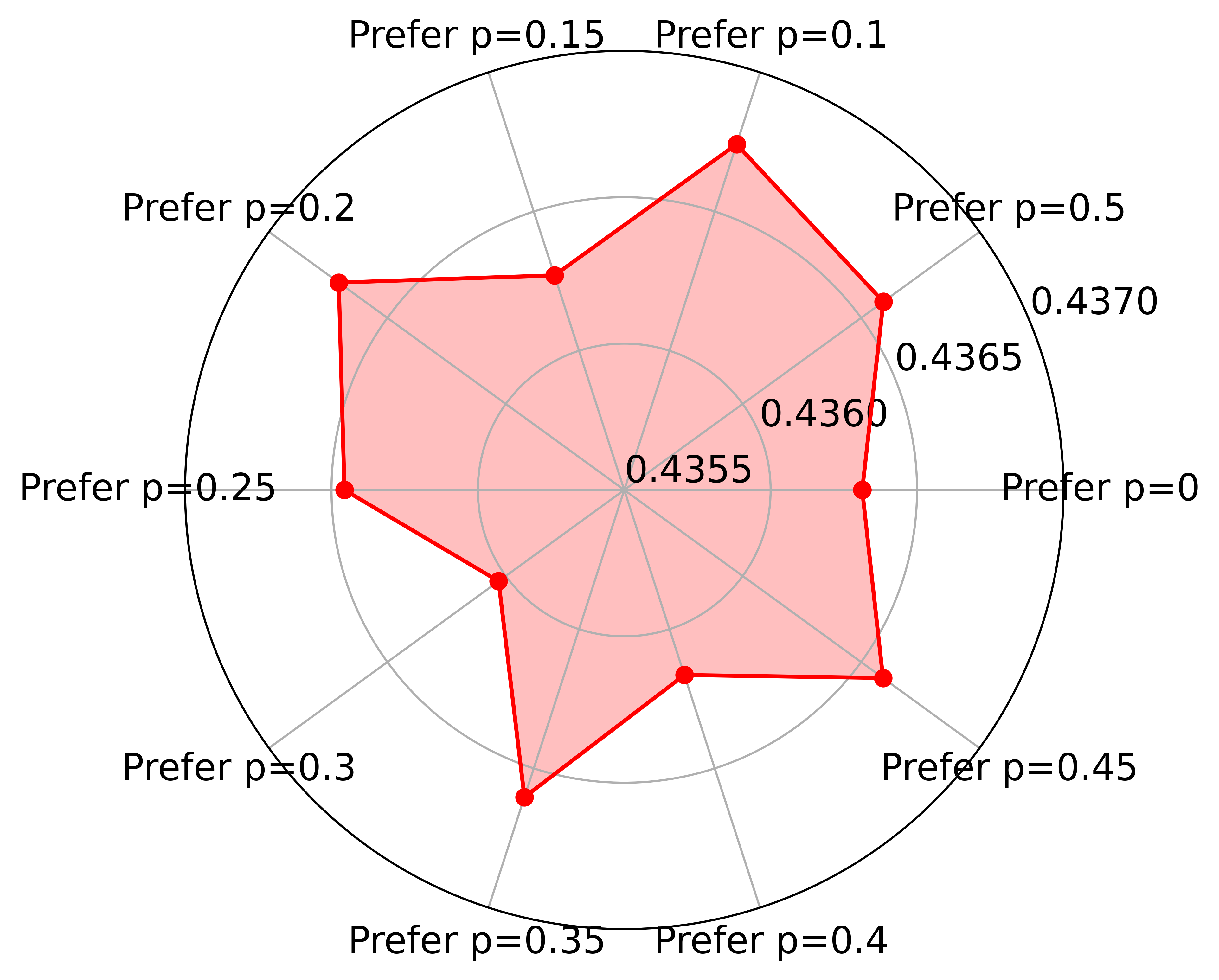}
        \caption{Test Performance for $10$ objectives}
        \label{fig:HO-non-test-radar-10}
    \end{subfigure}

    \caption{The average performance $1/L$ of $5$ and $10$ objectives without the guidance of preference.}
    \label{fig:HO-non-radar}
\end{figure}

\Cref{fig:HO-non-radar} depicts the metric $1/L$ for each task with corruption rate $p$ the scenarios of $S=5$ and $S=10$. Similar to \Cref{fig:ml-non}, the results indicate a random convergence pattern, meaning no single objective consistently dominates the others. This further emphasizes the effectiveness of our algorithm design in \Cref{alg:stoc}, as it successfully focuses on specific objectives and achieves better performance for those objectives.

\begin{figure}[htb]
    \centering
    \includegraphics[width=0.375\textwidth]{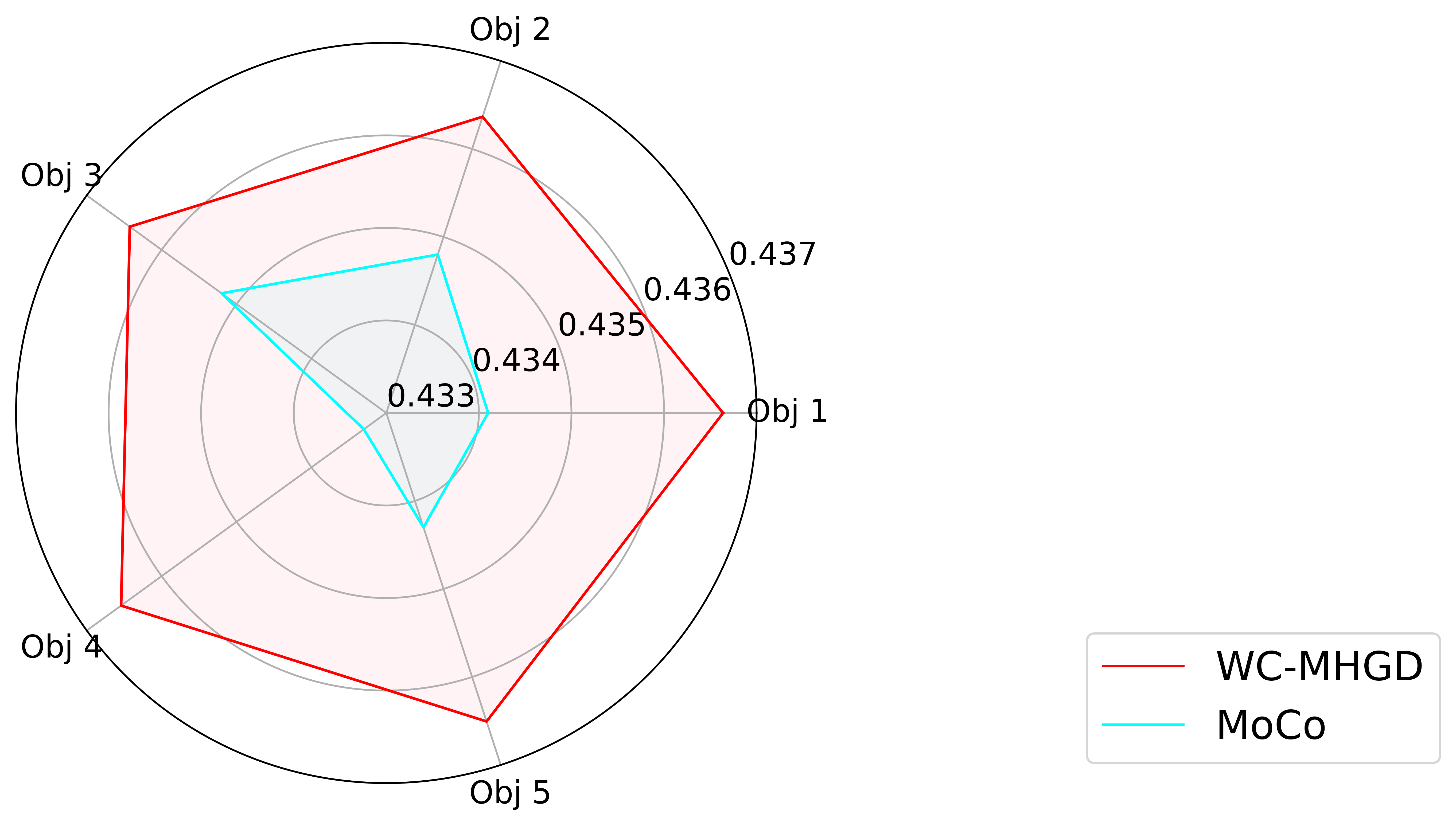}
    \caption{Comparison with baseline MoCo.}
    \label{fig:stoc-compare}
\end{figure}

\Cref{fig:stoc-compare} further demonstrates that our method consistently outperforms the baseline MoCo \citep{fernando2022mitigating} across all $5$ objectives. This, reaffirms the effectiveness of our algorithm in systematically exploring the Pareto front.